\documentclass[letterpaper]{article} 

\usepackage[draft]{aaai25}  

\usepackage{times}  
\usepackage{helvet}  
\usepackage{courier}  
\usepackage[hyphens]{url}  
\usepackage{graphicx} 
\usepackage{natbib}  
\usepackage{caption} 
\usepackage[utf8]{inputenc} 
\usepackage{url}            
\usepackage{booktabs}       
\usepackage{amsfonts}       
\usepackage{centernot}
\usepackage{nicefrac}       
\usepackage{microtype}      
\usepackage{xcolor}         
\usepackage{amsmath,amsthm, mathrsfs,mathtools}
\usepackage[ruled]{algorithm2e}
\usepackage{subcaption}
\usepackage{xfrac}
\usepackage{enumitem}
\usepackage{multirow}
\usepackage{tabularx}
\usepackage{pifont}%

\usepackage[ruled]{algorithm2e}

\usepackage{amsmath,amsthm, mathrsfs,mathtools}

\DeclareMathOperator{\E}{\mathbb{E}}
\DeclareMathOperator{\Prob}{\mathbb{P}}

\newcommand{\Wass}{\mathcal{W}}
\newcommand{\Causal}{\mathcal{C}_b}

\newcommand{\cN}{\mathcal{N}}

\newcommand{\norm}[1]{\|{#1}\|}

\newcommand{\indep}{\perp \!\!\! \perp}

\theoremstyle{plain}
\newtheorem{theorem}{Theorem}[section]

\newtheorem{proposition}[theorem]{Proposition}

\newtheorem{remark}[theorem]{Remark}

\theoremstyle{definition}
\newtheorem{definition}[theorem]{Definition}

\newcommand{\mg}[1]{\textcolor{blue}{\textbf{MG:} #1}}
\newcommand{\am}[1]{\textcolor{purple}{\textbf{Alan:} #1}}

\def\mb{\mathbb}

\def\mc{\mathcal}

\renewcommand{\P}{\mb{P}}

\newcommand{\cn}{\textcolor{red}{[\raisebox{-0.2ex}{\tiny\shortstack{citation\\[-1ex]needed}}]}}

\newcommand{\todo}[1]{\textcolor{red}{\textbf{TODO:} #1}}

\newcommand{\bcd}{\textit{FairBiT}}
\newcommand{\fairbit}{\textit{FairBiT}}
\newcommand{\fairlp}{\textit{FairLeap}}

\newcommand{\cmnt}[1]{\ignorespaces}  

\newcommand{\cmark}{\ding{51}}%
\newcommand{\xmark}{\ding{55}}%

\def\cddw{$\mathsf {CDD}^{\mathsf {wass}}_{f}$}
\def\cddlp{$\mathsf {CDD}^{\mathsf {\ell_p}}_{f}$}
\def\cddwtext{$\mathsf {CDD}^{\mathsf {wass}}$}
\def\cddlptext{$\mathsf {CDD}^{\ell_p}$}

\title{Auditing and Enforcing Conditional Fairness via Optimal Transport}
\author{
    Mohsen Ghassemi\equalcontrib\textsuperscript{\rm 1}, Alan Mishler\equalcontrib\textsuperscript{\rm 1}, Niccol\`o Dalmasso\equalcontrib\textsuperscript{\rm 1}, Luhao Zhang\textsuperscript{\rm 2}\thanks{Work done during an internship at J.P.Morgan AI Research.}\\
    Vamsi K. Potluru\textsuperscript{\rm 1},
    Tucker Balch\textsuperscript{\rm 1},
    Manuela Veloso\textsuperscript{\rm 1}
}
\affiliations {
    \textsuperscript{\rm 1}J.P.Morgan AI Research\\
    \textsuperscript{\rm 2}Department of Applied Mathematics and Statistics, Johns Hopkins University\\
    \{mohsen.ghassemi, alan.mishler, niccolo.dalmasso\}@jpmchase.com,
    luhao.zhang@jhu.edu
}
\usepackage{bibentry}

\begin{document}

\maketitle

\begin{abstract}
Conditional demographic parity (CDP) is a measure of the demographic parity of a predictive model or decision process when conditioning on an additional feature or set of features. Many algorithmic fairness techniques exist to target demographic parity, but CDP is much harder to achieve, particularly when the conditioning variable has many levels and/or when the model outputs are continuous. The problem of auditing and enforcing CDP is understudied in the literature. In light of this, we propose novel measures of {conditional demographic disparity (CDD)} which rely on statistical distances borrowed from the optimal transport literature. We further design and evaluate regularization-based approaches based on these CDD measures. Our methods, \fairbit{} and \fairlp{}, allow us to target CDP even when the conditioning variable has many levels. When model outputs are continuous, our methods target full equality of the conditional distributions, unlike other methods that only consider first moments or related proxy quantities. We validate the efficacy of our approaches on real-world datasets. 
\end{abstract}

\section{Introduction}
\cmnt{
\todo{
\begin{itemize}
    \item overfull hbox
    \item Make sure the order of items in the appendix makes sense and we refer to every section in the main paper.
\end{itemize}
}
}

Algorithmic decision-making has an increasing impact on individuals' lives, in areas such as finance, healthcare, and hiring. The growing field of algorithmic fairness aims to define, measure, and prevent discrimination in such systems.

Much of the early work in algorithmic fairness adopted as a fairness criterion \textit{demographic parity} (DP) (aka statistical parity), which requires the outputs of a model or decision process to be statistically independent of a sensitive feature such as race, gender, or disability status \cite{calders2009BuildingClassifiersIndependency, kamiran2010ClassificationNoDiscrimination, kamiran2012DecisionTheoryDiscriminationAware}. Demographic parity remains arguably the most widely studied fairness criterion to date \cite{hort2022BiasMitigation}, but it can permit an algorithm to behave in intuitively unfair ways \cite{dwork2012FairnessAwareness, hardt2016EqualityOpportunity}. Consider the hypothetical loan approval process summarized in Table \ref{tab:synthetic_example}. The overall approval rate is 50\% for both male and female applicants, \cmnt{so this process satisfies demographic parity;} but within each income level, males are more likely \cmnt{than females} to get approved. \cmnt{This is an example of Simpson's paradox; it is possible because income is positively correlated with approval probability, and females in this example are at a higher income level on average than males.} If applicants within income levels are equally qualified, then this process appears to discriminate against females.

\textit{Conditional demographic parity} (CDP) instead requires independence of the output and the sensitive feature conditional on a \textit{legitimate} or \textit{explanatory} feature or set of features, such as income in the loan example \cite{zliobaite2011HandlingConditionalDiscrimination, kamiran2013QuantifyingExplainableDiscrimination}. \cmnt{For example, CDP might require the loan approval rate to be the same for males and females within each income level, while permitting differences in approval rates across levels as well as at the population level. 
}One way to achieve CDP is to use only the legitimate features as model inputs, but this may result in unacceptably poor predictive performance. \cmnt{Additionally, it is frequently desirable for a practitioner to have a tunable knob for navigating fairness-performance trade-offs \cite{ZhaoInherent2019, kim2020FACTDiagnosticGroupa, liu2021AccuracyFairnessTradeoffs}}
When the legitimate features have a small number of levels, CDP can also be satisfied by maintaining a separate model for each subgroup defined by values of the legitimate feature and applying a method that targets DP to each subgroup. However, this approach may be infeasible when the legitimate feature has many levels.

To our knowledge, there is only one existing method designed to target conditional demographic parity even in the presence of many-valued legitimate features and/or continuous outputs \cite{xu2020algorithmic}. This method uses a regularizer that is derived from a particular characterization of conditional independence. We show, however, that except when the regularizer is exactly zero, minimizing this quantity does not provide guarantees with respect to the actual disparities.

Additionally, quantifying (un)fairness in the sense of conditional demographic parity remains underexplored in the literature. As it is generally impossible to exactly satisfy a fairness criterion without sacrificing model performance \cite{corbett2017algorithmic, ZhaoInherent2019}, it is crucial to define appropriate measures of \textit{conditional demographic disparity (CDD)}, i.e. to quantify violations of CDP. 
Quantifying CDD entails measuring disparities at every level of the legitimate features. While any existing measure of demographic disparity may be used at each level, aggregating the level-wise disparities to obtain a single CDD value requires careful attention. We are not aware of any previous work investigating this issue.

\begin{table}[ht]
\centering
\begin{tabular}{@{}lcccc@{}}
\toprule
\multirow{2}{*}{Income Level} & \multicolumn{2}{c}{Approval Rate} & \multicolumn{2}{c}{Applicants} \\ 
\cmidrule(l){2-5} 
              & Male    & Female   & Male     & Female    \\ \cmidrule(r){1-5}
High Income   & 80\% & 60\% & 10\%  & 60\%  \\
Medium Income & 60\% & 40\% & 30\%  & 30\%  \\
Low Income    & 40\% & 20\% & 60\%  & 10\%  \\
Overall         & 50\% & 50\% & 100\% & 100\% \\ \bottomrule
\end{tabular}
\caption{Toy example illustrating a process which satisfies demographic parity but not conditional demographic parity. Males and females have the same marginal (``Overall'') loan approval rate (50\%), but within each income level, males have higher approval rates than females. \cmnt{This is possible because females have higher income on average than males, and higher income leads to a greater approval probability. This is an example of the well-known Simpson's paradox.}}
\label{tab:synthetic_example}
\end{table}

\subsection{Contributions} 
\cmnt{\subsubsection{(Version 1 of this subsection)}
We first introduce two new generic measures of conditional demographic \emph{disparity} (CDD): the \emph{CDD in the $\ell^p$  sense} and the \emph{CDD in the Wasserstein sense}. Both measures take the entire (conditional) distributions into account and work for both classification and regression. We propose and evaluate regularizers designed to minimize different versions of these disparities.

One set of regularizers consists of weighted sums of distributional distances at each level of the legitimate feature. We additionally propose a regularizer which considers the entire joint distributions while enforcing conditional independence. We call this last method FairBiT: Fairness through Bi-causal Transport. The FairBiT regularizer is based on the \emph{bi-causal transport distance}, a distributional distance recently studied in the optimal transport literature.

All our regularizers work regardless of whether the algorithm output is continuous or discrete, and unlike the existing state of the art \cite{xu2020algorithmic}, they produce models which do not require access to the sensitive feature at inference time. \textcolor{red}{FairBiT in particular works even when the legitimate feature has many levels and a small amount of data per level.} We show that our methods generally provide better fairness-performance tradeoffs than other methods on a range of synthetic and real datasets.} 

We first introduce two new generic measures of CDD: the \emph{CDD in the Wasserstein sense} and the \emph{CDD in the $\ell^p$  sense}. Both measures take the entire (conditional) distributions into account and work for both classification and regression. We propose and evaluate regularizers designed to minimize different versions of these disparities.

For CDD in the Wasserstein sense, minimizing the disparity is nontrivial. We propose a regularizer based on the \emph{bi-causal transport distance}, a distributional distance recently studied in the optimal transport literature. We call this method \textit{FairBiT: conditional Fairness through Bi-causal Transport}. For CDD in the $\ell_p$ sense, we propose \textit{\fairlp: conditional Fairness in the $\ell_p$ sense}, where the regularizer is essentially a weighted $\ell_p$ norm of the vector of level-wise disparities.

Our regularization-based approaches provide practitioners with a tunable knob for navigating fairness-performance trade-offs \cite{ZhaoInherent2019, kim2020FACTDiagnosticGroupa, liu2021AccuracyFairnessTradeoffs}. All our regularizers work regardless of whether the algorithm output is continuous or discrete, and since they utilize the entire dataset, they can be meaningfully applied even when the legitimate feature has many levels. Unlike the existing state of the art \cite{xu2020algorithmic}, 
our proposed methods do not require access to the sensitive feature at inference time. We show that our methods generally provide better fairness-performance trade-offs than other methods on a range of datasets.

\subsection{Related Work}

\paragraph{Demographic parity (DP) and conditional demographic parity (CDP)}
There are a vast number of methods designed to target DP \cite{hort2022BiasMitigation}, but there has been very little investigation of CDP \cite{zliobaite2011HandlingConditionalDiscrimination, kamiran2013QuantifyingExplainableDiscrimination}. \cmnt{(See \cite{hort2022BiasMitigation, pessach2023ReviewFairnessMachine} for recent surveys of fairness definitions and methods. Note that the terms demographic and statistical parity came into usage after some of these methods were developed.)} \cmnt{Most methods designed for CDP operate by stratifying on the legitimate feature and applying methods designed to achieve DP within each stratum. In principle, any method for DP can be applied in this way, but this may be infeasible when the legitimate feature has many levels. , and it has the disadvantage that the amount of data per level is much smaller than the total sample size. By contrast, our method applies a single regularizer to the entire dataset, and it can handle legitimate features that are discrete with many levels.}
A natural way to target CDP is to stratify on the legitimate feature and apply methods designed to achieve DP within each stratum. However, this may result in poor overall model performance, especially if there is a small amount of data in each level. Our methods utilize the entire dataset during model training. Additionally, the majority of methods for DP are designed for classification or only designed to equalize the first moments of the two distributions, whereas we target full conditional independence in both classification and regression settings. 

The closest work to ours is \citet{xu2020algorithmic}, who propose DCFR, the Derivable Conditional Fairness Regularizer. DCFR also accommodates rich legitimate features and continuous outcomes and also involves a regularizer applied to the entire dataset. However, we show in Section \ref{sec:DCFR} that the regularizer is equivalent to a proxy quantity that is distinct from the disparity of interest. Small values of the proxy quantity do not necessarily imply small values of the disparity. By contrast, we utilize regularizers that directly target the distances between the relevant conditional distributions. Additionally, DCFR requires access to the sensitive feature at inference time, which our methods do  not.

\paragraph{Fairness for binary vs. continuous model outputs}
The majority of the algorithmic fairness literature considers classification settings in which the final model outputs are binary, though there is a growing set of methods that can handle regression settings \cite{chzhen2020fair,chzhen2022minimax,Romano2020neurips,fukuchi2024demographic,xian2024differentially,coston2021CharacterizingFairnessSet, mishler2021FADEFAirDouble,franklin2023ontology,jin2023fairness}. Many methods which can handle continuous outputs only aim to equalize the first moments of the output distributions across levels of the sensitive feature (e.g. the average loan amount for males vs. females); a smaller but growing number target full equality of the output distributions \cite{chzhen2020fair,chzhen2022minimax, Romano2020neurips,fukuchi2024demographic}. Our methods falls in the latter category, aiming to ensure equality of conditional output distributions across levels of the legitimate feature, regardless of whether the output is discrete or continuous. \cmnt{(Note that when the output is binary, equalizing the first moments is equivalent to equalizing the full distributions, since the distribution of a Bernoulli random variable is fully specified by its mean.)}

\cmnt{In general, fairness methods may be divided into \textit{pre-processing} methods, which transform the training data; \textit{in-processing} methods, which use constraints or regularization terms to guide the model training process; and \textit{post-processing} methods, which learn transformations of existing models. There are dozens or even hundreds of methods in each category, so once again we refer the reader to recent surveys of these methods \cite{hort2022BiasMitigation, pessach2023ReviewFairnessMachine}. Our method is an in-processing approach, since we use a regularizer during model training.}

\paragraph{Optimal transport for fairness}
Optimal transport has been increasingly used for fairness-related applications, including fair edge prediction, training fair predictors, and uncovering discrimination in predictors \cmnt{training fair classifiers, and uncovering discrimination in classifiers}  \cite{black2020fliptest,silvia2020general,laclau2021all,si2021testing,zehlike2020matching,chiappa2021fairness,buyl2022optimal,yang2022obtaining,rychener2022metrizing,jiang2020wasserstein,johndrow2019algorithm,gordaliza2019obtaining,jiang2020wasserstein,rychener2022metrizing,miroshnikov2022wasserstein,jourdan2023optimal,langbridge2024optimal,wang2024wasserstein}. \cmnt{Many methods work by imposing a small optimal transport distances between output distributions, in order to impose (approximate) statistical independence with respect to sensitive attributes. This allows users to target demographic parity as well as other criteria, whether by pre-processing(\cite{feldman2015certifying,johndrow2019algorithm,gordaliza2019obtaining}), in-processing (\cite{rychener2022metrizing}), or post-processing (\cite{jiang2020wasserstein,chzhen2020fair}) methods.} Optimal transport techniques generally require the specification of a distributional distance function. Most papers that apply optimal transport to fairness rely on the \textit{Wasserstein distance}.
Our proposed methods utilize the Wasserstein distance at each level of the legitimate feature, and aggregate these distances using an outer measure. One of our disparity measures, the CDD in the Wasserstein sense (Definition \ref{def:CDD_wass}) gives rise to a nested Wasserstein distance; to target this disparity, we utilize a distance known as the \textit{bi-causal transport distance} \cite{backhoff2017causal}. This distance does not appear to have been used previously in the context of algorithmic fairness. 

\section{Notation and Problem Setup} \label{section:notation_and_setup}
Throughout, we consider data drawn from a distribution $(X, A, Y) \sim \mathbb P$, where $X \in \mc X$ is a set of features, $A\in\{0,1\}$ is a binary sensitive feature, and $Y \in \mb R$ is a prediction target. We use ``$\indep$'' to denote statistical independence. In a slight abuse of notation, we use $\mathbb P$ to refer both to the probability measure and to its density, assuming the density is defined. We let $f: \mc X \to \mb R$ be a model whose (un)fairness we wish to measure. In practice, $f(X)$ might map to a prediction that is used downstream in some decision process, or it might map to an automated decision such as a loan approval decision. Everything that follows applies in either setting. 




\begin{definition}[Demographic parity (DP)]
The model $f(X)$ satisfies \textit{demographic parity (DP)} if \cmnt{$\Prob(f(X) \mid A=0) \equiv \Prob(f(X) \mid A =1)$, or in other words if } $f(X) \indep A$, or in other words if $\Prob(f(X) \mid A=0) \equiv \Prob(f(X) \mid A =1)$. \cite{agarwal2019fair}. \cmnt{Moreover, let $d(\cdot, \cdot)$ denote a distance between distributions on $\mc V$. The \textit{demographic disparity } for model $f(X)$ with respect to $d(\cdot, \cdot)$ is defined by
\begin{align}
    d(\Prob_{f(X)|A=0},\Prob_{f(X)|A=1}).
\end{align}}
\end{definition}

DP takes into consideration only the sensitive feature and model outputs; it is indifferent to the features $X$. As discussed in the introduction and illustrated in Table \ref{tab:synthetic_example}, this means that a model which satisfies DP may treat different groups differently within levels of $X$, which may result in intuitively unfair behavior. This motivates \textit{conditional demographic parity} \cite{kamiran2013QuantifyingExplainableDiscrimination, corbett2017algorithmic,ritov2017conditional} as an alternative.

\begin{definition}[Conditional demographic parity (CDP)]\label{def:cdp}
The model $f(X)$ satisfies conditional demographic parity with respect to a \cmnt{discrete} feature or set of 
\cmnt{discrete} features $L \subset X$ if
$
    f(X) \indep A | {L = l}
$,
for all $l \in \operatorname{supp} \mathbb P ({L | A= 0}) \cap \operatorname{supp} \mathbb P ({L| A= 1}).$
\end{definition}
We follow \citet{corbett2017algorithmic} in referring to $L$ as the \textit{legitimate feature(s)}. In the loan approval example (Table \ref{tab:synthetic_example}), $L$ would be the income level. We emphasize 
that the choice of $L$, and the choice of the sensitive feature $A$, are up to the user.

Throughout the remainder of the paper, we assume that $\operatorname{supp} \mathbb P ({L | A= 0}) = \operatorname{supp} \mathbb P ({L| A= 1})$, i.e. the legitimate features have the same support for both groups represented by the sensitive feature. Since any method targeting CDP 
requires multiple samples at each level of $L$, we further assume that $L$ is either naturally discrete or appropriately discretized. See Appendix \ref{subsec:app:continuous} for a discussion of this assumption.
 \cmnt{ Note that in Definiton \ref{def:cdp} we require $L$ to be discrete since the concept of CDP involves comparing conditional distributions at each ``level'' of $L$. This is not overly limiting since continuous features may be discretized using appropriate binning strategies informed by the fairness requirements and the problem setting\footnote{An extension of CDP to continuous features is an interesting direction of future work. We envision a ``smooth CDP'' definition where parity is required not only at each level but to some degree at different levels. The requirement will be reversely related to the distance between levels. Such definition would be applicable to both discrete and continuous feature (without any binning). The implications and appropriateness of such definition needs careful attention and is beyond the scope of this work.}.}

In this work, we investigate how to \textit{effectively promote conditional demographic parity in supervised learning problems}. 
A crucial step towards enforcing parity is to choose an appropriate disparity measure. We discuss this next.

\section{Measuring disparities} \label{section:CDD}

Quantifying violations of parity informs the development of fairness methods, and it allows the fairness of different models and methods to be compared on a continuous scale. Any definition of the \emph{conditional demographic disparity} (CDD), the violation of CDP, must take into account the conditional distributions for each possible level $l$ of the legitimate features, as well as how to aggregate these level-wise features to output a single value. 

In principle, we could utilize any measure of \emph{demographic disparity} (the violation of DP) from the literature to measure violations at each level $l$. Most previous work defines the demographic disparity by $|\E[f(X) \mid A = 1] - \E[f(X) \mid A = 0]|$ or by $\E[f(X) \mid A = 1]/\E[f(X) \mid A = 0]$, considering only the first moments of the distributions. (See Appendix \ref{subsec:app:DD}.) We instead consider distances between the full conditional distributions. We introduce two notions of the conditional demographic disparity, \textit{CDD in the Wasserstein sense} and \textit{CDD in the $\ell^p$ sense}. Here, ``$\ell^p$'' and ``Wasserstein'' refer to the aggregation method once the level-wise distances between the conditional distributions are obtained. \cmnt{Let us first introduce CDD in Wasserstein sense.} 


For any $p \in [1, \infty)$, let $\mc W_p(\cdot, \cdot; D)$ denote the $p$-Wasserstein distance with cost function $D$, so that $\mc W_p^p$ denotes $\mc W_p$ taken to the power $p$. The Wasserstein distance represents the smallest possible cost to transport all the probability mass from one distribution to another given cost function $D$. See Appendix \ref{sec:app:bicausal_ot} for more details.

\begin{definition}[CDD in the Wasserstein sense]
\label{def:CDD_wass}
Let $\mathsf d(\cdot, \cdot)$ denote a distance between distributions, let $p \in [1, \infty)$, and let $\mathcal{L}$ be the support of $L$.  The conditional demographic disparity in the Wasserstein sense (\cddwtext) for model $f(X)$ with legitimate features $L$ and sensitive feature $A$ is  
%
\begin{align*}
    &\mathsf {CDD}^{\mathsf {wass}}(f) := \Wass^p_p(\Prob ({L | A= 0}), \Prob ({L| A= 1}) ; D), 
\end{align*}
where the cost function  $D: \mathcal L\times \mathcal L \mapsto \mathbb R$ is defined as 
\begin{align*}
D(l,l') =\begin{cases}
     \mathsf d \big(\Prob(f(X) \mid L=l, A = 0),\\ \qquad~\Prob(f(X) \mid L=l', A = 1)\big), \quad l=l',\\
    \infty, \qquad\qquad\qquad\qquad \qquad\qquad\text{elsewhere} .
   \end{cases}  
\end{align*}
\end{definition}

Different distance functions $\mathsf d$ induce different notions of disparity. In particular, when $\mathsf d$ itself is the Wasserstein distance $\Wass_p^p$, the CDD in the Wasserstein sense becomes a nested Wasserstein distance. 
This nested Wasserstein distance is tightly related to the \textit{bi-causal transport distance}, recently studied in the optimal transport literature. While this method of aggregating the level-wise disparities may seem unintuitive at first, we will see in Section \ref{section:our_methods} that the bicausal distance naturally captures the conditional independence relationships that define conditional demographic parity. 

\cmnt{Let us now introduce CDD in $\ell^p$ sense and discuss its characteristics. }

\begin{definition}[CDD in the $\ell_p$ sense]
\label{def:CDD_lp}
Let $\mathsf d(\cdot, \cdot)$ denote a distance between distributions, let $p \in [1, \infty)$, and let $\mathcal{L}$ be the support of $L$. Define $D: \mathcal L \mapsto \mathbb R$ as $D(l) = $
\begin{align*}
    \mathsf d \left(\Prob(f(X) \mid L=l, A = 0), \Prob(f(X) \mid L=l, A = 1)\right).
\end{align*}
%
%
The conditional demographic disparity in the $\ell_p$ sense (\cddlptext) of the model $f(X)$ is $\mathsf {CDD}^{\ell_p}(f) := \| D \| _{\ell^p \left(\mathcal L; \mb Q(L)\right)}$ where $\mb Q(L)$ is a probability measure defined over $\mc L$.
\end{definition}
The disparity in Definition \ref{def:CDD_lp} is the weighted $\ell_p$ norm of the distances between the conditional distributions defined by levels $l$ of the legitimate features, where the associated measure $\mb Q$ determines the weight assigned to each level.

In the definitions above, different values of $\mathsf d$, $p$, and $\mb Q$ induce different notions of disparity. \cmnt{For example, if $\mathsf d(\cdot, \cdot)$ is the Kolmogorov distance and $p=\infty$, then the conditional demographic disparity of $f(X)$ is $\varepsilon$ 
\cmnt{if for all $ l \in \operatorname{supp} \mathbb P ({L | A= 0}) \cap \operatorname{supp} \mathbb P ({L| A= 1})$,
\begin{align*}
    |\Prob({f(X)\le \tau|L=l,A=0})-\Prob({f(X)\le \tau|L=l,A=1})| \le \varepsilon, 
\end{align*}
or in other words}
if the Kolmogorov distance between $\Prob(f(X) \mid L = l, A = 0)$ and $\Prob(f(X) \mid L = l, A = 1)$ is no more than $\varepsilon$ for every possible $l$. \am{Check this is true for both definitions.} Setting $p = \infty$ and enforcing a small corresponding disparity requires the model $f(X)$ to behave similarly for the two groups within every level of the legitimate feature $L=l$. On the other hand, setting $p$ to, say, 2, allows $f(X)$ to potentially behave quite differently for the two groups over areas of $\mathcal{L}$ with small measure. }
There is a growing literature that investigates how different quantitative notions of (un)fairness capture or fail to capture various legal and philosophical notions of fairness \cite{friedler2021ImPossibilityFairness, bothmann2024fairnessroleprotectedattributes}. In Appendix \ref{subsec:app:aggregation}, we suggest some desiderata that are relevant to choosing these values, but we leave a fuller investigation for future work. In our regularizers and in our experiments below, we set $d$ to be the Wasserstein distance for both definitions, and we fix $p = 1$ for \cddlptext and $p = 2$ for the inner and outer Wasserstein distances in \cddwtext. We investigate several versions of $\mb Q(L)$ for \cddlptext.

\cmnt{Definition \ref{def:CDD_lp} of CDD leaves the choice of $\mathsf d$, $p$ and $\mb Q$ to the practitioner. In our discussions above, we discussed some choices that yield interpretable measures of CDD which is a desirable trait. However, in many real world scenarios, there are no choices that are clearly more preferable than others from an ethical or regulatory standpoint. On the other hand, in Definition \ref{def:CDD_wass}, only $\mathsf d$ needs to be chosen by the practitioner. This definition, however, offers less interpretability compared to the choices discussed above.\mg{check}}

\section{Regularized approaches to enforce CDP}\label{section:our_methods}

In light of the discussion in Section \ref{section:CDD}, a natural approach to enforce CDP in risk minimization problems is to employ the CDD measures in Definitions \ref{def:CDD_wass} and \ref{def:CDD_lp} as regularizers. In particular, given an i.i.d. training sample $\{(X_i, Y_i, A_i)\}_{i=1}^N$, we consider the following target problem:
\begin{align}\label{problem:CDD}
    \min_{f_\theta\in\mathcal{F}}\; \frac{1}{N} \sum_{i=1}^N g(f_\theta (X_i),Y_i) + \lambda~ \mathsf {CDD}(f_\theta),
\end{align}
where $g$ is a differentiable loss function, $\mathcal{F} = \{f_\theta:\theta \in \Theta \}$ is the class of models $f_\theta(X)$ under consideration, indexed by 
$\theta$, $\lambda>0$ is a penalty parameter to encourage fairness, and $\mathsf {CDD}$ represents either \cddwtext{} or \cddlptext.

\cmnt{In this section, we introduce two regularization-based methods for enforcing CDP based on the two CDD measures introduced in Section \ref{section:CDD}.} 

\cmnt{Our regularized CDP approaches provide practitioners a tunable knob for navigating fairness-performance trade-offs \cite{ZhaoInherent2019, kim2020FACTDiagnosticGroupa, liu2021AccuracyFairnessTradeoffs}. Moreover, these methods permit imposing closeness of the distributions rather than only closeness of the first moments. }

\subsection{Enforcing CDD in the Wasserstein sense} \label{subsec:enforcing_cdd_wass}




The CDD measure 
\cddwtext{} contains a discontinuous inner cost function and is non-differentiable; additionally, computing the Wasserstein distance exactly is known to be computationally intractable \cite{arjovsky2017wasserstein, salimans2018improving}. To tackle these challenges, we leverage an interesting connection between \cddwtext{} and the \textit{bi-causal transport distance}, recently studied in the optimal transport literature \cite{backhoff2017causal}. 
This allows us to find tractable approximations to \cddwtext{} such that the resulting minimization problem can be solved using gradient-based methods.

\subsubsection{Connecting \cddwtext{} to bi-causal transport distance} 
In order to define the bi-causal transport distance, we first need to define transport plans in general and bi-causal transport plans in particular. Consider two distributions $(U, V) \sim \Prob$ and $(\tilde{U}, \tilde{V}) \sim \tilde{\Prob}$ on a common measure space $\mathcal{U} \times \mathcal{V}$. The set $\Gamma(\tilde{\Prob},\P)$ of \emph{transport plans} denotes the collection of all probability measures on the space $(\mathcal{U} \times \mathcal{V}) \times (\mathcal{U} \times \mathcal{V})$ with marginals $\tilde{\Prob}$ and $\Prob$. The set $\Gamma_{bc}( \tilde{\Prob},\P) \subset \Gamma( \tilde{\Prob}, \P)$ of \emph{bicausal transport plans} is given by
\begin{align*}
    &\Gamma _{bc} (\tilde{\Prob}, \Prob) = \{ \gamma \in \Gamma (\tilde{\Prob}, \Prob) \text{ s.t. for } ((\tilde{U}, \tilde{V}), (U, V)) \sim \gamma, \\
        &\qquad \qquad \qquad\qquad \qquad\quad\;\; U \indep \tilde{V} \mid \tilde{U} \text{ and } \tilde{U} \indep V \mid U\}.
\end{align*}
We are now ready to define the bi-causal transport distance.

\begin{definition}[Bi-causal transport distance (BCD)] \label{def:bicausal_distance}
    The \textit{bi-causal transport distance} (BCD, referred to hereafter simply as the \textit{bi-causal distance}) between $\tilde{\Prob}$ and $\Prob$, denoted by $\Causal^p (\tilde{\Prob}, \Prob)$, is defined as
    \begin{align*}
        \inf _{\gamma \in \Gamma _{bc} (\tilde{\Prob}, \Prob)} \E _{((\tilde{U}, \tilde{V}),(U, V)) \sim \gamma} \Big[  C \norm{\tilde{U}-U}^p + \norm{\tilde{V}-V}^p
    \Big].
   \end{align*}
\end{definition}  

See Appendix \ref{sec:app:bicausal_ot} for more background on bi-causal transport. In our setting, $U$ will correspond to the legitimate feature $L$, and $V$ will correspond to the model output $f(X)$. The presence or absence of the tilde corresponds to the two levels of the sensitive feature $A$.  The following result connects BCD to Definition \ref{def:CDD_wass} of CDD. 
\begin{proposition}\label{prop:bcd_cdd}
     Consider a model $f:\mc X \to \mc Y$. Let $\mathcal{L}$ be the support of the legitimate feature $L$. 
     Let $\underline d= \min_{l,l'\in \mc L} \|l-l'\|_p^p$ denote the minimum distance between two levels of the legitimate features. 
     Moreover, let $\bar d =  \max_{x,x'\in \mc X} \|f(x)-f(x')\|_p^p$ denote the diameter of $\mc Y$. Then, the bi-causal distance $\Causal^p(\P ({f(X), L | A= 0}) , \P ({f(X), L | A= 1}))$ with $C>\frac{\bar d}{\underline d}$ is equivalent to $\mathsf {CDD}^{\mathsf {wass}}_f$ 
     with $\mathsf d(\cdot,\cdot)=\Wass_p^p(\cdot,\cdot)$.
\end{proposition}
See Appendix \ref{sec:app:proofs} for the proof. This result implies that for $C>\frac{\bar d}{\underline d}$, the bi-causal distance is equivalent to \cddwtext{} with the inner metric $\mathsf d(\cdot,\cdot)$ equal to $\Wass_p^p(\cdot,\cdot)$. 
We describe our BCD-regularized approach next.

\paragraph{F\MakeLowercase{air}B\MakeLowercase{i}T: Conditional Fairness through Bi-causal Transport} \label{paragraph:fairbit}

For features $X$, legitimate features $L \subset X$, and sensitive feature $A$, we define the regularizer
\begin{align*}
    \mathcal{B}(f):= \Causal^p \left(\Prob({L, f(X) | A=0}), \Prob({L,f(X) | A=1})\right),
\end{align*}
the bi-causal transport distance between $\Prob({L, f(X) | A=0})$ and  $\Prob({L,f(X) | A=1})$, with $C>\frac{\bar d}{\underline d}$ as described in proposition \ref{prop:bcd_cdd}. 
The bi-causal distance can be viewed as a nested Wasserstein distance (Proposition \ref{prop:nested} in the Appendix), one that takes a different form from the nested Wasserstein distance expressed in Definition \ref{def:CDD_wass} but which is equivalent under the conditions given in Proposition \ref{prop:bcd_cdd}. The reformulated nested Wasserstein expression contains a smooth inner cost function; this enables us to estimate the BCD using a nested version of the Sinkhorn divergence \cite{sinkhorn1964relationship, cuturi2013sinkhorn, pichler2021nested}, which is an entropy-regularized version of the Wasserstein distance. Let us denote this estimate of the BCD by $\hat{\mc{B}}(f)$.
Our proposed approach, \textit{Fairness through Bi-causal Transport (\bcd)}, aims to solve the version of Problem \eqref{problem:CDD} with $\mathsf {CDD}(f_\theta)$ set to $\widehat{\mc{B}}(f_{\theta})$.
%

Since $\widehat{\mathcal{B}}(f_\theta)$ is differentiable in $\theta$, this problem is amenable to gradient-based solvers.
The computational complexity of the \bcd{} regularizer is $\mathcal{O}(n^2 + |L|^2) = \mathcal{O}(n^2)$, where $|L| \leq n$ is the number of levels of $L$ observed in the sample. \cmnt{Hence, \bcd{} inherits a $\mathcal{O}(n^2)$ computational complexity with respect to the sample size $n$. The total computational complexity is however $\mathcal{O}(n^2)$, as the complexity only increases additively in the number of levels of the legitimate features in the observed samples $|L|$ --- in other words, $\mathcal{O}(|L|^2 + n^2) = \mathcal{O}(n^2)$;} See Appendix~\ref{app:bicausal-dist} for further details as well as the nested Sinkhorn divergence algorithm.

\subsection{Enforcing CDD in the $\ell_p$ sense} \label{subsec:enforcing_cdd_lp}

The regularization term $\mathsf{CDD}(f_{\theta})$ in this case takes the form $\mathsf {CDD}^{\ell_p}({f_{\theta}}) = \| D \| _{\ell^p \left(\mathcal L; \mb Q(L)\right)}$ as described in Definition \ref{def:CDD_lp}.
Our proposed method based on \cddlptext{} is called \textit{\fairlp: Conditional Fairness in the $\ell_p$ sense}. The parameters $p$ (the order of the $\ell_p$ norm) and $\mb Q(L)$ (the probability measure over $\mc L$) determine how the level-wise disparities are aggregated. 
We particularly highlight three aggregation strategies, which result in three variants of \fairlp{}:
\begin{enumerate}
    \item \textit{Fairleap (uniform)}: A simple average with $\mb Q = \mathbb{U}(L)$, which we use from here forward to denote the uniform distribution over $L$: This puts equal emphasis on every observed level of $L$. 
    \item \textit{Fairleap ($\P(L)$)}: A weighted average with $\mb Q = \P(L)$ and $p=1$: This prioritizes levels of $L$ with more mass.
    \item \textit{Fairleap (Ave. $\P(L|A)$)}: A weighted average with $\mb Q = \frac{\Prob(L \mid A = 0)+\Prob(L \mid A = 1)}{2}$ and $p=1$: This prioritizes levels of $L$ with more mass, within either class of $A$, and avoids favoring the majority class. 
\end{enumerate}
In practice, the unknown distribution $\Prob$ is replaced with the empirical distribution. A main advantage of these choices of $p$ and $\mb Q$ is that they yield interpretable regularizers. In terms of the inner distance function $\mathsf d$, the definition of \cddlptext{} is generic and admits any distributional distance. 
In our implementations, we choose Wasserstein distance due to the known connection between closeness of distribution in Wasserstein sense and parity of performance in downstream tasks \cite{villani2009optimal, santambrogio2015optimal, xiong2023fair}. 
Similar to \fairbit, here we estimate the Wasserstein distance by employing the Sinkhorn divergence, which is differentiable. 
The computational complexity of the \fairlp~regularizer is $\mathcal{O}(\frac{n^2}{|L|})$. See Appendix~\ref{app:bicausal-dist} for further details. 


%

\section{Discussion of Existing Methods} \label{section:comparison}
\cmnt{In this section, we discuss existing approaches for enforcing CDP. Except for DCFR \cite{xu2020algorithmic} which has been designed with the intent of enforcing CDP, 

We also discuss adaptations of two canonical DP approaches for CDP, namely Pre-processing Repair \cite{Feldman2015disparate} and Adversarial Debiasing \cite{zhang2018mitigating}.  We show that while Pre-processing Repair can  be used for enforcing an approximation of CDP, our modification of Adversarial Debiasing is a reasonable approach for promoting CDP.}

\cmnt{\mg{How about "In this section, we discuss the existing methods that appear as reasonable candidates for enforcing CDP and argue the need for our novel approach by stating the shortcomings of the existing art."} 
}

In this section, we discuss the methods that we experimentally compare to FairBiT and FairLeap in the next section. We analyze the state of the art method for CDP with rich legitimate features and illustrate why it may not in fact minimize the disparity of interest. In order to widen the scope of our comparisons, we modify an existing approach for DP to apply to CDP, and we consider under what circumstances another existing approach for DP may result in (approximate) CDP without modification. Finally, we discuss an existing method for DP that is computationally closest to ours. Table \ref{tab:comparison} compares our methods to all these methods.

\begin{table*}[!t]
\centering
\begin{tabular}{c|c|c|c|c|}
\cline{1-5}
\multicolumn{1}{|c|}{Methods}                  & \begin{tabular}[c]{@{}c@{}}Perf.-fairness\\ tradeoff param\end{tabular} & Targets CDP & \begin{tabular}[c]{@{}c@{}}No sensitive feat.\\ (training)\textsuperscript{\textdagger}\end{tabular} & \begin{tabular}[c]{@{}c@{}}No sensitive feat.\\ (inference)\end{tabular} \\ \hline
\multicolumn{1}{|c|}{DCFR} 
& \cmark                                                                            & \xmark$^*$          & \xmark                                                                     & \xmark\textsuperscript{\textdaggerdbl}                                                                    \\ \hline
\multicolumn{1}{|c|}{Pre-processing repair} 
& \cmark                                                                            & \xmark        & \xmark                                                                     & \xmark                                                                   \\ \hline
\multicolumn{1}{|c|}{Adversarial debiasing} 
& \xmark                                                                             & \cmark$^{**}$        & \cmark                                                                      & \cmark                                                        \\ \hline
\multicolumn{1}{|c|}{Wasserstein regularization} 
& \cmark                                                                            & \xmark          & \cmark                                                                      & \cmark                                                        \\  \hline
\multicolumn{1}{|c|}{\fairbit{} \& \fairlp~(ours)}          & \cmark                                                                            & \cmark         & \cmark                                                                      & \cmark                                                        \\ \hline
\end{tabular}
\caption{Comparison of \bcd{} and \fairlp{} to the methods described in Section \ref{section:comparison}. $^*$DCFR is intended to target CDP but targets a proxy quantity that does not necessarily yield small CDD values (Section \ref{sec:DCFR}). $^{**}$Adversarial debiasing here refers to our modified version in which we pass $(f(X), L)$ to the adversary instead of just $f(X)$. \textsuperscript{\textdagger}For adv. debiasing, only the adversary requires access to the sensitive feature $A$. For Wasserstein reg. and FairBiT, the gradients for the regularization term may be computed separately on each iteration and passed to the analysts who are training the model $f(X)$. This is important for example in a financial setting where teams with access to sensitive features are distinct from model developers. \textsuperscript{\textdaggerdbl}DCFR can be easily modified not to require the sensitive feature at inference time, though this will in general result in a decrease in model performance.}\label{tab:comparison}
\end{table*}

\subsection{State of the art: DCFR}\label{sec:DCFR}

The current state of the art for CDP with rich legitimate features is the approach proposed in \citet{xu2020algorithmic}. This method, named Derivable Conditional Fairness Regularizer (DCFR), aims to learn a representation $Z$ of the input features such that $Z \indep A \mid L$, from which it follows that a model $f(Z)$ will satisfy CDP. The method utilizes an adversarial learning approach with a regularizer defined as $R_{\text{DCFR}} (Z,L,A) := \sup_{h\in H_{ZF}} Q(h)$, where $H_{ZF}$ is the set of all bounded square-integrable functions with values in $[0,1]$, and
\begin{align}
     Q(h) :=  &\E[\mathbf{1}_{\{A=1\}}\Prob(A=0|L)h(Z,L)]-\nonumber \\
     &\qquad\quad\E[\mathbf{1}_{\{A=0\}}\Prob(A=1|L)h(Z,L)]. \label{eq:dcfr}
\end{align}
$Q(h)$ is motivated by a characterization of conditional independence given by \citet{daudin1980partial}. However, Proposition~\ref{prop:dcfr_alternative_form} shows that $R_{\text{DCFR}}$ aims to minimize a quantity that does not uniformly bound the CDD (in the sense of either Definition \ref{def:CDD_wass} or \ref{def:CDD_lp}). This is contrast with our proposed methods, which aim to directly minimize the CDD. 

\cmnt{
However, Proposition~\ref{prop:dcfr_alternative_form} shows that $R_{\text{DCFR}}$ does not minimize CDD \mg{we need to mention what notion of CDD we are talking about here.} directly, but rather targets a related proxy quantity. \mg{It is not even a valid proxy bcz small values doesnt mean small cdd} This is in contrast with our proposed regularizers, which directly target CDD. 
}

\begin{proposition}\label{prop:dcfr_alternative_form}
The regularizer proposed by \cite{xu2020algorithmic} can be equivalently expressed as
\begin{equation}
R_{\text{DCFR}}= \E[ ( \Prob(A=1|Z,L) - \Prob(A=1|L) )_+ ].
\end{equation}
\end{proposition}
See Appendix \ref{sec:app:proofs} for the proof of Proposition \ref{prop:dcfr_alternative_form}. Note that the value of this regularizer depends on which level of the sensitive feature is labeled as 1. Furthermore, while CDP holds if $R_{\text{DCFR}} = 0$, small values of $R_{\text{DCFR}}$ do not necessarily imply small values of CDD. Figure \ref{fig:res:Xucomp_simple_example_body} illustrates this vis-a-vis FairBiT via a simple synthetic loan example in which we vary both the proportions of males vs. females applying for loans and the respective acceptance rates. (In this simple example, $L = \emptyset$, so \cddwtext = \cddlptext, and the y-axis is therefore simply labeled ``$\mathsf {CDD}$.'') Each line in Figure~\ref{fig:res:Xucomp_simple_example_body} is obtained by changing acceptance rates for a fixed proportion of male vs. female applicants; the more unbalanced this proportion, the steeper the CDD-$R_{\text{DCFR}}$ curve gets, while \bcd~ enjoys a consistent relationship with CDD regardless of this proportion.
The same is true for \fairlp; details and further analysis are given in Appendix~\ref{sec:illustrative-example-dcfr}.

\begin{figure}[!ht]
    \centering
    \begin{minipage}{.235\textwidth}
        \centering
        \includegraphics[width=\linewidth]{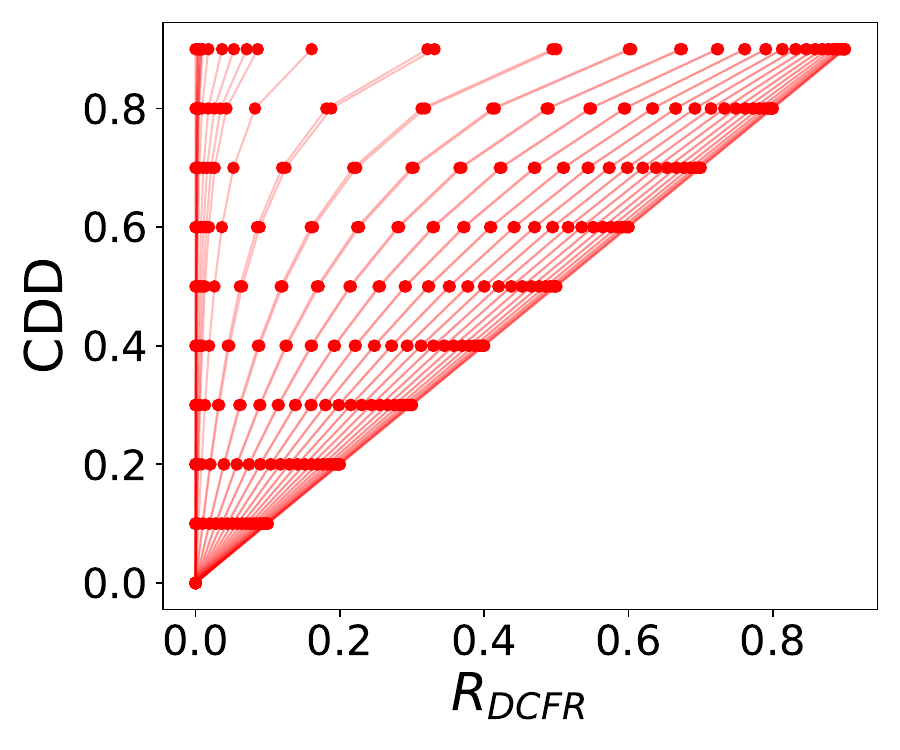}
    \end{minipage}%
    \begin{minipage}{0.235\textwidth}
        \centering
        \includegraphics[width=\linewidth]{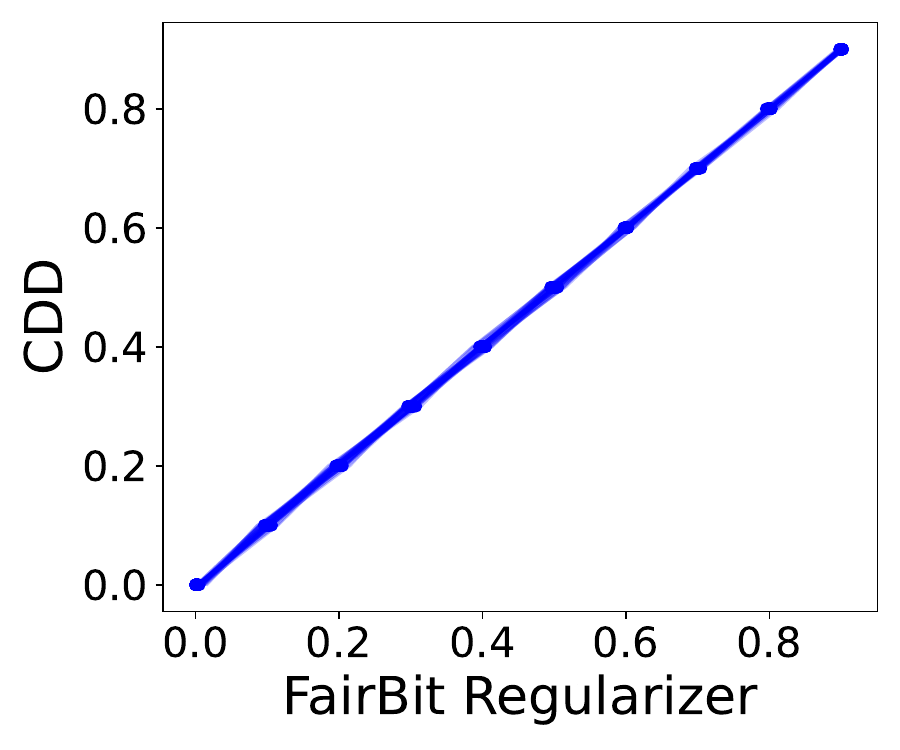}
    \end{minipage}
    \caption{Conditional demographic disparity ($\mathsf {CDD}$) versus $R_{\text{DCFR}}$ (left) and the value of FairBiT regularizer (right) in a synthetic loan setting, varying the proportion of males vs. females and the loan acceptance rates. The proportion of males vs. females controls the slope of the CDD-$R_{\text{DCFR}}$ curve, with higher ratios yielding steeper curves. A 45 degree line only occurs if the ratio of males to females is 1:1.}
    \label{fig:res:Xucomp_simple_example_body}
\end{figure}

\subsection{Modifying Adversarial Debiasing for CDP}\label{sec:modified-DP-approaches}



Adversarial debiasing \cite{zhang2018mitigating} employs an adversarial training framework to target DP (as well as equality of odds and equality of opportunity). In the case of DP, the model $f(X)$ is trained to predict the target $Y$ from the input features $X$, while the adversary is trained to predict the sensitive feature $A$ from the outcome of the model $f(X)$. We modify their method to provide the adversary not only with $f(X)$ but also the legitimate features $L$. This is motivated by the following remark.
\begin{remark}\label{remark:weak_union} It follows from the weak union property of conditional independence that if $(L, f(X)) \indep A$, then $f(X) \indep A\mid L$, i.e. 
CDP is satisfied. 
\end{remark}
If the adversary is unable to predict $A$ from $(f(X), L)$, then, by Remark \ref{remark:weak_union}, $f(X)$ will satisfy CDP.  If $L$ is correlated with $A$, then this will not happen in general. \cmnt{However, this is not necessarily an issue since the condition $(L, f(X)) \indep A$ is stronger than CDP.} When $L$ is correlated with $A$, the learning procedure may encourage $f(X) \indep A \mid L$ (i.e. CDP) in order to minimize the amount of information the adversary is able to obtain from $(f(X), L)$. Hence, this approach may effectively target CDP. We note however that this method does not include a parameter to control the fairness-performance trade-off\cmnt{and the choice of the adversary architecture is an additional hyper-parameter that impacts the success of this approach \sout{Once again, this may happen, but it is not guaranteed to happen in general. Suppose as an extreme example that $L$ is perfectly correlated with $A$. Then the adversary will always be able to predict $A$ regardless of $f(X)$, and CDP will not hold.}
In cases where $L$ is correlated with $A$, adversarial learning will encourage $f(\tilde{X}) \indep A \mid L$ as well as $L \indep A \mid f(\tilde{X}) $ and $f(\tilde{X}) \indep A$. 
This approach however does not include a parameter to control the fairness-performance trade-off}, and the choice of the adversary architecture is an additional hyper-parameter that impacts the success of the approach.

\cmnt{Recall that $L \subset X$, and $A \not\subset X$. If the features $X$ are independent of the sensitive feature $A$, then the condition of the remark is trivially satisfied, so CDP is satisfied. Of course, this is not the case in general, but $X$ may be transformed to be orthogonal to $A$. Since the condition $(L, f(X)) \indep A$ is stronger than CDP, it is harder to satisfy in general, but nevertheless this observation leads to two methods which may effectively target CDP in some circumstances.}

\subsection{Employing DP methods for CDP without modification}\label{sec:dp-methods-no-mods}

For a more extensive empirical comparison, we consider how two existing methods designed for demographic parity may be used/modified to approximately target conditional demographic parity. 

\paragraph{Pre-processing repair} This method proposed by Feldman et al. \cite{Feldman2015disparate} is a pre-processing method which utilizes a transformation or ``repair'' $X \mapsto \tilde{X}$ such that $\tilde{X}$ is approximately independent of $A$. \cmnt{, at least with respect to the empirical measure in the training data.} Since $L \subset X$, it follows from Remark \ref{remark:weak_union} that $f(\tilde{X})$ is approximately independent of $A$ conditional on $\tilde{L}$. Note that $f(\tilde{X}) \indep A \mid \tilde{L} \centernot\implies f(\tilde{X}) \indep A \mid L$, but if the transformation $L \mapsto \tilde{L}$ happens to be minimal, then this method may result in approximate CDP. This method includes a \textit{repair level} parameter in $[0, 1]$, where 0 indicates no transformation, 1 indicates full orthogonalization, and values in between represent degrees of repair.


\paragraph{Wasserstein regularization} 
We also investigate a method which penalizes the loss function with a Wasserstein distance penalty $\mathcal{W}_p^p(\Prob(f(X) \mid A = 0), \Prob(f(X) \mid A = 1))$ computed via the Sinkhorn algorithm \cite{cuturi2013sinkhorn}, which we refer to simply as \textit{Wasserstein regularization}. This method, which targets DP, is the closest approach in the literature when it comes to regularization-based methods using optimal transport to enforce fairness \cite{rychener2022metrizing}, which is our motivation for including it in our comparisons. However, as illustrated in Table \ref{tab:synthetic_example} in the Introduction, DP does not imply CDP, and small values of this disparity do not imply small values of CDD.

\section{Experiments} \label{section:experiments}

\begin{figure*}[!t]
    \centering
    \begin{subfigure}{.225\textwidth}
        \centering
        \includegraphics[width=\linewidth]{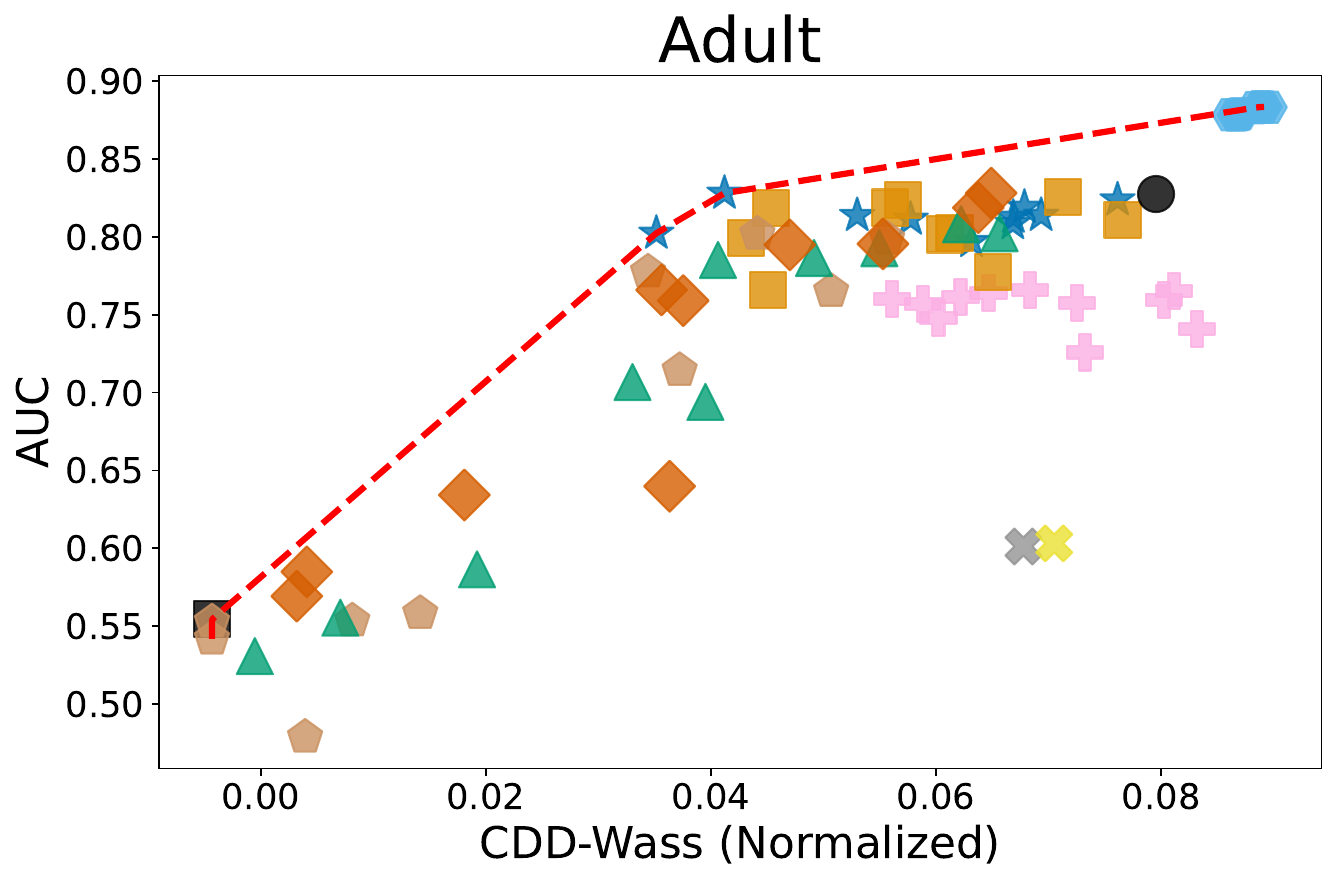}
    \end{subfigure}%
    \begin{subfigure}{0.225\textwidth}
        \centering
        \includegraphics[width=\linewidth]{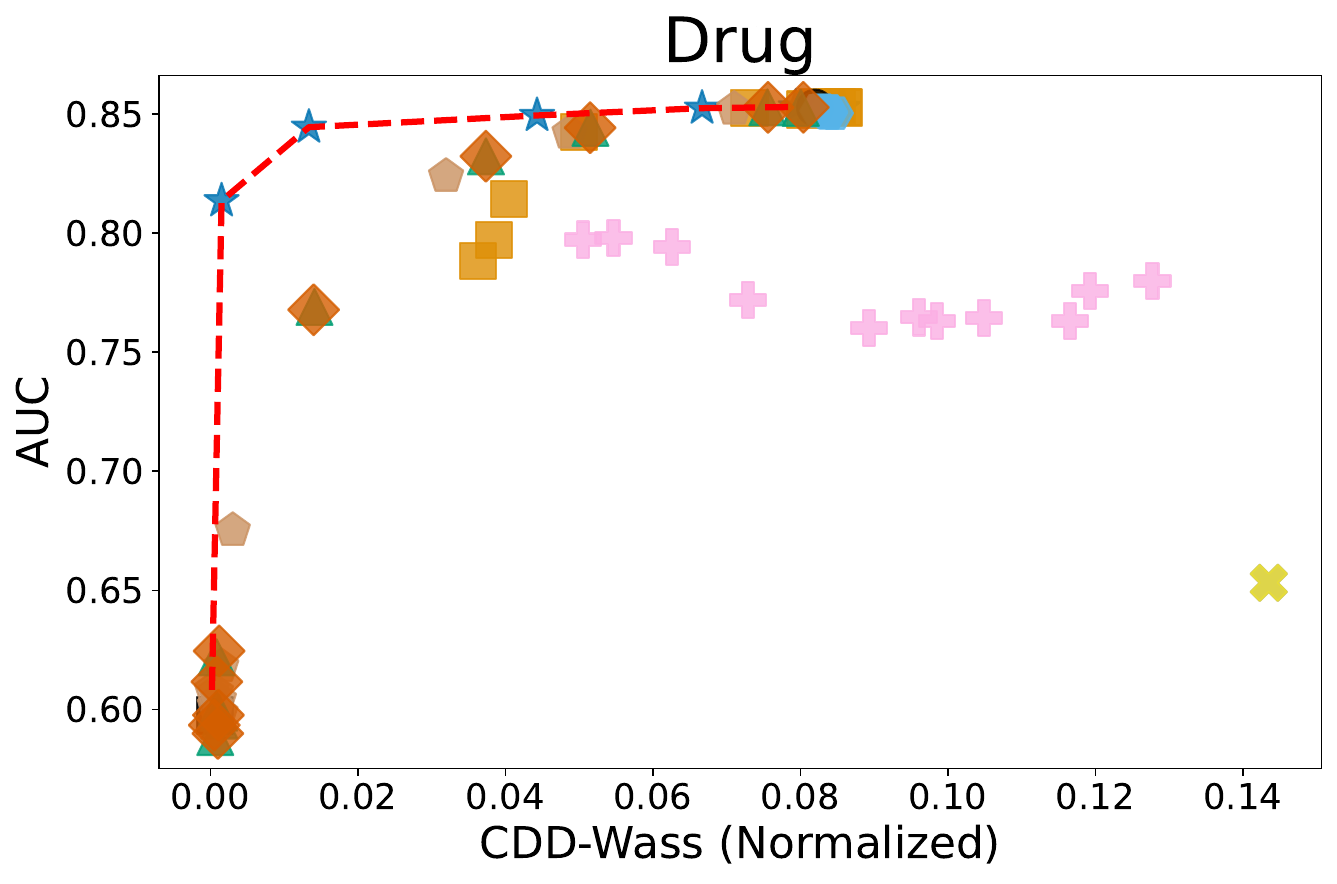}
    \end{subfigure}
    \begin{subfigure}{.23\textwidth}
        \centering
        \includegraphics[width=\linewidth]{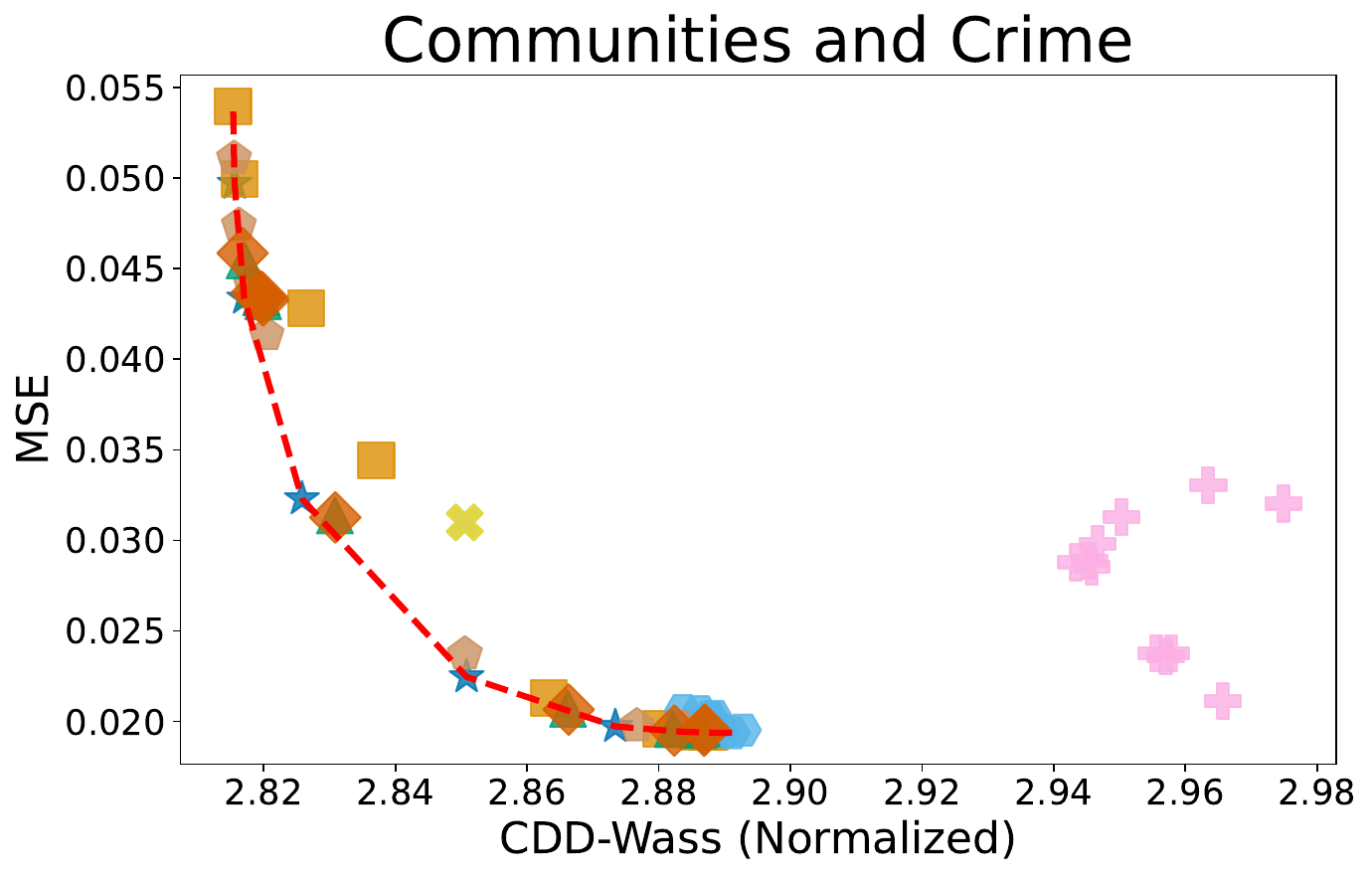}
    \end{subfigure}%
    \begin{subfigure}{0.315\textwidth}
        \centering
        \includegraphics[width=\linewidth]{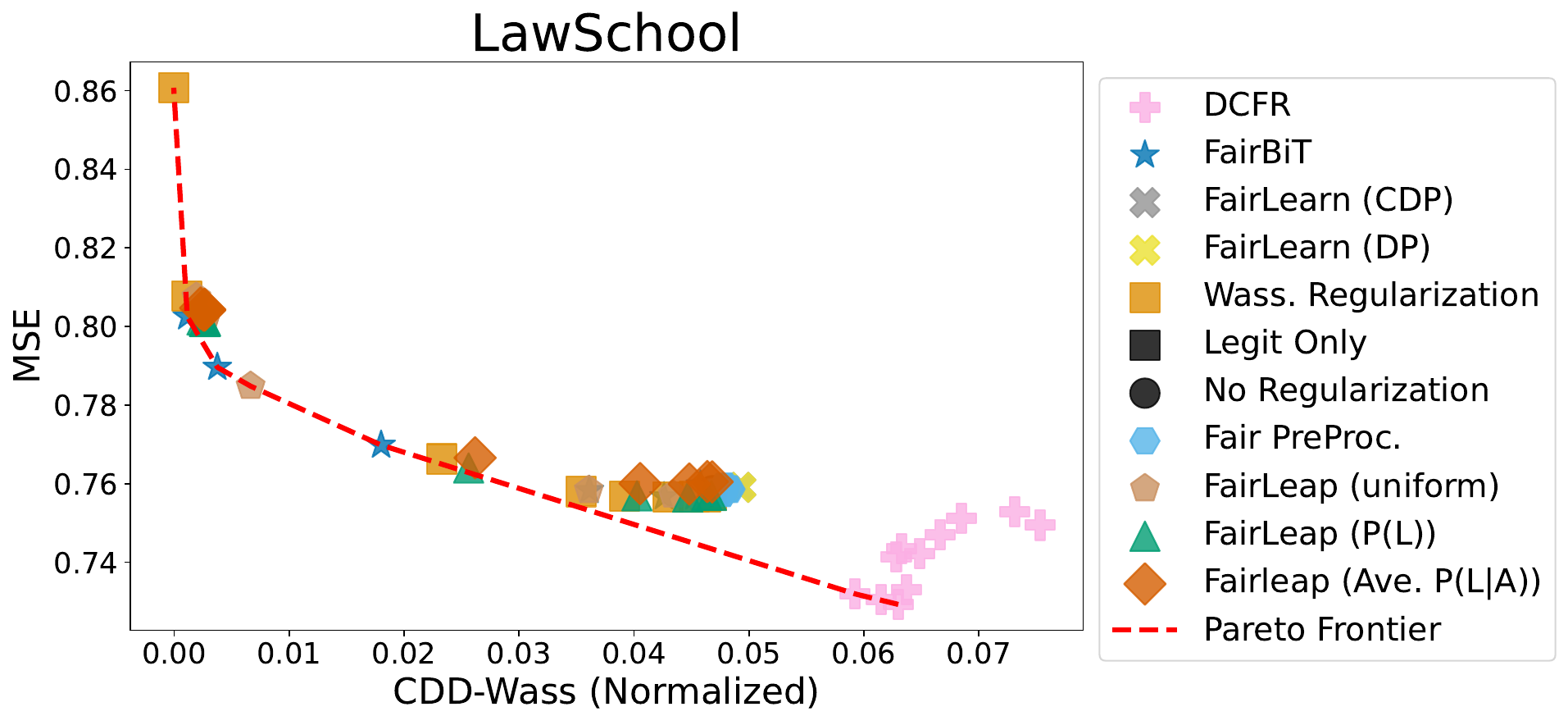}
    \end{subfigure}

    \begin{subfigure}{.225\textwidth}
        \centering
        \includegraphics[width=\linewidth]{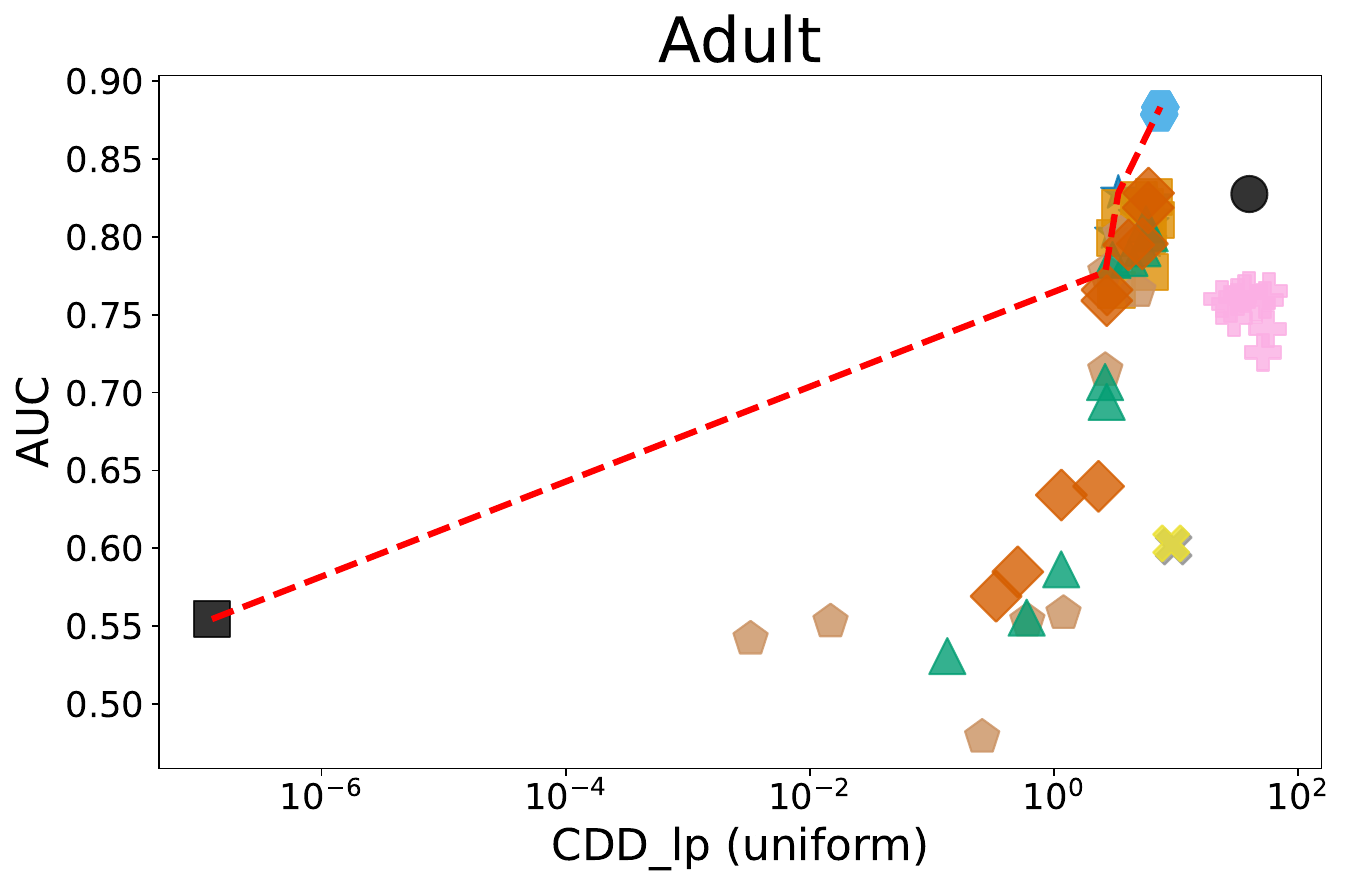}
    \end{subfigure}%
    \begin{subfigure}{0.225\textwidth}
        \centering
        \includegraphics[width=\linewidth]{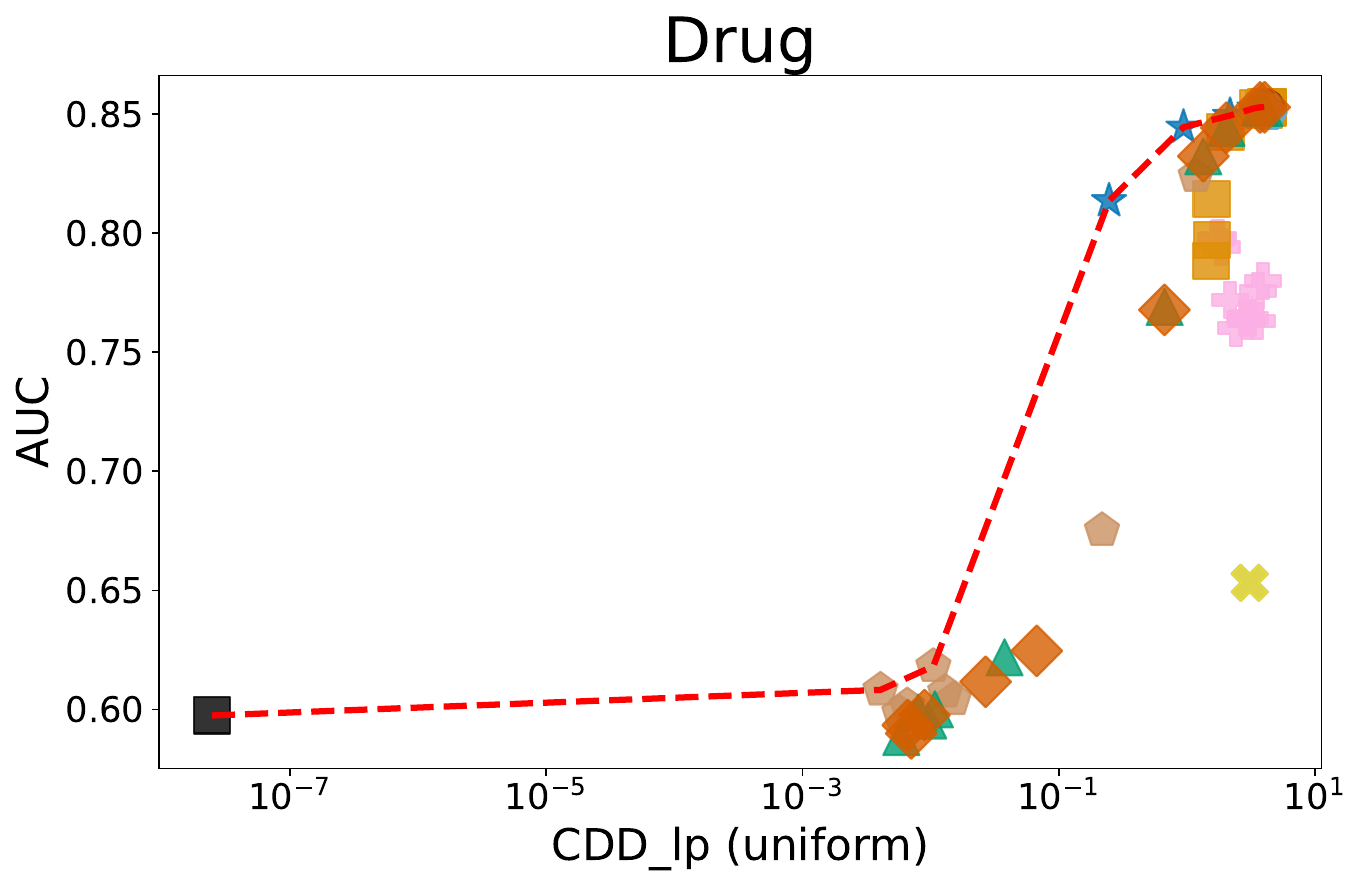}
    \end{subfigure}
    \begin{subfigure}{.23\textwidth}
        \centering
        \includegraphics[width=\linewidth]{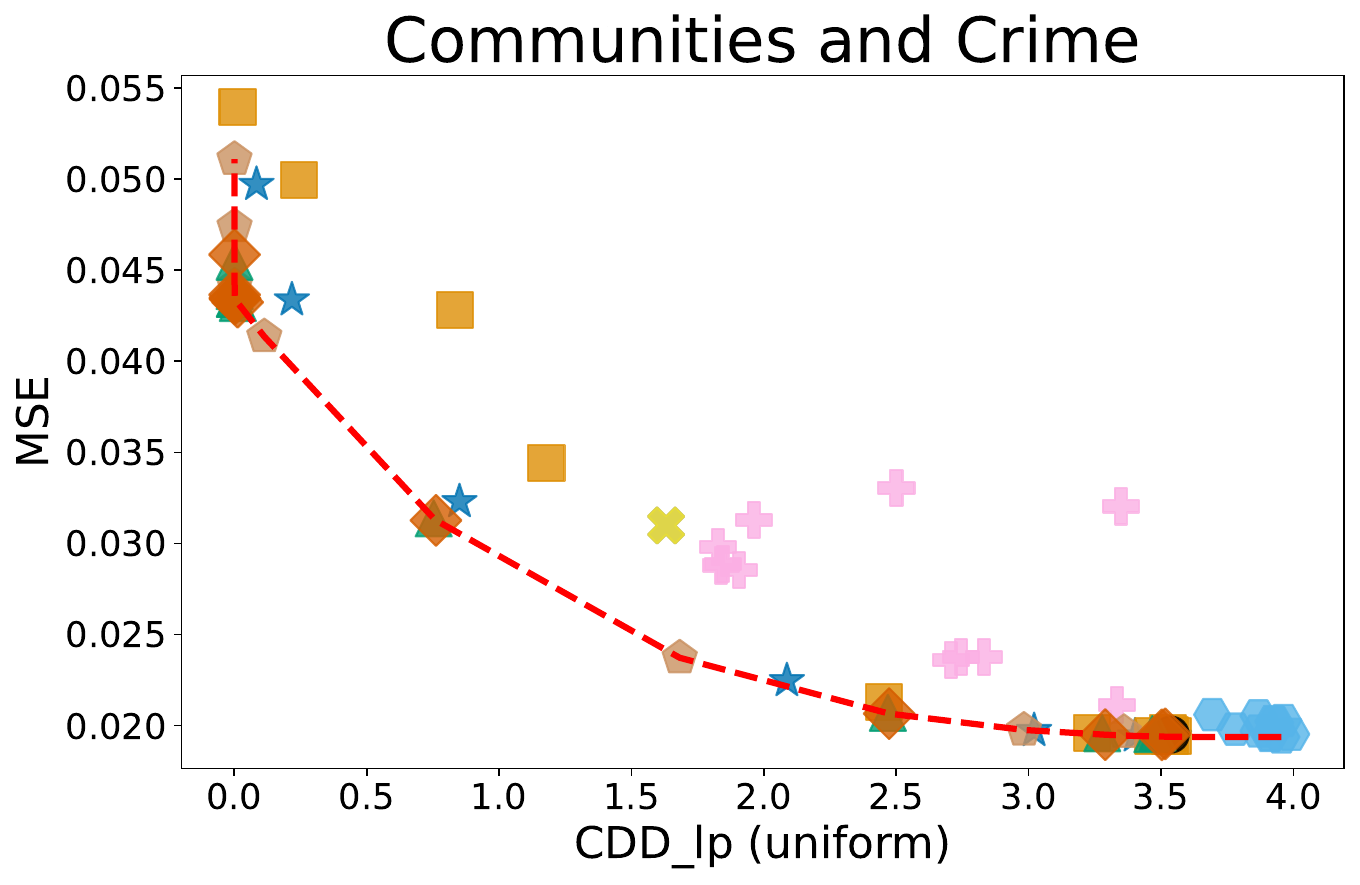}
    \end{subfigure}%
    \begin{subfigure}{0.315\textwidth}
        \centering
        \includegraphics[width=\linewidth]{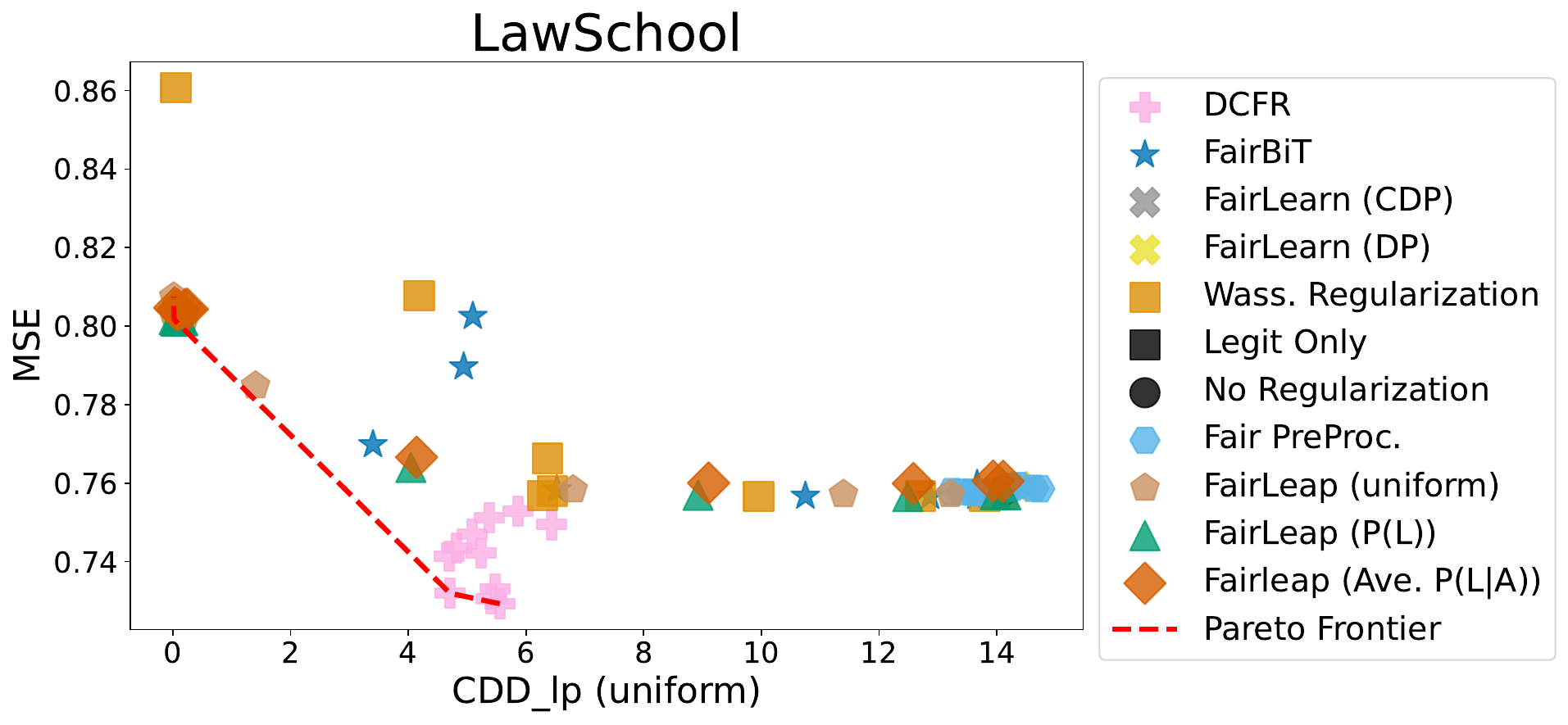}
    \end{subfigure}
    \caption{
    Fairness-predictive power trade-offs and Pareto frontiers for the four real datasets. \textit{Top row:} Trade-offs when measuring fairness using a \cddwtext{} (we use a `normalized'' version for presentation purposes. See Appendix \ref{sec:app:exp_details} for details). \textit{Bottom row:} Trade-offs when measuring fairness using \cddlptext{} (with $\mb Q(L) = \mb U(L)$ and $p=1$). Predictive power (PP) is measured by AUC for classification tasks (first and second columns from left) and MSE for regression tasks (third and fourth columns from left). Results are averaged over 10 runs, with multiple points per model due indicating different hyper-parameter values. Overall, \bcd\ and the variants of \fairlp{} are consistently part of the Pareto frontier, hence providing better fairness-PP trade-offs than many of the other proposed methods. See text and Appendix~\ref{sec:app:exp_details} for more details.}
    \label{fig:empirical-comparison-results}
\end{figure*}





In this section, we compare the effectiveness of the methods described in Sections~\ref{section:our_methods} and \ref{section:comparison} on four real datasets commonly used in the fairness literature \cite{fabris2022algorithmic}.  

\subsection{Setup} \label{sec:exp_setup}

We include two classification tasks on the \texttt{Drug} \cite{misc_drug_consumption_(quantified)_373} and \texttt{Adult} \cite{misc_adult_2}~datasets; and two regression tasks on the \texttt{Law School}  \cite{ramsey1998lsac} and \texttt{Communities and Crime} \cite{misc_communities_and_crime_183}~ datasets. 
In all experiments, we use a multi-layer perceptron (MLP) with two hidden layers containing 50 and 20 nodes and a rectified linear unit (ReLu) activation function. The loss function is set to be cross-entropy for classification (after a softmax activation) and mean squared error (MSE) for regression. We measure the \emph{predictive power (PP)} of each method by the area under the ROC curve (AUC) for the classification tasks and MSE for the regression task.
For \cddlptext{}, we report results for $\mb Q(L) =\mb U(L)$, the uniform distribution over $L$. (We include the results for the variants with $\mb Q(L) =\P(L)$ and $\mb Q(L) =\frac{\P(L|A=0)+\P(L|A=0)}{2}$  in Appendix~\ref{subsec:app:additional_results}). 
We report the mean over $10$ runs for every method-hyperparameter combination; we omit the hyperparameter labels in the figures for clarity. See Appendix \ref{sec:implementation-details-app} for more details on hyperparameter values and implementation details.

We compare the following methods, along with two baselines for reference:

\begin{itemize}
    \item \fairbit{} (Section~\ref{subsec:enforcing_cdd_wass}),
    \item 
    \textit{\fairlp{} (uniform)}, \textit{\fairlp{} ($\P(L)$)}, and \textit{\fairlp{} (Ave. $\P(L|A)$)} (Section~\ref{subsec:enforcing_cdd_lp}),
    \item \textit{DCFR} (Section~\ref{sec:DCFR}),
    \item \textit{Adversarial debiasing} (Section~\ref{sec:modified-DP-approaches}),
    \item \textit{Pre-processing repair} (Section~\ref{sec:dp-methods-no-mods}),
    \item \textit{Wasserstein regularization} (Section~\ref{sec:dp-methods-no-mods}),
    \item \textit{No regularization}, i.e., model training without enforcing any fairness constraint,
    \item \textit{Legitimate only}, i.e., model training using only the legitimate feature $L$.
\end{itemize}

As DCFR and pre-processing repair require access to the sensitive feature, we make the sensitive feature available during training to all methods for the purpose of comparing downstream performance. However, we note that our proposed methods \fairbit{} and \textit{\fairlp{}}, as well as the adversarial debiasing method, do not necessarily require direct access to the sensitive feature during training (see Table 2).
We refer the reader to Appendix~\ref{sec:app:exp_details} for more comprehensive results, including means, standard deviations, details on hyper-parameter values, and evaluations of demographic disparity.

\subsection{Results} \label{sec:exp_results}


Figure~\ref{fig:empirical-comparison-results} reports the fairness-predictive power (referred to hereafter as fairness-PP) trade-offs for the four real datasets. We report conditional demographic disparity in the Wasserstein sense with $p = 2$ (Figure~\ref{fig:empirical-comparison-results}, top) as well as in the $\ell_p$ sense with $\mb Q(L) = \mb U(L)$ and $p=1$ (Figure~\ref{fig:empirical-comparison-results}, bottom).
Each point represents a method-hyperparameter combination; we refer the reader to Appendix~\ref{sec:app:exp_details} for more details. 

Overall, \bcd\ and the three variants of \fairlp{} are consistently among the highest performing methods, generally providing better fairness-PP trade-offs than the other methods, as indicated by the Pareto frontiers. When evaluated by \cddwtext{} (Figure~\ref{fig:empirical-comparison-results}, top row), \fairbit{} has points on the Pareto frontier for every experiment and generally outperforms all other methods. Similarly, when \cddlptext{} with $\mb Q(L) = \mb U(L)$ and $p=1$ is the measure (Figure~\ref{fig:empirical-comparison-results}, bottom row), the method corresponding to this CDD measure, namely \textit{\fairlp{} (uniform)}, has points on the Pareto frontier for all 4 datasets. 
Overall, we observe that whether CDD is measured by \cddwtext{} or \cddlptext{}, all our proposed methods are also among the best performing.
Moreover, these results show that our regularized methods provide a wide range of points in fairness-PP space, allowing practitioners to choose their desired tradeoff point. 

We summarize the performance of the other methods in our experiments as follows.
\begin{itemize}
    \item Both versions of \textit{Adversarial Debiasing} show poor predictive performance and conditional fairness on the \texttt{Adult}, \texttt{Drug}, and \texttt{LawSchool} datasets. On the \texttt{Communities and Crime} dataset, they slightly improve CDD levels but still underperform \fairbit{} and \fairlp{} at similar levels of MSE.
    \item \textit{Preprocessing repair} achieves the best predictive performance levels (on par with the model with no regularization), but fails to enforce conditional fairness.
    \item \textit{DCFR} performs well on \texttt{LawSchool} when CDD is measured by \cddlptext{}, but is not part of the Pareto frontier in the other datasets.
    \item \textit{Wasserstein regularization} performs well in terms of enforcing conditional fairness in regression tasks, but this comes at the cost of higher MSE than \fairbit{} and \fairlp{} variants for the same CDD values. In classification tasks, it provides reasonable fairness-PP trade-offs, often close to or on the Pareto frontier, albeit providing worse trade-offs than \fairbit{} and variants of \fairlp{}.
    \item \textit{Legitimate only}, as expected, achieves full CDP but suffers significantly in terms of predictive performance. For presentation purposes, we have omitted \textit{legit-only} from Figure \ref{fig:empirical-comparison-results}, and present it in Figure \ref{fig:empirical-comparison-results_including_legit} in Appendix~\ref{sec:app:exp_details}. 
\end{itemize}

\section{Conclusion}\label{section:conclusion}

In this work, we propose novel measures to quantify the violation of conditional demographic parity (CDP), or conditional demographic disparity (CDD), in the Wasserstein sense and in the $\ell_p$ sense, based on distributional distances borrowed from the optimal transport literature. We design regularization-based approaches to enforce CDP based on these two measures, \fairbit{} and \fairlp{}. Our methods provide tunable knobs for navigating fairness-performance trade-offs, and they can be applied even when the conditioning variable has many levels. When model outputs are continuous, our methods targets CDP not just in the approximate sense of closeness of the first moments but in terms of full equality of the conditional distributions. We empirically compare our methods to the state of the art for CDP (\textit{DCFR}) as well as methods designed to target (unconditional) demographic parity, including a method that we modify to approximately target CDP. Our evaluations over four real datasets show that our methods generally provide better fairness-performance trade-offs than existing methods.

Potential directions of future work include (i) extending the adoption of bi-causal distance to target other notions of conditional fairness such as equalized odds, (ii) combining methods proposed in this work (e.g. \textit{pre-processing repair} and \bcd) to further improve downstream fairness-performance trade-offs, and (iii) exploring in-training regularization approaches for fair synthetic data generation in healthcare and financial applications \cite{giuffre2023harnessing, potluru2023synthetic}.

\section*{Disclaimer}
This paper was prepared for informational purposes by the Artificial Intelligence Research group of JPMorgan Chase \& Co. and its affiliates ("JP Morgan'') and is not a product of the Research Department of JP Morgan. JP Morgan makes no representation and warranty whatsoever and disclaims all liability, for the completeness, accuracy or reliability of the information contained herein. This document is not intended as investment research or investment advice, or a recommendation, offer or solicitation for the purchase or sale of any security, financial instrument, financial product or service, or to be used in any way for evaluating the merits of participating in any transaction, and shall not constitute a solicitation under any jurisdiction or to any person, if such solicitation under such jurisdiction or to such person would be unlawful.




\bibliography{refs}


\newpage
\clearpage
\newpage

\begin{LARGE}
\begin{center}
\textbf{Appendix}
\end{center}
\end{LARGE}

\appendix

\cmnt{
\todo{
\begin{itemize}
    \item update tables with new results for preproc and legit
\end{itemize}
}
}

\section{Details on bi-causal transport distance}\label{sec:app:bicausal_ot}

The bi-causal transport distance, aka the nested transport distance, is defined through a constrained optimal transport problem. 
The optimal transport problem was originally formulated as the problem of transporting one distribution to another while incurring minimal cost \cite{monge1781memoire, kantorovitch1958translocation}. In our setup, for probability measures $\tilde{\Prob}, \Prob$ on measure space $\mc U \times \mc V$, the set $\Gamma (\tilde{\Prob},\Prob)$ of \textit{transport plans} denotes the collection of all probability measures on the space $(\mc U \times \mc V)\times (\mc U \times \mc V)$ with marginals $\tilde{\Prob}$ and $\Prob$.  Then for a cost function $\mathsf d: ((\mc U \times \mc V)\times (\mc U \times \mc V)) \to [0,\infty]$, the optimal transport problem is that of finding the transport plan $\gamma^*$ that attains the infimum
\begin{align*}
    \inf_{\gamma \in \Gamma (\tilde{\Prob},\Prob)}\int_{(\mc U \times \mc V)\times (\mc U \times \mc V)} \mathsf d((\tilde{u},\tilde{v}),(u,v)) d \gamma ((\tilde{u},\tilde{v}),(u,v)).
\end{align*}

A common distance defined through optimal transport is the \textit{Wasserstein distance}. \cmnt{Here, we require the cost function $\mathsf d$ to be a metric. } For $\tilde{\Prob}, \Prob$ with finite $p$-th moment, the Wasserstein distance of order $p$ with metric $\mathsf d$ is defined by $\mathcal{W}_p^p(\tilde{\Prob},\Prob; \mathsf d) := 
    \inf_{\gamma\in\Gamma(\tilde{\Prob},\Prob)} \E_{(\tilde{U},\tilde{V}),(U,V)\sim\gamma}\big[\mathsf d((\tilde{u}, \tilde{v}), (u, v))^p\big]$.
(Note that $\mathcal{W}_p$ refers to the Wasserstein distance of order $p$, while $\mathcal{W}_p^p$ refers to the Wasserstein distance taken to the power $p$, for simplicity of expression.)

With the common choice of the $\ell_p$ norm for $\mathsf d$, $\mathcal{W}_p^p$ becomes
\begin{align*}
    \inf_{\gamma\in\Gamma(\tilde{\Prob},\Prob)} \E_{(\tilde{U},\tilde{V}),(U,V)\sim\gamma}\big[ \norm{\tilde{U}-U}_p^p+ \norm{\tilde{V}-V}_p^p\big].
\end{align*}

The bi-causal transport distance can be defined in the framework above with a constrained set of transport plans $\Gamma _{bc}$ \cite{backhoff2017causal}. The bi-causal transport distance is defined in Definition \ref{def:bicausal_distance} in the main body of the paper; for clarity of exposition, we reiterate the definition here with some additional detail.

\begin{definition}[Causal and bi-causal transport plans]
A joint distribution $\gamma \in \Gamma (\tilde{\Prob}, \Prob)$ is called a \textit{causal transport plan} if for $((\tilde{U}, \tilde{V}), (U, V)) \sim \gamma$, $U$ and $\tilde{V}$ are conditionally independent given $\tilde{U}$:
\begin{align*}
 U \indep \tilde{V} \mid \tilde{U}.
\end{align*}
    \cmnt{We denote by $\Gamma _c (\tilde{\Prob}, \Prob)$ the set of all transport plans $\gamma \in \Gamma (\tilde{\Prob}, \Prob)$ that are causal.} Analogously, bi-causal transport plans are transport plans that are “causal in both directions.” The set of all such plans is given by
    \begin{align*}
        &\Gamma _{bc} (\tilde{\Prob}, \Prob) = \{ \gamma \in \Gamma (\tilde{\Prob}, \Prob) \text{ s.t. for } ((\tilde{U}, \tilde{V}), (U, V)) \sim \gamma, \\
        &\qquad \qquad \qquad\qquad \qquad\quad\;\; U \indep \tilde{V} \mid \tilde{U} \text{ and } \tilde{U} \indep V \mid U\}.
    \end{align*}
\end{definition}

The origin of causal transport can be traced back to the Yamada-Watanabe criterion for stochastic differential equations \cite{yamada1971uniqueness,jean1980weak,kurtz2014weak}. In the theory of optimal transport, \citet{lassalle2013causal} studied the transport problem in continuous time under the so-called causality constraint, and \citet{backhoff2017causal} considers a discrete-time analogue.
Causal transport has been applied to stochastic optimization \cite{acciaio2020causal}, as well as other areas such as stochastic control \cite{acciaio2019extended} and machine learning \cite{xu2020cot}. 
The bicausal transport distance is a symmetrized analog of the casual transport distance, which has been exploited to study the stability and sensitivity of multistage stochastic programming \cite{pflug2010version,pflug2012distance,pflug2014multistage,pflug2015dynamic,pflug2016empirical}. 
Most previous works consider a multi-period definition of bi-causal transport distance, whose main motivation is to investigate optimal transportation problems with filtrations and their applications
to stochastic calculus. In our problem, our definition can be viewed as a bi-causal transport plan \cite{backhoff2017causal} with two periods, which is also considered in \cite{yang2022decision}, to capture the conditional independence relations that we care about in the context of conditional demographic parity.
%
\begin{definition}[Bi-causal transport distance (BCD)] 
    The \textit{bi-causal transport distance} (BCD, referred to hereafter simply as the \textit{bi-causal distance}) between $\tilde{\Prob}$ and $\Prob$, denoted by $\Causal^p (\tilde{\Prob}, \Prob)$, is defined as
    \begin{align*}
        \inf _{\gamma \in \Gamma _{bc} (\tilde{\Prob}, \Prob)} \E _{((\tilde{U}, \tilde{V}),(U, V)) \sim \gamma} \left[  C \norm{\tilde{U}-U}_p^p + \norm{\tilde{V}-V}_p^p
    \right].
   \end{align*}
\end{definition}  
The following proposition shows that the bi-causal distance can be viewed as a nested Wasserstein distance \cite{backhoff2020adapted}. 
\begin{proposition}
\label{prop:nested}
The bi-causal distance can equivalently be written as $\Causal^p(\tilde{\Prob}, \Prob) = \Wass^p_p(\tilde{\Prob}_{\tilde{U}},\Prob_{U}; D)$ where
\begin{align*}
    D^p(\tilde{U},U) = \Wass_p^p(\tilde{\Prob}_{\tilde{V}|\tilde{U}}, \Prob_{V|U}) + C \norm{\tilde{U}-U}_p^p. 
\end{align*}
\end{proposition}
%

As mentioned in Section \ref{subsec:enforcing_cdd_wass}, $U$ in our setting corresponds to the legitimate feature $L$, while $V$ corresponds to the model output $f(X)$. The presence or absence of the tilde corresponds to the two levels of the sensitive feature $A$, so that $\Wass_p^p(\tilde{\Prob}_{\tilde{V}|\tilde{U}}, \Prob_{V|U})$ represents the Wasserstein distance between $\Prob(f(X) \mid L=l, A = 0)$ and $\Prob(f(X) \mid L=l, A = 1)$. 

\section{Additional discussion on disparity definitions}\label{sec:app:disparity}

\subsection{Demographic disparity}\label{subsec:app:DD}
In Appendix \ref{subsec:app:additional_results} below, we provide additional results from our experiments illustrating that methods which target demographic parity are not necessarily successful at targeting conditional demographic parity. In this section, therefore, we define the \emph{demographic disparity}, the violation of demographic parity.

Most previous work defines the demographic disparity (the violation of demographic parity) by $|\E[f(X) \mid A = 1] - \E[f(X) \mid A = 0]|$ or by $\E[f(X) \mid A = 1]/\E[f(X) \mid A = 0]$ and aims to make this quantity as close to 0 (for the difference) or 1 (for the ratio) as possible \cite{pessach2023ReviewFairnessMachine}. When $f(X)$ is binary, equalizing the first moments in this way is equivalent to enforcing that $f(X) \indep A$, but if $f(X)$ is continuous, then $f(X)$ may behave very differently for the two groups even if the means of the two conditional distributions are equal. As with the conditional demographic disparity (Definitions \ref{def:CDD_wass} and \ref{def:CDD_lp}, we define the demographic disparity via a distributional distances, as we are interested in minimizing the discrepancy between the entire distributions, not just the first moments. 

\begin{definition}[Demographic disparity (DD)]\label{def:sp}
Let $d(\cdot, \cdot)$ denote a distance between distributions. The demographic disparity for model $f(X)$ with respect to $\mathsf d(\cdot, \cdot)$ is defined by
\begin{align}
    \mathsf d(\Prob_{f(X)|A=0},\Prob_{f(X)|A=1}).
\end{align}
In other words, we define the demographic disparity as the distance between the conditional distributions of $f(X)$ for the two groups represented by the sensitive feature.
\end{definition}

The choice of distance function $\mathsf d$ may depend on the problem setting. For example, minimizing the Kolmogorov distance in some sense requires fairness across the entire range of values of $f(X)$. In some settings, this may be ethically appropriate, while in others it may be unnecessary and might incur an unacceptable trade-off in performance. The Wasserstein distance has previously been used to target demographic parity, in part because the closeness of two distributions in the Wasserstein sense is related to closeness in the downstream performance of models trained on those distributions \cite{villani2009optimal, santambrogio2015optimal, xiong2023fair}. 
We utilize the Wasserstein distance to measures DD in our additional results below.

\subsection{CDP with continuous legitimate features}\label{subsec:app:continuous}

The conditional demographic disparity (CDP) is defined over the set $\mc S = \operatorname{supp} \mathbb P ({L | A= 0}) \cap \operatorname{supp} \mathbb P ({L| A= 1})$, i.e. on the common support of the legitimate feature across the two groups defined by the sensitive feature. In practice, if the legitimate feature is continuous, then the empirical common support in any finite sample will be empty, which means the CDD cannot be estimated. Moreover, any method targeting CDP will require multiple samples at each level $l$ of $L$. This is why we need to assume that $L$ is either naturally discrete or appropriately discretized. This is not overly limiting, since continuous features may be discretized using appropriate binning strategies informed by the fairness requirements and the problem setting.

As an alternative, we envision a ``smooth CDP'' definition, where parity is required not only at each level but to some degree across different levels. The extent to which model output distributions are allowed to differ across levels will be reversely related to the distance between levels. Such a definition would be applicable to both discrete and continuous legitimate features, without any binning required. The implications and appropriateness of such a definition need careful attention and are beyond the scope of this work.

\subsection{Aggregation strategies for CDD in the $\ell_p$ sense}\label{subsec:app:aggregation}
The conditional demographic disparity (CDD) in the $\ell_p$ sense (Definition \ref{def:CDD_lp}) aggregates the level-wise disparities at each value of the legitimate features according to a measure $\mb Q$. An obvious candidate for $\mb Q$ is $\mb P(L)$, where the weights are proportional to the amount of data at each level. However, this measure may overly ``favor'' the majority class in cases where the data is unbalanced with respect to the sensitive feature. For example, using the same scenario as in Table~\ref{tab:synthetic_example} in the Introduction, suppose there were 10 times as many female as male loan applicants, and suppose that most females had high incomes and most males had low incomes. Suppose that the approval probability was the same for high income female vs. male applicants, but that low income females had a higher approval probability than low income males. Suppose that $p$ were set to $1$, so that the disparity were simply $\E[D(L)]$, the average distributional distance across levels of $L$. In this example, $\E[D(L)]$ would be relatively small, since much of the mass of $L$ would be concentrated in the high income level, but the majority of males would be subject to the disparate outcomes represented by the low income tier.

Alternatively, one could use $\mb Q = \frac{\Prob(L \mid A = 0)+\Prob(L \mid A = 1)}{2}$ to avoid favoring the majority class. This probability measure ensures that disparities which primarily affect only one level of the sensitive feature {(e.g. males in the example in the preceeding paragraph)} are more influential in the overall disparity measure. \cmnt{However, this average measure, just like $\Prob (L)$, results in an underemphasis on the levels of $L$ with smaller probability mass (density).}

A third possible choice is $\mb Q = \mb U (L)$, where $\mb U(L)$ is the uniform distribution over values of $L$. This puts equal emphasis on every level of $L$, which may or may not be desirable. Any distribution $\mb Q$ results in a valid CDD measure; the choice depends on the user's own values and problem-specific characteristics like the distributions of $A$ and $L$.

In all the disparity definitions given above, different distance functions $\mathsf d$ and different values of $p$ induce different notions of disparity. For example, in the CDD in the $\ell_p$ sense (Definition \ref{def:CDD_lp}), if $\mathsf d(\cdot, \cdot)$ is the Kolmogorov distance and $p=\infty$, then the conditional demographic disparity of $f(X)$ is $\varepsilon$ 
\cmnt{if for all $ l \in \operatorname{supp} \mathbb P ({L | A= 0}) \cap \operatorname{supp} \mathbb P ({L| A= 1})$,
\begin{align*}
    |\Prob({f(X)\le \tau|L=l,A=0})-\Prob({f(X)\le \tau|L=l,A=1})| \le \varepsilon, 
\end{align*}
or in other words}
if the Kolmogorov distance between $\Prob(f(X) \mid L = l, A = 0)$ and $\Prob(f(X) \mid L = l, A = 1)$ is no more than $\varepsilon$ for every possible $l$. Setting $p = \infty$ and enforcing a small corresponding disparity requires the model $f(X)$ to behave similarly for the two groups within every level of the legitimate feature $L=l$. On the other hand, setting $p$ to, say, 2, allows $f(X)$ to potentially behave quite differently for the two groups over areas of $\mathcal{L}$ with small measure.

\section{Computing the bi-causal distance} \label{app:bicausal-dist}

As shown in Proposition \ref{prop:nested}, the bi-causal distance can be viewed as a nested Wasserstein distance with particular inner and outer distance functions. Computing the Wasserstein distance exactly is known to be computationally intractable \cite{arjovsky2017wasserstein, salimans2018improving}. We approximate the Wasserstein distance by the Sinkhorn divergence \cite{sinkhorn1964relationship, cuturi2013sinkhorn}, which is computationally efficient and corresponds to regularizing the Wasserstein distance with an entropy term. It also improves the scalability of traditional methods.
The convergence of this entropic regularized optimal transport problem is theoretically studied in \citet{carlier2017convergence}, while \citet{ghosal2022convergence} obtain the rate $\mathcal{O}(t^{-1})$ for a large class of cost functions and marginals. The consistency and differential properties are studied in \citet{luise2018differential}. 

To summarize the idea behind computing the bi-causal transport distance, the nested Sinkhorn divergence is a multistage stochastic program. Assuming the observed samples $(X_i, Y_i, A_i)$ are i.i.d. samples from a stochastic process, the algorithm proposed by \citet{pichler2021nested} discretizes the whole space and filtration of the stochastic process (i.e., the whole feature space) using scenario trees and uses those compute the optimal transport plan \cite{heitsch2009scenario}. We compute the gradient through the Sinkhorn divergence algorithm and propagate them via stochastic gradient descent to minimize the empirical risk as in \citet{oneto2020exploiting}.

Algorithm~\ref{alg:nested} presents the details of the nested algorithm used to calculate the bi-causal distance; it can be viewed as a special case of Algorithm 1 in \cite{pichler2021nested} with two time periods. For each period, we use the Sinkhorn algorithm (\cite{cuturi2013sinkhorn,flamary2021pot}). To simplify the notation, denote $\Prob_0 := (f \otimes \operatorname{Id}) _\# \mathbb P _{X | A = 0}$ and $\Prob_1  := (f \otimes \operatorname{Id}) _\# \mathbb P _{X | A = 1}$. For the discrete nested distance, we use scenario trees to model the whole space and filtration (i.e., the whole feature space). Denote $\cN_0^X, \cN_0^Y$ ($\cN_1^X$, $\cN_1^Y$, resp.) the set of all nodes of $\Prob_0$ ($\Prob_1$, resp.). We use $m \prec i$ to denote the predecessor $m$ of the node $i$, and similarly $i \succ m$ to denote the successors $i$ of the node $m$. In this notation, we can define the two probability distributions as:

\begin{align*}
    \Prob_0 = \sum_{m\in\cN_0^X} p_0^m \sum_{i\succ m} q_0^i \delta_{(x_0^m,y_0^i)}, \\ \Prob_1 = \sum_{
n\in\cN_1^X} p_1^n \sum_{j\succ n} q_1^j, \delta_{(x_1^n,y_1^j)},
\end{align*}
where $\delta_{a,b}$ is the Kronecker delta function, i.e. $\delta_{a,b}(x, y)=1$ only for $x=a$ and $y=b$, and otherwise $0$. Moreover, $p_k^m$ denotes the marginal distribution of the predecessor $m$, where $q_k^i$ is the conditional distribution of the node $i$ for $k=0,1$. 
The Sinkhorn algorithm \cite{cuturi2013sinkhorn} cam be viewed as an entropic regularized optimal transport problem, by adding a logarithmic penalty term $H(\gamma) = \sum_{i,j} \gamma_{i,j} \log \gamma_{i,j}$ to the objective. 

\cmnt{Algorithm~\ref{alg:nested} shows the nested procedure for calculation the bi-causal distance via the nested Sinkhorn.} The inputs to Algorithm~\ref{alg:nested} are the distributions $\Prob_0$ and $\Prob_1$, which in practice correspond to the features $(x_0, y_0) \sim \Prob_0$ and the marginal probabilities $p_0$ --- equivalently, $(x_1, y_1)$ and $p_1$ for $\Prob_1$. As we discretize the feature space using scenario trees, we denote $x^m_0$ the features belonging to the node $m \in \cN_0^X$, and $y^i_0$ the labels belonging to the node $m \in \cN_0^Y$ (and equivalently for $\Prob_1$).



\begin{remark}
In our algorithms, we use the Sinkhorn solver \cite{flamary2021pot} whose time complexity has proven to be nearly $\mathcal{O}(n^2)$. This means that the time complexity of computing the regularizer in \fairlp~ (Section \ref{subsec:enforcing_cdd_lp}) is $ \mc O(|L|)  + \sum_{j=1}^{|L|}\mathcal{O} (n_j^2) = \mc O(\frac{n^2}{|L|})$ where $n_j$ is the number of samples with legitimate feature value $l_j$ and all $n_j$'s are comparable.
Moreover, it follows that our nested Sinkhorn algorithm in \fairbit~has time complexity of $ |L| ^2 + \sum _j \sum _k n _j n _k = |L| ^2 + n ^2 = \mathcal{O} (n ^2)$ where $n_j$ is the number of samples with legitimate feature value $l_j$ and all $n_j$'s are comparable, $n = \sum _{k = 1} ^K n _k$ is the sample size, and $|L|$ is the number of the levels the legitimate features take in the training sample (so that necessarily $|L| \leq n$). 
\end{remark}

\begin{algorithm}
\caption{Nested Sinkhorn iteration of the bi-causal distance $\Causal^p(\Prob_0,\Prob_1)$}
\label{alg:nested}
\SetKwInOut{Input}{Input}
\SetKwInOut{Output}{Output}
\SetKwInOut{Result}{Result}
\Input{The distribution $\Prob_0$, $\Prob_1$ by specifying $x_0^m$, $y_0^i$, $p_0^i$, $x_1^n$, $y_1^j$, $p_1^j$ for $m\in\cN_0^X$, $i\in\cN_0^Y$, $n\in\cN_1^X$, $j\in\cN_1^Y$, and regularization parameter $\lambda>0$.}
\Output{Bi-causal distance between the two distributions $\Prob_0$ and $\Prob_1$.}
\vspace{0.2cm}
\For{every combination of nodes $m\in\cN_0^X$,$n\in\cN_1^X$}{

For all combination of leaf nodes $i\in\cN_0^Y$ and $j\in\cN_1^Y$ with predecessor $m\prec i$ and $n\prec j$: 

$$d_Y(i,j)^p=|x_0^m-x_1^n|^p+|y_0^i-y_1^j|^p$$

solve the linear program
\begin{align*}
    &\gamma^{\star}_Y (m,n)^p = \nonumber\\
    &\quad\min_{\gamma_Y} \sum_{\substack{i\succ m \\ j\succ n}} \gamma_Y(i,j|m,n) d_Y(i,j)^p - \frac{1}{\lambda} H(\gamma_Y)\\
    &\text{subject to }\sum_{j\succ n}\gamma_Y(i,j|m,n)=q_0^i,\qquad i\succ m,\\
    &\qquad\qquad\sum_{i\succ m}\gamma_Y(i,j|m,n)=q_1^j,\qquad j\succ n,\\
    &\qquad\qquad\gamma_Y(i,j|m,n)\ge0.
\end{align*}
}

solve the linear program
\begin{align}
    \label{eqn:minimize-gamma-x}
    \min_{\gamma_X} &\sum_{\substack{i\succ m \nonumber\\ j\succ n}} \gamma_X(m,n) \gamma^\star_Y(m,n)^p -  \frac{1}{\lambda} H(\gamma_X)\nonumber\\
    \text{subject to }&\sum_{j\succ n}\gamma_X(m,n)=p_0^m,\qquad\qquad i\succ m,\nonumber\\
    &\sum_{i\succ m}\gamma_X(m,n)=p_1^n,\qquad\qquad j\succ n,\nonumber\\
    &\gamma_X(m,n)\ge0.
\end{align}

\vspace{0.2cm}

\Return{The bi-causal distance between $\Prob_0$ and $\Prob_1$ is the $p$-th root of the optimal value of \eqref{eqn:minimize-gamma-x}.}
\end{algorithm}

\begin{remark}
\label{rmk:faster-sinkhorn}
    {In the above algorithm, the complexity mostly depends on the complexity of the Sinkhorn solver we used. A recent work by \citet{lakshmanan2023fast} has significantly reduce the arithmetic operations to compute the distances from $\mathcal{O} (n ^2)$ to $\mathcal{O} (n \log n)$, which enables access to large and high-dimensional data sets. By employing nonequispaced fast Fourier transform as what they do, we may be able to reduce our time complexity of $ K ^2 +\mathcal{O} (n \log n)$. We leave this to future work.}
\end{remark}

\section{Relationship between DCFR and CDD}\label{sec:illustrative-example-dcfr}
Here we consider in further detail the relationship between conditional demographic disparity (CDD) and the Derivable Conditional Fairness Regularizer (DCFR) method proposed by \citet{xu2020algorithmic}.

The CDD considers the distances between the conditional distributions of the model $f(X)$ given the sensitive feature $A$ and the legitimate feature $L$, i.e. between $\Prob (f(X) \mid L, A = 1)$ and $\Prob (f(X) \mid L, A = 0)$. The DCFR regularizer $R_{\text{DCFR}}$, by contrast, aims to minimize $\left(\Prob (A = 1 \mid Z, L) - \Prob (A = 1 \mid L)\right)_+$, the (positive part function of the) differences between the conditional distributions of the sensitive feature given the legitimate feature and a variable $Z = g(X, A)$ that is a transformation of the input features. DCFR promotes the learning of the transformation $g(X, A)$ simultaneously with the training of a model $f(Z)$. (See Figure 2 in \citealt{xu2020algorithmic}.)

In the special case that $R_{\text{DCFR}} = 0$, the CDD is also 0, meaning that Conditional Demographic Parity holds. This is illustrated via the following chain of reasoning:
\begin{align*}
    & R_{\text{DCFR}} := \E\left[\left(\Prob (A = 1 \mid Z, L) - \Prob (A = 1 \mid L)\right)_+\right] = 0 \\
    & \implies \Prob(A = 1 \mid Z = z, L = \ell) = \Prob(A = 1 \mid L = \ell)~\forall \ell, z \\ & \implies Z \indep A \mid L \\
    & \implies f(Z) \indep A \mid L.
\end{align*}
However, it does not follow that if $R_{\text{DCFR}}$ is small, then $f(Z)$ is approximately independent of $A$ given $L$. More precisely, the magnitude of $R_{\text{DCFR}}$ does not necessarily guarantee anything about the magnitude of the CDD. One reason for this is that without further constraints, closeness in the conditional distributions $\Prob(A \mid Z, L)$ does not imply closeness in the distributions $\Prob(Z \mid A, L)$. The following proposition illustrates this point. The proof is given below in Section \ref{subsec:supplementary_prop_proof}.
\begin{proposition} \label{prop:dcfr_cdd_relationship}
Assuming each of the conditional probabilities and densities are well-defined and the denominators are non-zero in the following expression, we have
\begin{align*}
\frac{\lvert \P(Z = z \mid A = 1, L = \ell) - \P(Z = z \mid A = 0, L = \ell) \rvert}{\lvert \P(A = 1 \mid Z = z, L = \ell) - \P(A = 1 \mid L = \ell) \rvert} =
\\
\frac{\P(Z = z \mid L = \ell)}{\P(A=0 \mid L = \ell)\P(A=1 | L = \ell)}.
\end{align*}
\end{proposition}
This follows by a simple application of Bayes' rule. The numerator on the left represents the difference in the densities of $Z$ at $z$ conditional on $L = \ell$ for the two groups, while the denominator represents the differences at $Z = z$ in the probability mass functions that appear in $R_{\text{DCFR}}$.
The expression on the right shows that their ratio can be arbitrarily large, since $\P(A = 0 \mid L = \ell)\P(A = 1 \mid L = \ell)$ can be arbitrarily small and the density $\Prob(Z = z\mid L = \ell)$ could be arbitrarily large. The relationship between the conditional distributions on the left therefore varies across levels of $z, \ell$ in a problem-dependent way. By contrast, in \fairlp{} we directly use (an empirical estimate of) CDD in the $\ell_p$ sense as the regularizer. Moreover, Proposition \ref{prop:bcd_cdd} shows that the bi-causal distance (with large enough C) employed in \fairbit{} is equivalent to the CDD in the Wasserstein sense, irrespective of the data generating distribution.


Note additionally that when $f(Z)$ is continuous, closeness in the distributions $\Prob(Z \mid A)$ does not necessarily imply closeness in the distributions $\Prob(f(Z) \mid A)$. For example, consider a simple setting in which $Z \mid A = 0 \sim N(0, 1)$ and $Z \mid A = 1 \sim N(\delta, 1)$ with $\delta > 0$. It is easy to see that for common distance measures such as $1-$ and $2-$Wasserstein distance, total variation distance, etc., the distance between these two distributions can be made arbitrarily small by making $\delta$ small. However, for any fixed $\delta$, the distance between $\Prob(f(Z) \mid A = 1)$ and $\Prob(f(Z) \mid A = 0)$ can be made arbitrarily large by making $f(Z)$ an appropriate increasing function of $Z$. This further complicates the use of DCFR outside a classification setting. 

\subsection{Data generating process for Figure \ref{fig:res:Xucomp_simple_example_body}}
To illustrate the relationship between DCFR and CC, we consider a simple loan setting, in which the sensitive feature $A$ represents sex (male vs. female), $X = L = \emptyset$, and $f(X, A) \in \{0, 1\}$ represents whether an applicant is approved for a loan or not, with $a, b, c, d$ representing the probability mass in each cell:

\begin{table}[h]
    \centering
    \begin{tabular}{@{}lcc@{}}
        \toprule
                  & Male & Female \\ \midrule
        Approved  & a    & b      \\
        Not approved & c & d \\ \bottomrule
    \end{tabular}
\end{table}%
In this simple setting, since $L = \emptyset$, there are no levels to aggregate across when computing the conditional demographic disparity, so \cddw = \cddlp, and we denote them both by $\mathsf {CDD}$. We have: 
\begin{align*}
\mathsf {CDD} = \frac{a}{a+c}-\frac{b}{b+d}\\
R_{\text{DCFR}} = \frac{a}{a+b}-\frac{c}{c+d}.
\end{align*}

In this simple example, reducing $\mathsf {CDD}$ implies minimizing the difference in acceptance rate by sex, while reducing $R_{\text{DCFR}}$ implies minimizing the difference between the ratio of admitted males and the proportion of males in the applicant group. 

To produce Figure \ref{fig:res:Xucomp_simple_example_body} in Section \ref{sec:DCFR}, we create a simulation in which we repeatedly sample $10,000$ applicants, varying both the proportion of males in the applicant pool $r_m$ and the difference in acceptance rate by a hyperparameter $\delta$ as follows:
\begin{align*}   
&\Prob(\text{Approve} | \text{Male}) = 0.5 + \delta, \\ 
&\Prob(\text{Approve} | \text{Female}) = 0.5 - \delta, \qquad \text{for}\, \delta \in [0, 0.5).
\end{align*}
For each sample, we compute $R_{\text{DCFR}}$ and $\mathsf {CDD}$ as above, and we compute the \fairbit{} regularizer (which here is equivalent to the \fairlp{} regularizer due to the fact that $L = \emptyset$) as described in Section \ref{subsec:enforcing_cdd_wass}. Figure \ref{fig:res:Xucomp_simple_example_body} illustrates how the relationship between $R_{\text{DCFR}}$ and $\mathsf {CDD}$ depends on the proportion of males. Each line represents different values of $\delta$ for a fixed proportion of males $r_m$. The closer $r_m$ is to 1, the steeper the slope in the $\mathsf {CDD}$ - $R_{\text{DCFR}}$ curve. A given small value of $R_{\text{DCFR}}$ can correspond to an arbitrarily large value of $\mathsf {CDD}$. The only case where $\mathsf {CDD}$ and $R_{\text{DCFR}}$ are equal in this example is when $r_m=0.5$, i.e., the proportions in the candidate pool are balanced.  In contrast, \fairbit{} (and equivalently, \fairlp{}) enjoy a consistent relationship with $\mathsf {CDD}$, with some variance along the 45 degree line due to the finite sample size in the example.

Note that since $L = \emptyset$ in this example, this example is equivalent to fixing the \fairbit{} regularizer and  $R_{\text{DCFR}}$ for a specific level $L=l$. This example shows that $R_{\text{DCFR}}$ does not promote small demographic disparity in each level, and therefore it does not promote small $\mathsf {CDD}$, regardless of how level-wise disparities are aggregated. 

\subsection{Proof of Proposition \ref{prop:dcfr_cdd_relationship}} \label{subsec:supplementary_prop_proof}
\begin{proof}
Let $N_{\text{left}}, D_{\text{left}}$ denote the numerator and denominator on the left hand side of the proposition:
\begin{align*}
    &N_{\text{left}} = \lvert \P(Z = z \mid A = 1, L = \ell) \nonumber\\
    &\qquad\qquad\qquad\qquad- \P(Z = z \mid A = 0, L = \ell) \rvert \\
    &D_{\text{left}} = \lvert \P(A = 1 \mid Z = z, L = \ell) - \P(A = 1 \mid L = \ell) \rvert
\end{align*}
First we note that:

\begin{align*}
    &\P(Z = z| L = \ell) = \\
    &\qquad\P(Z = z\mid A=1, L = \ell) \P(A=1 \mid L = \ell)\\ 
    &\qquad +\P(Z = z\mid A=0, L = \ell) \P(A=0 \mid L = \ell) 
\end{align*}
Therefore,
\begin{align*}
    &\P(Z = z\mid A=0, L = \ell) = \\
    &\qquad\frac{\P(Z = z\mid L = \ell) }{\P(A=0 \mid L = \ell)}\\
    &\qquad- \frac{ \P(Z = z\mid A=1, L = \ell) \P(A=1 \mid L = \ell)}{\P(A=0 \mid L = \ell)}, 
\end{align*}

which we can use to rewrite $N_{\text{left}}$ as:

\begin{align*}
& \lvert \P(Z = z\mid A=1, L = \ell) \\
&- \frac{\P(Z = z\mid L = \ell) }{\P(A=0 \mid L = \ell)}\\
&+ \frac{\P(Z = z\mid A=1, L = \ell) \P(A=1 \mid L = \ell)}{\P(A=0 \mid L = \ell)}\rvert \\
&= \lvert\frac{- \P(Z = z\mid L = \ell)}{\P(A=0 \mid L = \ell)} \\
&+ \frac{\P(Z = z\mid A=1, L = \ell) [\P(A=0) + \P(A=1)]}{\P(A=0 \mid L = \ell)}\rvert \\
&= \lvert\frac{\P(Z = z\mid A=1, L = \ell) - \P(Z = z\mid L = \ell)}{\P(A=0 \mid L = \ell)}\rvert,
\end{align*}

since $\P(A=0) + \P(A=1) = 1$ as the sensitive feature $A$ is binary. We then note that:

\begin{align*}
    &\P(A=1 \mid Z = z, L = \ell) = \\
    &\qquad\qquad\frac{\P(Z = z\mid A=1, L = \ell) \P(A=1\mid L = \ell)}{\P(Z = z\mid L = \ell)} 
\end{align*}
Therefore,
\begin{align*}
    &\frac{\P(A=1 \mid Z = z, L = \ell)}{\P(A=1\mid L = \ell)} =\\
    &\qquad\qquad\qquad\frac{\P(Z = z\mid A=1, L = \ell)}{\P(Z = z\mid L = \ell)} := \gamma .
\end{align*}

Given the definitions above, we have that:

\begin{align*}
& \frac{N_{\text{left}}\P(A=0\mid L = \ell)}{\P(Z = z\mid L = \ell)} = \lvert \gamma - 1 \rvert, \\
& \frac{D_{\text{left}}}{\P(A=1\mid L = \ell)} = \lvert \gamma - 1 \rvert \\
\implies & \frac{N_{\text{left}}}{D_{\text{left}}} = \frac{\P(Z = z\mid L = \ell)}{\P(A=0\mid L = \ell)\P(A=1\mid L = \ell)}
\end{align*}


\end{proof}

\section{Proofs}\label{sec:app:proofs}

In this section, we provide the proofs of Propositions~\ref{prop:nested}, \ref{prop:bcd_cdd}, and \ref{prop:dcfr_alternative_form}. 
\begin{proof}[Proof of Proposition~\ref{prop:bcd_cdd}]
%
Consider CDD in Wasserstein sense
$
    \mathsf {CDD}^{\mathsf {wass}}_f = \Wass^p_p(\Prob ({L | A= 0}), \Prob ({L| A= 1}) ; D), 
$
with
\begin{align*}
D_{cdd}(l,l') =\begin{cases}
     \mathsf d_{cdd}\big(\Prob(f(X) \mid L=l, A = 0),\\ \qquad~\Prob(f(X) \mid L=l', A = 1)\big), ~ l=l',\\
    \infty, \qquad\qquad\qquad\qquad \qquad\text{elsewhere}.
   \end{cases}  
\end{align*}
Define the set $\mb B = \{(l,l')|~l,l'\in \mc S,~ l\neq l'\}$. Consider two transport plans $\gamma_1,\gamma_2 \in \Gamma (\P({L | A= 0}), \P ({L| A= 1})$ such that $\forall (l,l')\in \mb B:~\gamma_1(l,l')= 0$ but $\exists (l,l')\in \mb B: ~\gamma_2(l,l')\neq 0$.
It follows from the definition of \cddw{} that for bounded $\mathsf d$, we have $D_{cdd}(l,l)<D_{cdd}(l,l')$ for all $l\neq l'$. Therefore, 
transport plan $\gamma_2$ cannot be an optimal transport plan of the optimal transport problem with cost function $D_{cdd}(l,l')$, since any transport plan $\gamma_1$ such that $\forall (l,l')\in \mb B:~\gamma_1(l,l')= 0$ will have a lower cost. 

On the other hand, Proposition \ref{prop:nested} states that $\Causal^p(\tilde{\mb P}, \mb P)$ can be written as $\Wass^p_p(\tilde{\Prob}_{\tilde{U}},\Prob_{U}; D_{bcd})$ where
\begin{align*}
    D_{bcd}^p(\tilde{U},U) = \Wass_p^p(\tilde{\Prob}_{\tilde{V}|\tilde{U}}, \Prob_{V|U}) + C \norm{\tilde{U}-U}_p^p. 
\end{align*}
Therefore, when $C>\frac{\bar d}{\underline d}$, we have 
\begin{align*}
    \forall (l,l')\in \mb B: \quad &C \norm{l'-l}_p^p \\
    &  >\bar d \\
    &\geq \mc W_p^p\big(\Prob(f(X) \mid L=l, A = 0),\\
    &\qquad\qquad\qquad\Prob(f(X) \mid L=l, A = 1)\big) .
\end{align*}
It follows that any transport at the same level of $L$ has a lower cost than any transport between different levels. Therefore, the optimal transport plan $\gamma$ with cost function $D_{bcd}$ must satisfy $\forall (l,l')\in \mb B:~\gamma(l,l')= 0$.

Therefore, when $\mathsf d_{cdd}(\cdot,\cdot) = \Wass_p^p(\cdot,\cdot)$ and $C>\frac{\bar d}{\underline d}$, both \cddw and $\Causal^p$ correspond to optimal transport plans with zero mass on $\mb B$ and have the same cost functions when $l=l'$, and therefore they are equivalent.
\end{proof}

\begin{proof}[Proof of Proposition \ref{prop:nested}]
It follows from Definition \ref{def:bicausal_distance} that bi-causal distance can be written as
    \begin{align*}
    &\Causal^p (\tilde{\Prob}, \Prob) := \nonumber\\
    &\inf _{\gamma \in \Gamma _{bc} (\tilde{\Prob}, \Prob)} \E _{((\tilde{U}, \tilde{V}),(U, V)) \sim \gamma} \left[ 
            {\color{red}C}\norm{\tilde{U}-U}_p^p + \norm{\tilde{V}-V}_p^p
    \right]=\nonumber\\
    & \inf_{\gamma_1\in\Gamma(\tilde{\Prob}_{\tilde{U}},\Prob_{U})}  \E_{(\tilde{U},U)\sim\gamma_1}\Big[ \inf_{\gamma_2\in\Gamma(\tilde{\Prob}_{\tilde{V}|\tilde{U}},\Prob_{V|U})} \E_{(\tilde{V},V)\sim\gamma_2}\Big[\nonumber\\
    &\qquad\qquad\qquad\qquad C\norm{\tilde{U}-U}_p^p +\norm{\tilde{V}-V}_p^p \mid (\tilde{U},U) \\
    &\qquad\qquad\qquad\qquad\qquad\qquad\qquad\qquad\qquad\qquad\qquad\Big]\Big]
    \end{align*}
\end{proof}

\begin{proof}[Proof of Proposition \ref{prop:dcfr_alternative_form}]
By definition of $Q(h)$ in \citet{xu2020algorithmic}:
\begin{align*}
     Q(h) &= \E[\mathbf{1}_{\{A=1\}}\Prob(A=0|L)h(Z,L)]\\
     &\qquad- \E[\mathbf{1}_{\{A=0\}}\Prob(A=1|L)h(Z,L)]\\
     & =\E\big[\big(\mathbf{1}_{\{A=1\}}\Prob(A=0|L)\\
     &\qquad- \mathbf{1}_{\{A=0\}}\Prob(A=1|L)\big)h(Z,L) \big]\\
     & =\E\big[\E\big[\mathbf{1}_{\{A=1\}}\Prob(A=0|L)\\
     &\qquad- \mathbf{1}_{\{A=0\}}\Prob(A=1|L)|Z,L\big]h(Z,L) \big]  \\
     & =\E\Bigg[ \Big (\Prob(A=1|Z,L)\Prob(A=0|L)\\
     &\quad- \Prob(A=0|Z,L)\Prob(A=1|L)\Big) h(Z,L)\Bigg]\\
     & =\E[ {g(Z,L)}h(Z,L)].
\end{align*}
Since $h \in L^2$ and $0 \leq h(Z, L) \leq 1$, then it follows that:
\begin{align*}
    h^\star &= \arg\sup_h Q(h)\\
    &= \mathbf{1}_{\{g(Z, L) > 0\}} \\
    & = \mathbf{1}_{\{\Prob(A=1|Z,L)\Prob(A=0|L)- \Prob(A=0|Z,L)\Prob(A=1|L)>0\}}\\
      & = \mathbf{1}_{\{  \frac{\Prob(A=1|Z,L)}{\Prob(A=1|L)} >\frac{\Prob(A=0|Z,L)}{\Prob(A=0|L)} \}}\\
      & = \mathbf{1}_{\{ \Prob(A=1|Z,L)>\Prob(A=1|L)\}.}
\end{align*}
Hence we have that:
\begin{align*}
    R_{\text{DCFR}} &= Q(h^\star) \\
    &= \E[\big(\Prob(A=1|Z,L)\Prob(A=0|L)\\
    &\qquad- \Prob(A=0|Z,L)\Prob(A=1|L)\big)_+]\\
      &=\E[ ( \Prob(A=1|Z,L)-\Prob(A=1|L) )_+ ]
\end{align*}

\end{proof}

\section{Experiments Details} \label{sec:app:exp_details}

This Section includes a more detailed breakdown of the numerical experiments. Section~\ref{sec:datasets-details-app} provides more information on the datasets used, Section~\ref{sec:implementation-details-app} lists all the implementation details including hyper-parameter ranges for each method, Section~\ref{subsec:app:additional_results} provides further discussion on the fairness-performance tradeoffs of all methods, Section~\ref{sec:result-details-app} provides an exhaustive inventory of all results, including additional figures in this appendix. Finally, Section~\ref{sec:result-considerations-app} includes any further considerations which was not included in the main paper due to space concerns.

\subsection{Datasets Details}\label{sec:datasets-details-app}

\begin{enumerate}
    %
    \item \texttt{Adult} dataset \cite{misc_adult_2}. The Adult dataset is based on the 1994 U.S. census data. It contains demographic information for $48,842$ respondents. It contains $14$ features and a target variable (\texttt{income}: whether income is over $\$50,000$ or not). We consider \texttt{sex} as the sensitive feature and \texttt{occupation} as the legitimate feature.
    In this dataset, the occupation \texttt{Armed Forces} and \texttt{Priv-House-serv} are very rare, so we combine them with \texttt{Protective Services}. We encode the values of \texttt{occupation} as integer values, i.e. $L \in \{0,..,11\}$, resulting in $10$ levels. We use a subsample (selected uniformly at random without replacement) of size $5,000$ for computational reasons (see details below).
    \item \texttt{Drug} dataset \cite{fehrman2017five}. This dataset contains self-declared drug usage history for 1885 respondents, for 19 different drugs. For each drug, the respondents declare the time of the last consumption (possible responses are: never, over a decade ago, in the last decade/year/month/week, or on the last day). It also has demographic variables and scores on several psychological tests. In our classification problem, the task is to predict the response ``never used'' versus ``others'' (i.e., used) for \texttt{cannabis}. We treat \texttt{education} as the legitimate feature ($L \in \{0,..,8\}$, resulting in $9$ levels) and \texttt{gender} as the sensitive feature.
    \item \texttt{Law School} dataset \cite{ramsey1998lsac}. This dataset comes from a survey conducted by the Law School Admission Council (LSAC) across 163 law schools in the United States in 1991. It contains law school admission records and law school performance of $21,027$ students. 
    The prediction task is to predict the students' first-year average grade (\texttt{zfygpa}). We use the candidate LSAT score as legitimate feature (discretized in deciles, $L \in \{0,..,9\}$, resulting in $10$ levels) and the candidate gender as the sensitive feature.
    \item \texttt{Communities and Crime} dataset \cite{misc_communities_and_crime_183}. This dataset combines socio-economic data from the $1990$ US Census, law enforcement data from the $1990$ US LEMAS survey, and crime data from the $1995$ FBI UCR. It includes data for $1,994$ communities across the United States, represented by $126$ socio-economic, demographic, and law enforcement-related features. The target variable is \texttt{ViolentCrimesPerPop} which represents the rate of violent crimes per $100,000$ population in each community. We cosider \texttt{racepctblack} (percentage of the population in a community that is Black) as the sensitive feature. We cosider \texttt{medIncome} (the median income of the community) as the legitimate feature (discretized in deciles, $L \in \{0,..,9\}$, resulting in $10$ levels).
\end{enumerate}

See Table \ref{tab:dataset_features} for the list of the features used in each dataset.

\begin{table*}[ht!]
\centering
\begin{tabular}{|c|c|c|}
\hline
\multicolumn{1}{|c|}{\textit{Unfairness metric}}                                                            & \textit{Figures}  \\ \hline 
 \begin{tabular}[c]{@{}c@{}}\multicolumn{1}{c}{\cddw} \end{tabular} & \ref{fig:empirical-comparison-results}(top row) 
 \\ \hline
\multicolumn{1}{|c|}{\cddlp (uniform)}       & \ref{fig:empirical-comparison-results}(bottom row)
\\ \hline
\multicolumn{1}{|c|}{\cddlp ($P(L)$) } & \ref{fig:empirical-comparison-results_appendix}(top row)
\\ \hline
\multicolumn{1}{|c|}{\cddlp (Ave. $P(L|A)$)} & \ref{fig:empirical-comparison-results_appendix}(middle row)
\\ \hline
\multicolumn{1}{|c|}{DP} & \ref{fig:empirical-comparison-results_appendix}(bottom row)
\\ \hline
\end{tabular}
\caption{List of all figures for all experiments across different fairness metrics.}\label{tab:figures-inventory}
\end{table*}

\begin{table*}[ht!]
\centering
\begin{tabular}{|c|c|}
\hline
\multicolumn{1}{|c|}{\textit{Dataset}}                                  & \textit{Tables} 
\\ \hline 
 \begin{tabular}[c]{@{}c@{}}\multicolumn{1}{c}{\texttt{Adult}} \end{tabular} & \ref{tab:adult_rest}, \ref{tab:adult_fairlp}  \\ \hline
\multicolumn{1}{|c|}{\texttt{Drug}}       & \ref{tab:drug_rest}, \ref{tab:drug_fairlp}  
\\ \hline
\multicolumn{1}{|c|}{\texttt{Law School}}   & \ref{tab:lawschool_rest}, \ref{tab:lawschool_fairlp} 
\\ \hline
\multicolumn{1}{|c|}{\texttt{Communities and Crime}}& \ref{tab:crime_rest}, \ref{tab:crime_fairlp}
\\ \hline
\end{tabular}
\caption{List of all tables for all experiments across the four datasets.}\label{tab:tables-inventory}
\end{table*}

\begin{table*}[ht!]
\begin{tabular}{|c|c|c|c|c|}
\cline{1-5}
\multicolumn{1}{|c|}{Dataset}         & Input features                                                                                                       & Legitimate feature & Sensitive feature & Target   \\ \hline 
\begin{tabular}[c]{@{}c@{}}\multicolumn{1}{c}{\texttt{Adult}} \end{tabular}  & \begin{tabular}[c]{@{}c@{}}Age, Workclass, Fnlwgt, \\Education, Education-num, \\ Capital-loss,\ Race, Capital-gain, \\Relationship,  Hours-per-week, \\ Native-country, Marital-status \end{tabular} 
  & Occupation  &    Sex       &  Income \\ \hline
\multicolumn{1}{|c|}{\texttt{Drug}}       & \begin{tabular}[c]{@{}c@{}}Education, Country, Nscore, \\ Escore, Oscore, Ascore, \\ Cscore, Impulsive, SS\end{tabular} & Education          & gender            & Cannabis \\ \hline
\multicolumn{1}{|c|}{\texttt{Law School}} & \begin{tabular}[c]{@{}c@{}} ugpa, fam\_inc, fulltime, \\ tier, lsat\_decile \end{tabular}                                                                        & lsat\_decile       & gender            & zfygpa   \\ \hline
\multicolumn{1}{|c|}{\begin{tabular}[c]{@{}c@{}}\texttt{Communities} \\\texttt{and Crime}\end{tabular} } &          124 features (see Section \ref{sec:datasets-details-app})                                                              &   {medIncomeDecile}   &   {racepctblack}   &  {ViolentCrimesPerPop}  \\ \hline
\end{tabular}
\vspace{0.2cm}
\caption{List of variables for \texttt{Adult}, \texttt{Drug}, \texttt{Law School}, and \texttt{Communities and crime} datasets.}
\label{tab:dataset_features}
\end{table*}

\subsection{Implementation Details}\label{sec:implementation-details-app}

In all experiments we use a two-layer deep multi-layer perceptron (MLP), with hidden layers including 50 and 20 nodes, and a rectified linear unit (ReLu) activation function. We train the MLP using stochastic Adam optimizer, with 500 training epochs and a learning rate of $l_r = 0.001$, for the all methods and datasets. All methods are equipped with early stopping, with a patience set to $50$ epochs. We chose these values since 
they resulted in the convergence of the training procedure for all methods under consideration to non-trivial predictors, with reasonable model variance across 10 runs. The constant $C$ in the computation of \cddw{} (see Proposition \ref{prop:bcd_cdd}) is set to $C=1$ for all four datasets, which is larger than the lower bound $\frac{\bar d}{\underline d}$ in all cases and hence satifies Proposition~\ref{prop:bcd_cdd}. Finally, during training we apply stratified sampling for generating each training batch to ensure at least one sample from each level of the legimitate feature $L$ is included on a given batch. To do so, we set the batch size to $256$ for \texttt{Drug}, $1024$ for \texttt{Adult} and $128$ for \texttt{Law School} and \texttt{Crime}.

For \textit{Wasserstein regularization} and \fairlp{} we compute the 1-Wasserstein distance, i.e., setting $p=1$. For \fairbit{}, we compute the Sinkhorn approximation using the \texttt{sinkhorn} and \texttt{sinkhorn2} functions from the Python optimal transport (POT) package \cite{flamary2021pot}, which optimizes for the 2-Wasserstein distance, i.e., $p=2$ for both the inner and outer metrics in \cddwtext{}. The details for each method are as follows:

\begin{itemize}
    \item For \fairbit{} and \textit{Wasserstein regularization}, values of $\lambda_b$ and $\lambda_w$ are 10 logarithmically-spaced numbers in the range $[0.0001, 10]$:
    \begin{align*}
        \lambda_b, \lambda_w \in \{0.0001, 
0.0005, 
0.0022, 
0.0100, 
0.0464, \\
0.2154, 
1.0000, 
4.6416, 
21.5443, 
100\}
    \end{align*}
    \item For all 3 versions of \fairlp, values of $\lambda_{lp}$ are 10 logarithmically-spaced numbers in the range $[0.0001, 100]$:
    \begin{align*}
        \lambda_b, \lambda_w \in \{0.0001, 
0.0004, 
0.0013, 
0.0046, 
0.0167,\\ 
0.0599,
0.2154, 
0.7743, 
2.7826, 
10 \}
    \end{align*}
    \item For Pre-processing repair, we select $\lambda_r \in [0,1]$, with the following values:
    $$\lambda_r \in \{ 0.0, 0.1, 0.2, 0.3, 0.4, 0.5, 0.6, 0.7, 0.8, 0.9, 1.0\}$$

    \item For DCFR, we consider the same range $(0,20]$ as in \cite{xu2020algorithmic}, with the following values:
    \begin{align*}
        \lambda_{DCFR} \in \{0.1, 0.25, 0.5, 0.75, 1.0, 1.5,\\ 2.0, 5.0, 10.0, 15.0, 20.0\}
    \end{align*}
\end{itemize}

Finally, our \fairbit{} method requires substantial computation due to our non-optimized Python implementation, unlike all other approaches which utilize direct calls from the Python Optimal transport package (POT, \citealt{flamary2021pot}), which are natively written in C++ with a Python binding. For instance, it requires around 2 hours to train one model using the full \texttt{Law School} dataset on a machine equipped with 32GB of RAM, a 16GB Tesla T4 and a 2.50GHz Intel(R) Xeon(R) Platinum 8259CL CPU, while for \fairlp{}, as well as other methods, the same training is executed in less than 10 minutes.

\subsection{Additional Results}\label{subsec:app:additional_results}


Figure~\ref{fig:empirical-comparison-results_including_legit} includes the performance of all methods (including \textit{legit-only}) with fairness metrics \cddwtext{} (normalized version) on the top row and \cddlptext{} (uniform) on the bottom row. We normalize the \cddwtext{} distance by substracting the Wasserstein distance between the distribution of the legitimate feature at the two levels of the sensitive feature, i.e., subtracting $\mathcal{W}(\mathbb{P}(L|A=0), \mathbb{P}(L|A=1))$. This corresponds to the term $C \norm{\tilde{U}-U}^p$ in Definition~\ref{def:bicausal_distance}, which is not dependent on the model $f(X)$ and as such equal across all methods. Figure~\ref{fig:empirical-comparison-results_including_legit} is a repetition of the results presented in Figure \ref{fig:empirical-comparison-results} in the main body, with the addition of the results for \textit{legit-only}. \textit{Legit-only} completely removes conditional disparities but significantly suffers in terms of predictive performance. This is expected behavior since it relies only on a few (in our experiments, only 1) feature(s).

Figure \ref{fig:empirical-comparison-results_appendix} presents the performance of all methods (including \textit{legit-only}) with fairness metrics \cddlptext{} ($P(L)$), \cddlptext{} (Ave. $P(L|A)$), and demographic parity on the top, middle, and bottom rows, respectively. 
Overall, when fairness is measured by \cddlptext{} metrics, \bcd\ and the variants of \fairlp{} are consistently among the highest performing, often providing better fairness-predictive power trade-offs than the other proposed methods. When fairness is measured by DD, unsurprisingly \textit{Wasserstein Reg.} performs the best. In this case, DCFR has good performance especially on regression datasets with some points on the Pareto frontier. \fairbit{} and the variants of \fairlp{} generally do not perform as well on demographic disparity across all datasets, although they are still part of or close to the Pareto frontier in classification settings.

\subsection{Results Details}\label{sec:result-details-app}

We provide the mean and standard deviation results for all our methods (for all regularization parameter values where applicable) in Tables \ref{tab:adult_rest} through \ref{tab:crime_fairlp}.  
Tables~\ref{tab:figures-inventory} and \ref{tab:tables-inventory} below provide exhaustive pointers to the respective result figures and tables for each experiment across different datasets and fairness metrics. See Sections~\ref{section:experiments}, \ref{subsec:app:additional_results} and Figures captions for comments and takeaways on models performance.

\subsection{Discussion on our choices of evaluation metrics for predictive performance}\label{sec:result-considerations-app}

While some previous papers use the AUC of the performance-fairness curves (see e.g., \citealt{rychener2022metrizing}) to describe the overall fairness-performance behavior of each predictor, in our experiments we find this measure to be misleading since the ranges of the fairness and performance results are different across different methods. As we are not aware of any single measure that can appropriately summarize the overall fairness-performance behavior of a predictor, we report results as the means and standard deviations over 10 runs across separate random seeds for each method. 

\begin{figure*}[!t]
    \centering
    \begin{subfigure}{.215\textwidth}
        \centering
        \includegraphics[width=\linewidth]{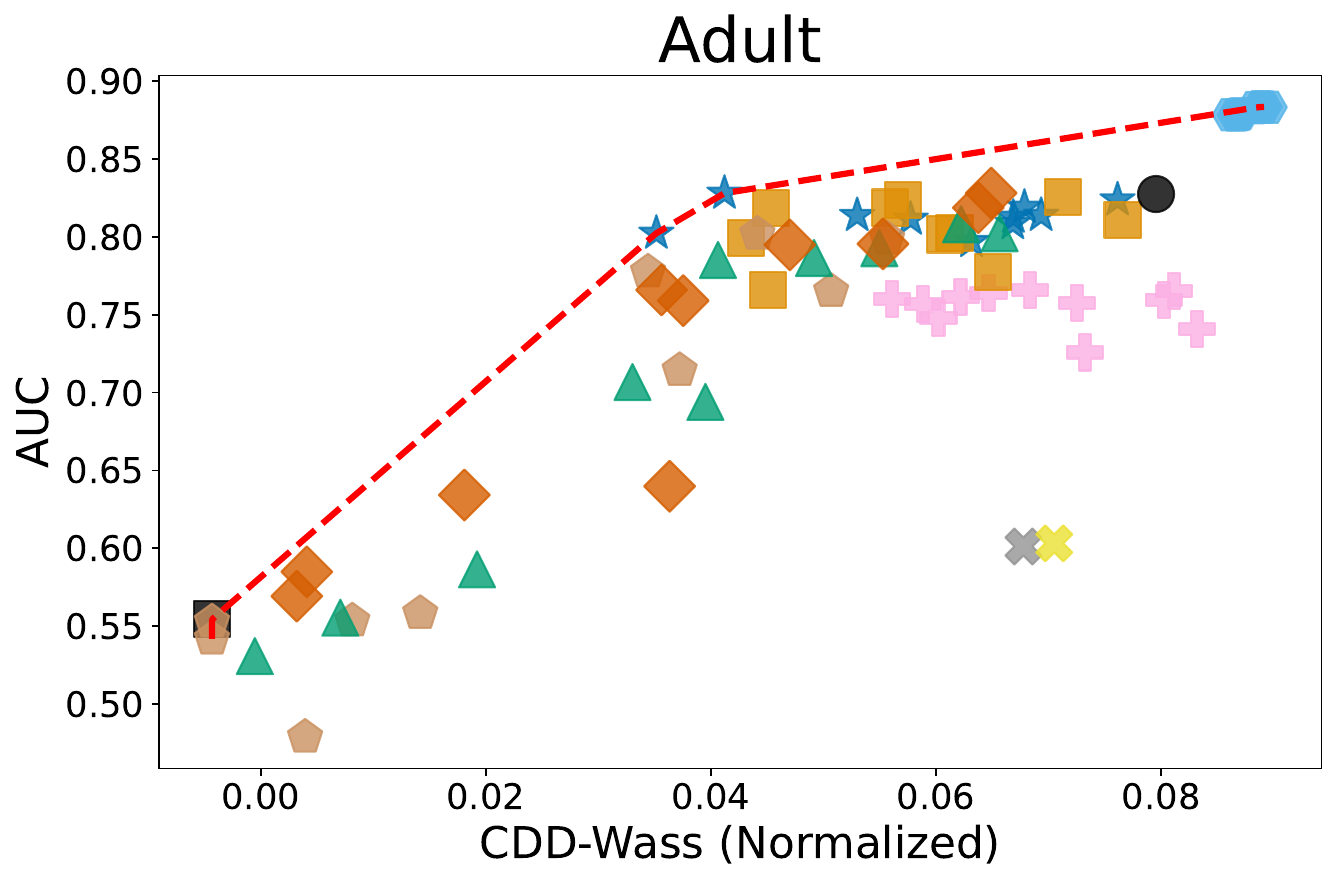}
    \end{subfigure}%
    \begin{subfigure}{0.215\textwidth}
        \centering
        \includegraphics[width=\linewidth]{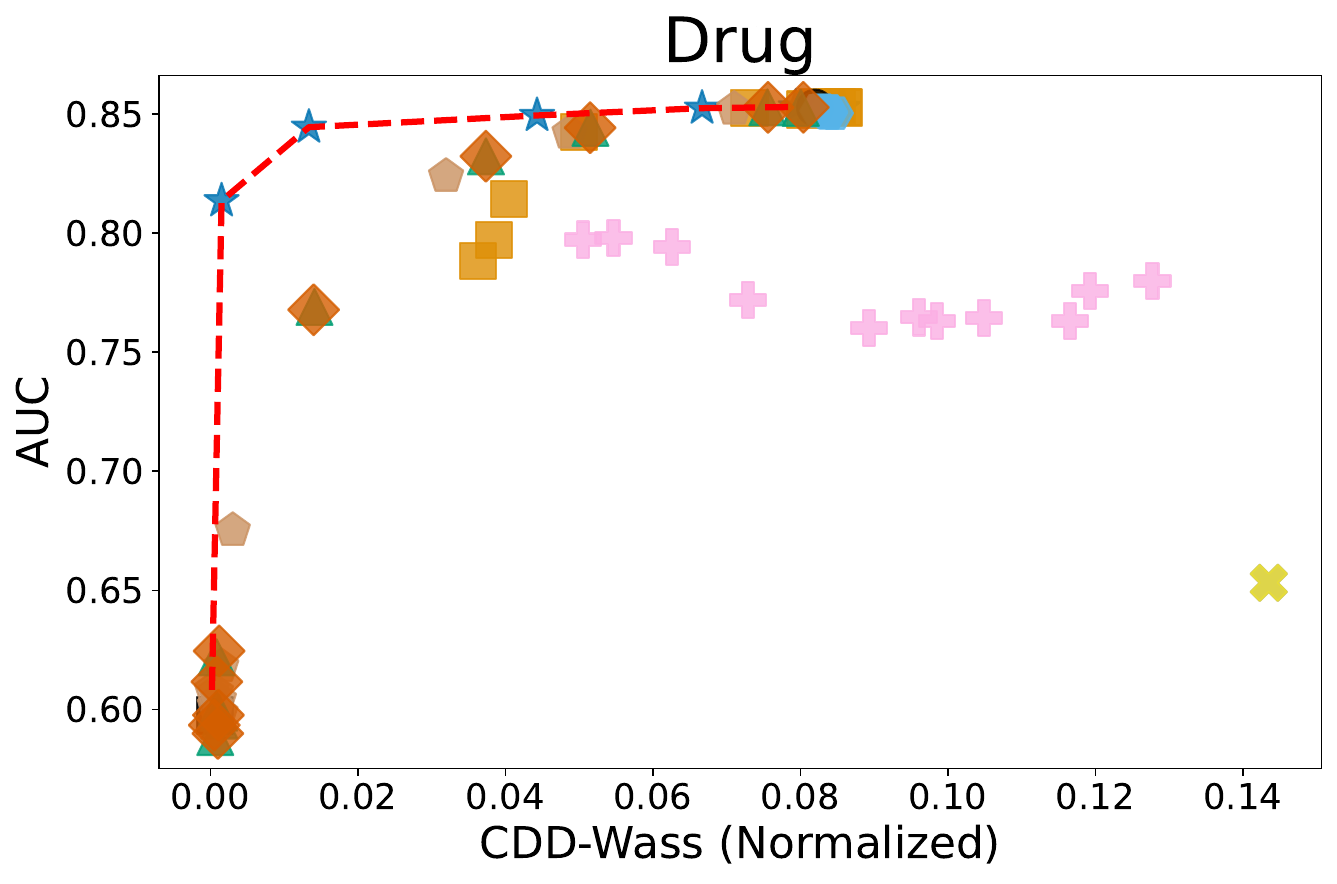}
    \end{subfigure}
    \begin{subfigure}{.225\textwidth}
        \centering
        \includegraphics[width=\linewidth]{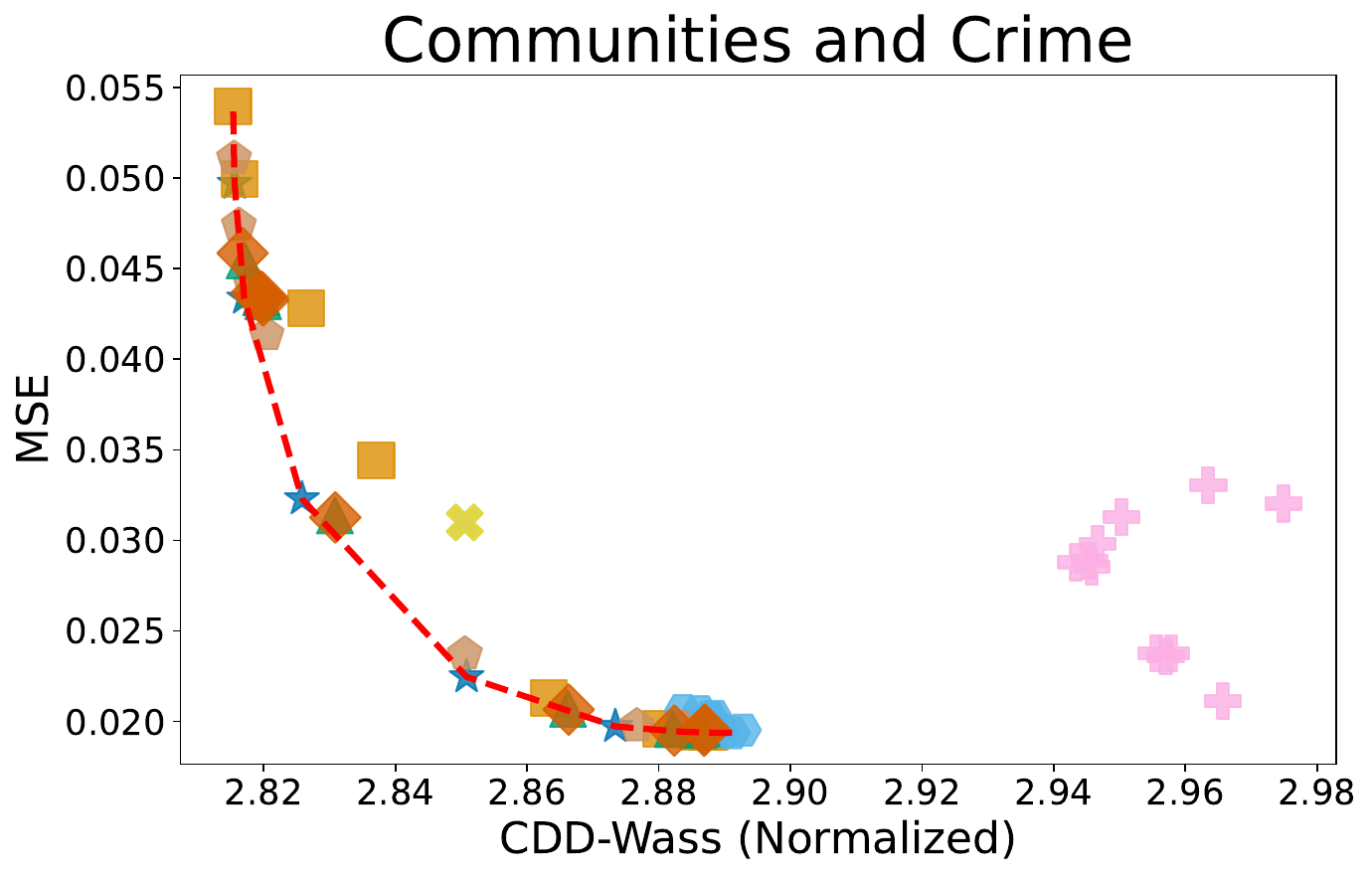}
    \end{subfigure}%
    \begin{subfigure}{0.31\textwidth}
        \centering
        \includegraphics[width=\linewidth]{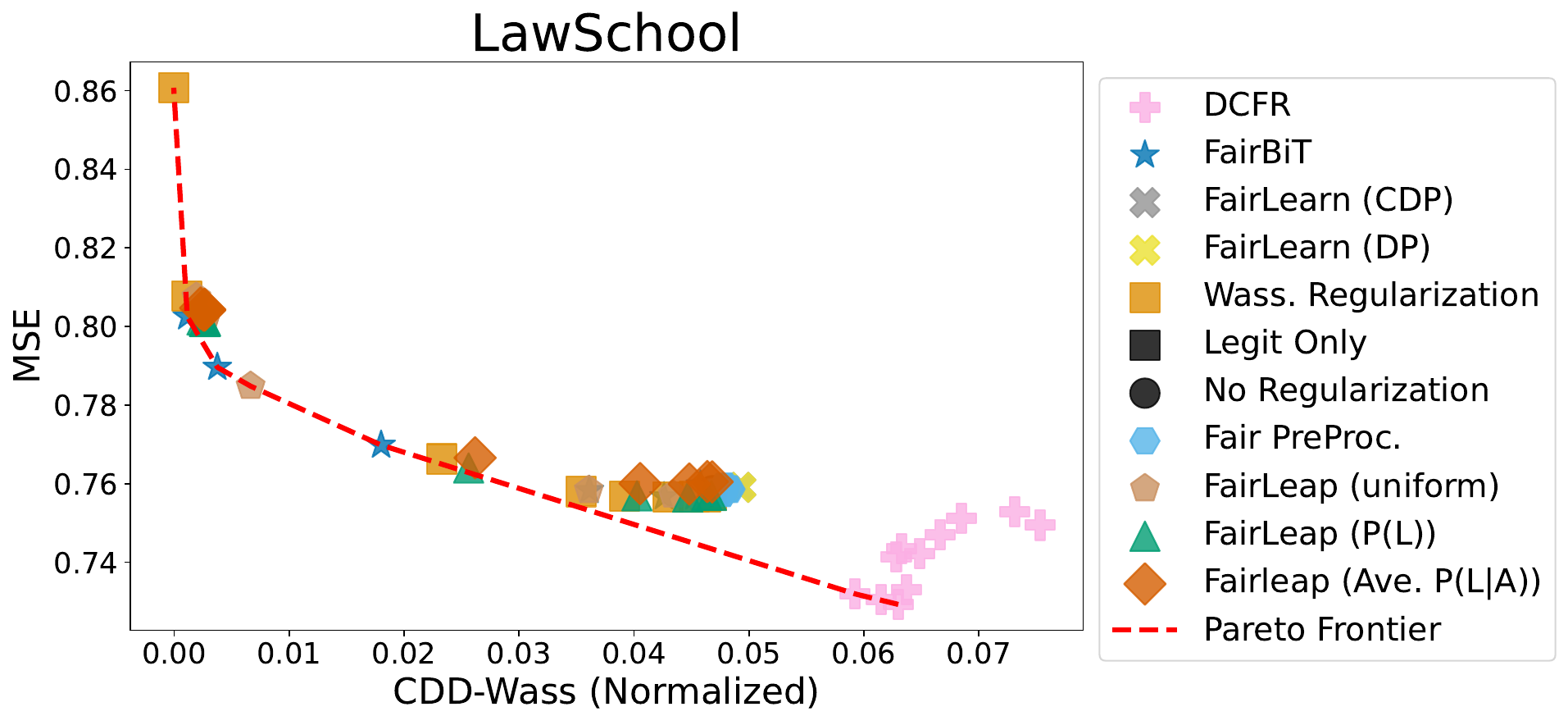}
    \end{subfigure}

    \begin{subfigure}{.215\textwidth}
        \centering
        \includegraphics[width=\linewidth]{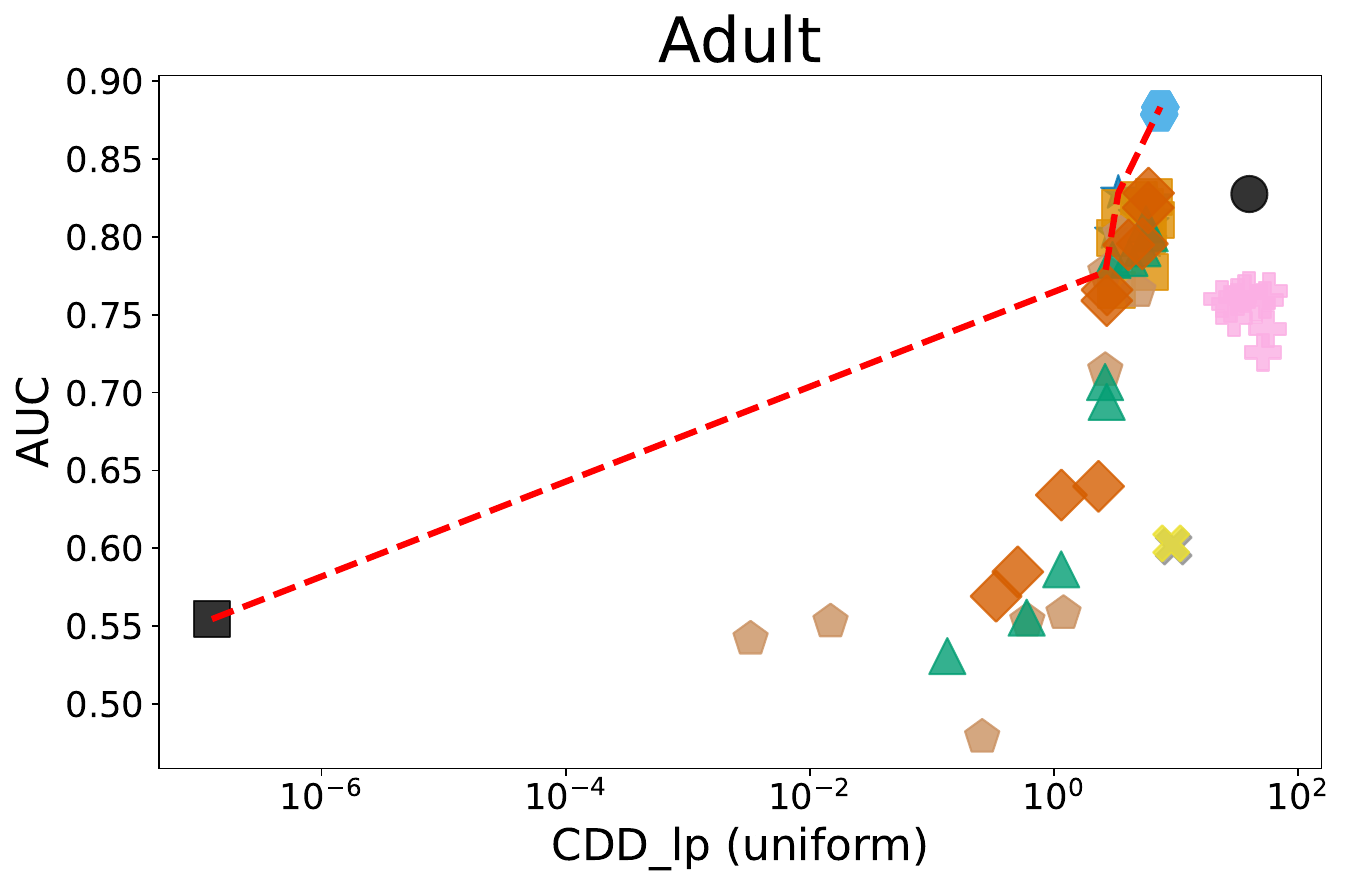}
    \end{subfigure}%
    \begin{subfigure}{0.215\textwidth}
        \centering
        \includegraphics[width=\linewidth]{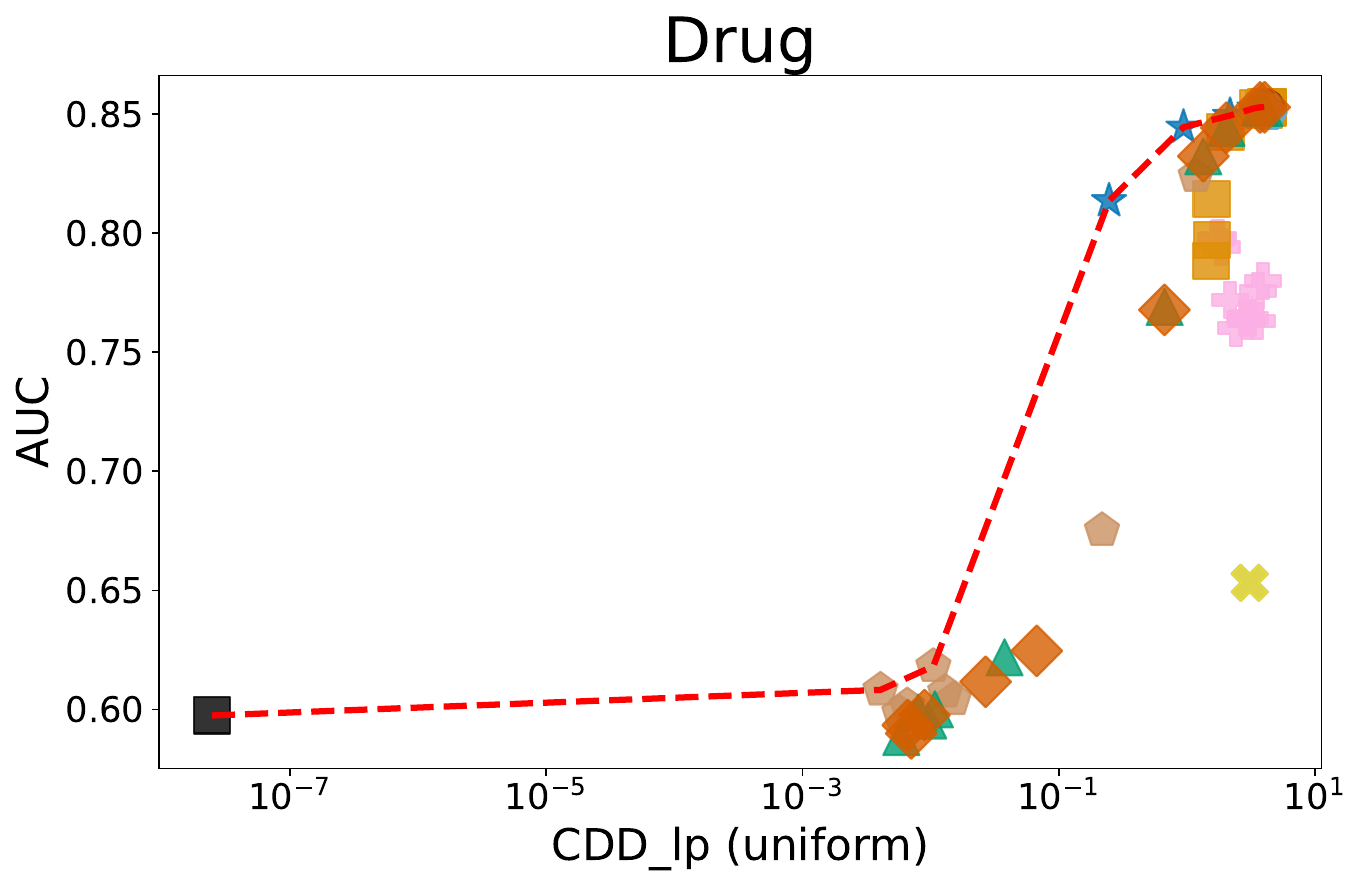}
    \end{subfigure}
    \begin{subfigure}{.225\textwidth}
        \centering
        \includegraphics[width=\linewidth]{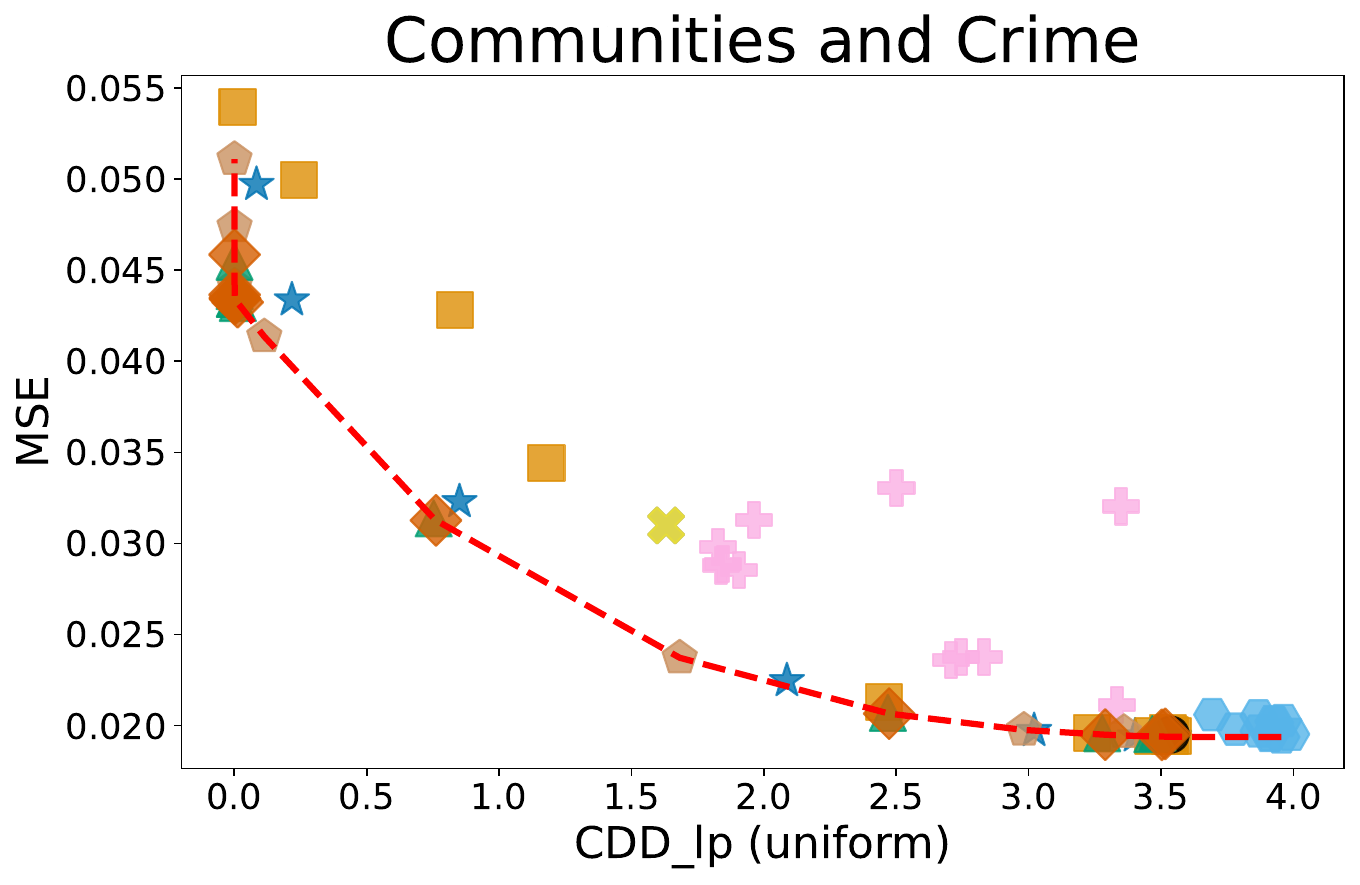}
    \end{subfigure}%
    \begin{subfigure}{0.31\textwidth}
        \centering
        \includegraphics[width=\linewidth]{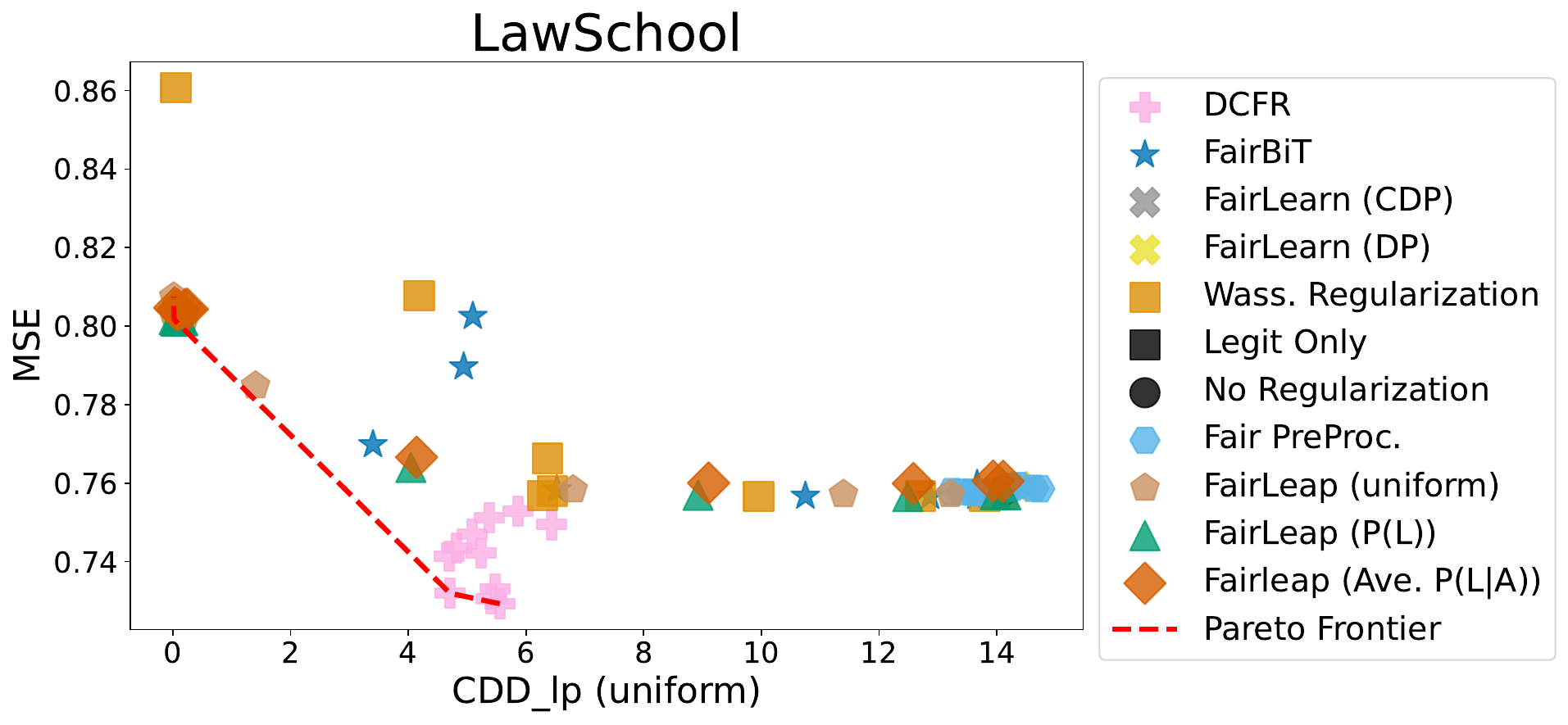}
    \end{subfigure}
    \caption{
    Fairness-predictive power trade-offs for \texttt{Adult}, \texttt{Drug}, \texttt{Communities and Crime}, and \texttt{LawSchool} datasets. The figures in the top row present the results when fairness is measured by \cddw{}. The results when fairness metric is \cddlp{} (with $\mb Q(L) = \mb U(L)$ and $p=1$) are presented in the bottom row. Predictive power (PP) is measured by AUC for Classification and MSE for regression. Results are averaged over 10 runs, with different values for the same methods due to different hyper-parameter settings; see Appendix~\ref{sec:app:exp_details} for details. These figures include \textit{legit-only} as a reference. \textit{Legit-only}, unsurprisingly, achieves full conditional parity but significantly suffers in terms of predictive performance.}
    \label{fig:empirical-comparison-results_including_legit}
\end{figure*}

\begin{figure*}[!t]
    \centering

    \begin{subfigure}{.215\textwidth}
        \centering
        \includegraphics[width=\linewidth]{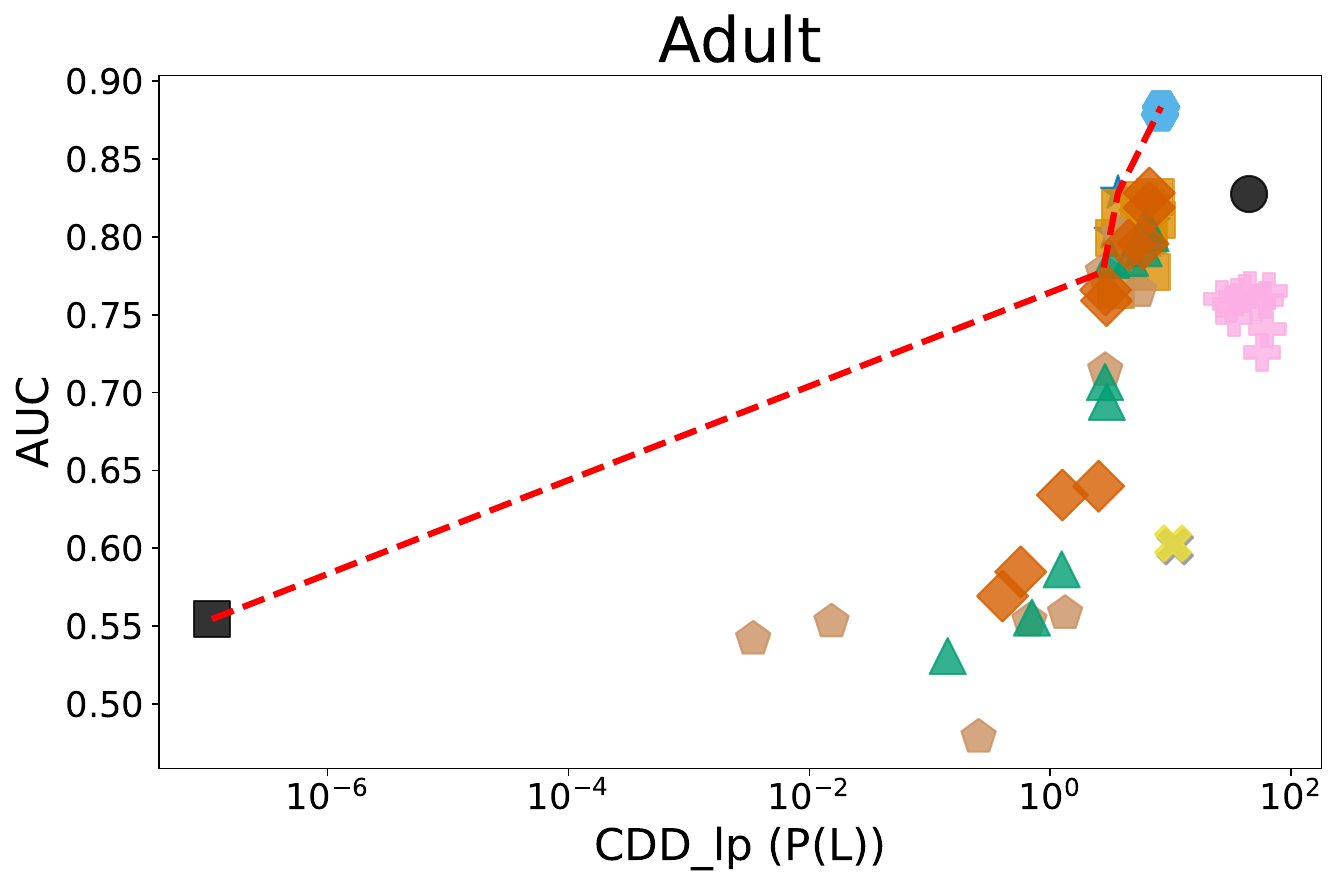}
    \end{subfigure}%
    \begin{subfigure}{0.215\textwidth}
        \centering
        \includegraphics[width=\linewidth]{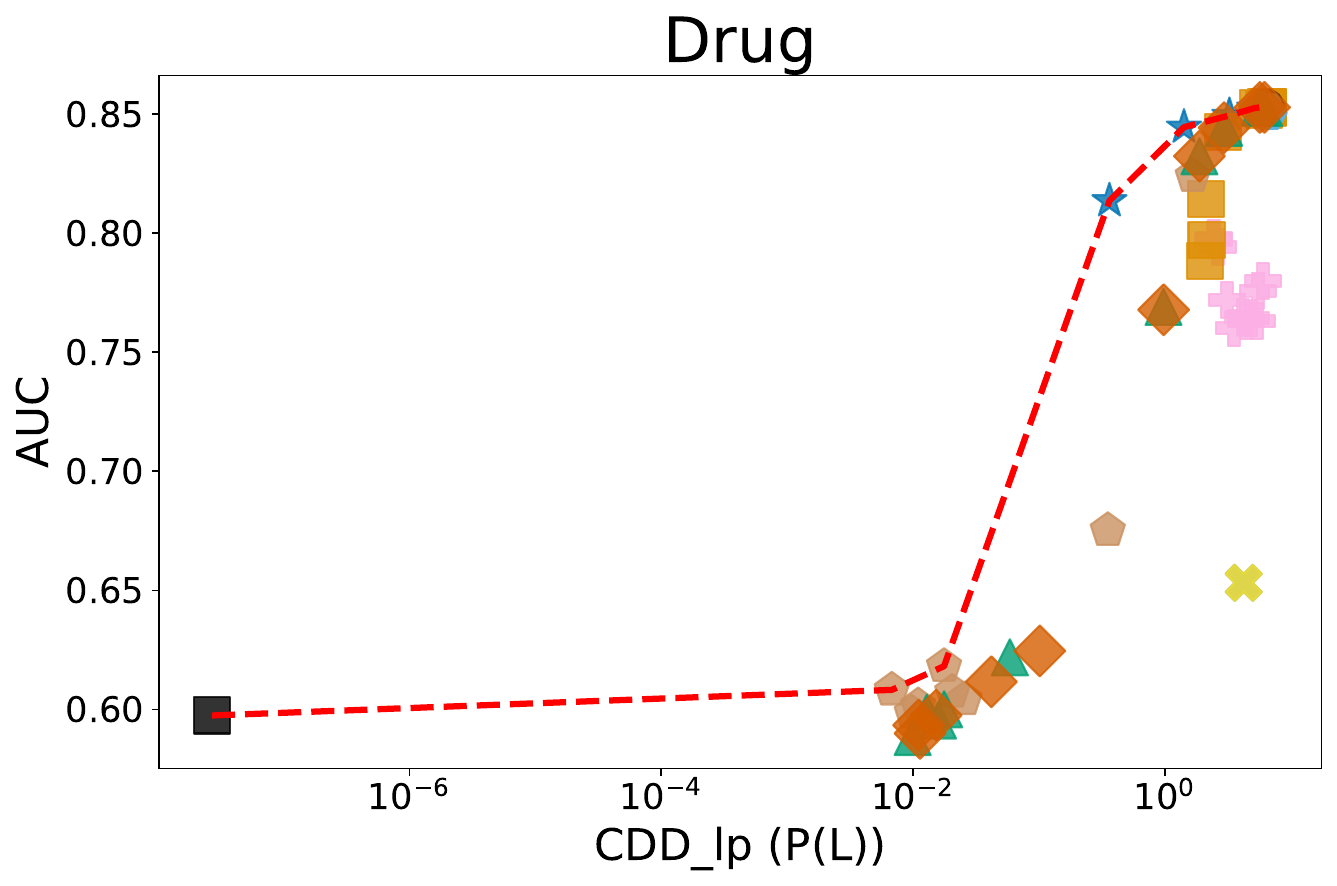}
    \end{subfigure}
    \begin{subfigure}{.225\textwidth}
        \centering
        \includegraphics[width=\linewidth]{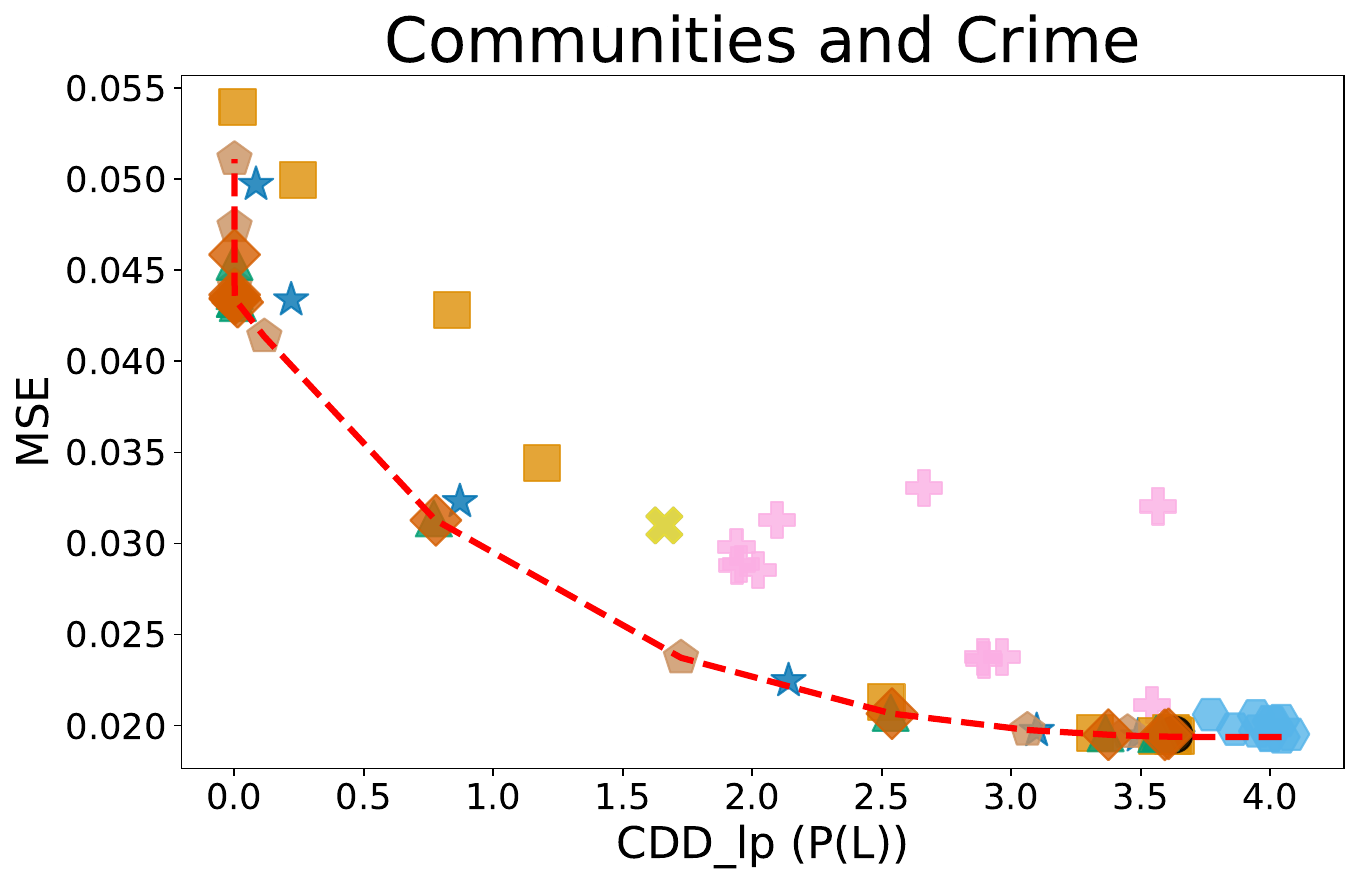}
    \end{subfigure}%
    \begin{subfigure}{0.31\textwidth}
        \centering
        \includegraphics[width=\linewidth]{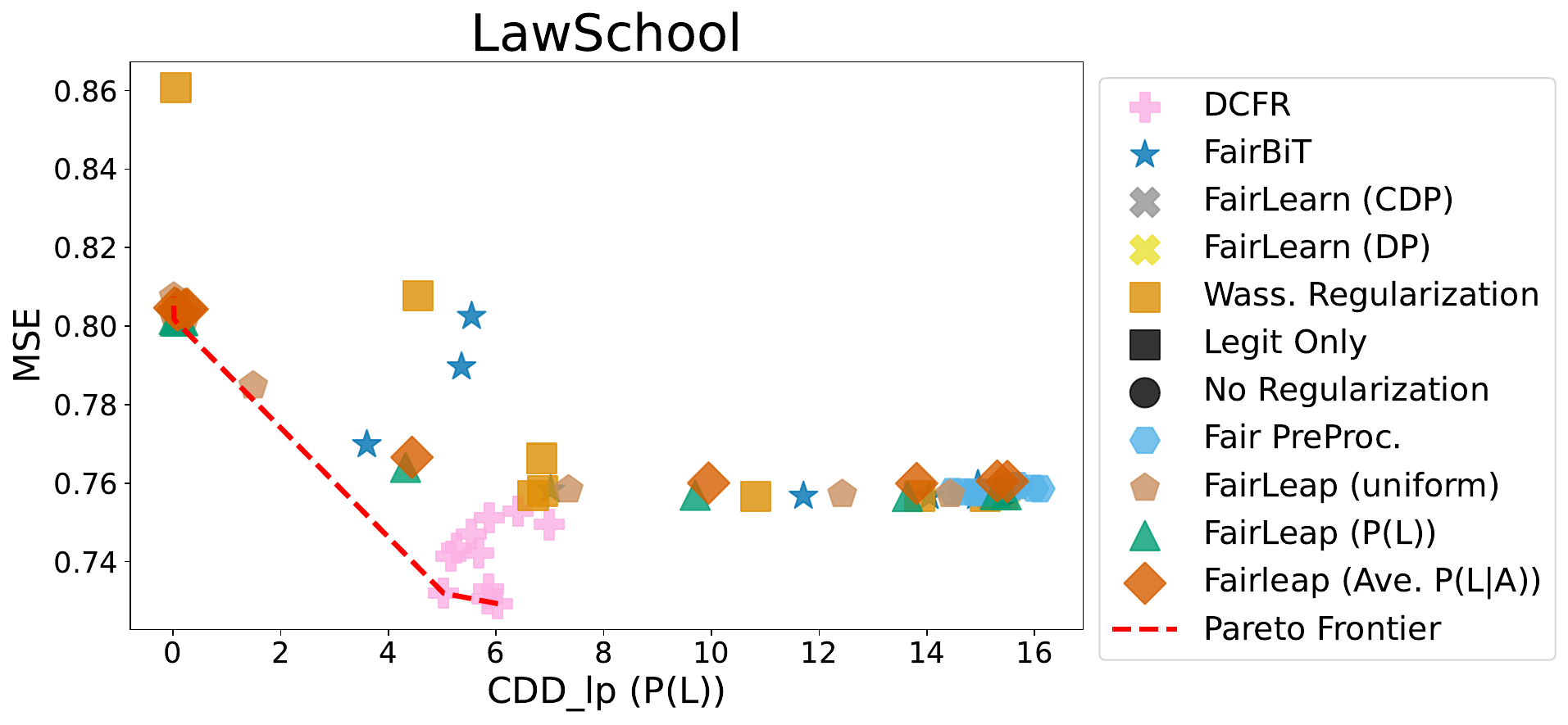}
    \end{subfigure}
     
    \begin{subfigure}{.215\textwidth}
        \centering
        \includegraphics[width=\linewidth]{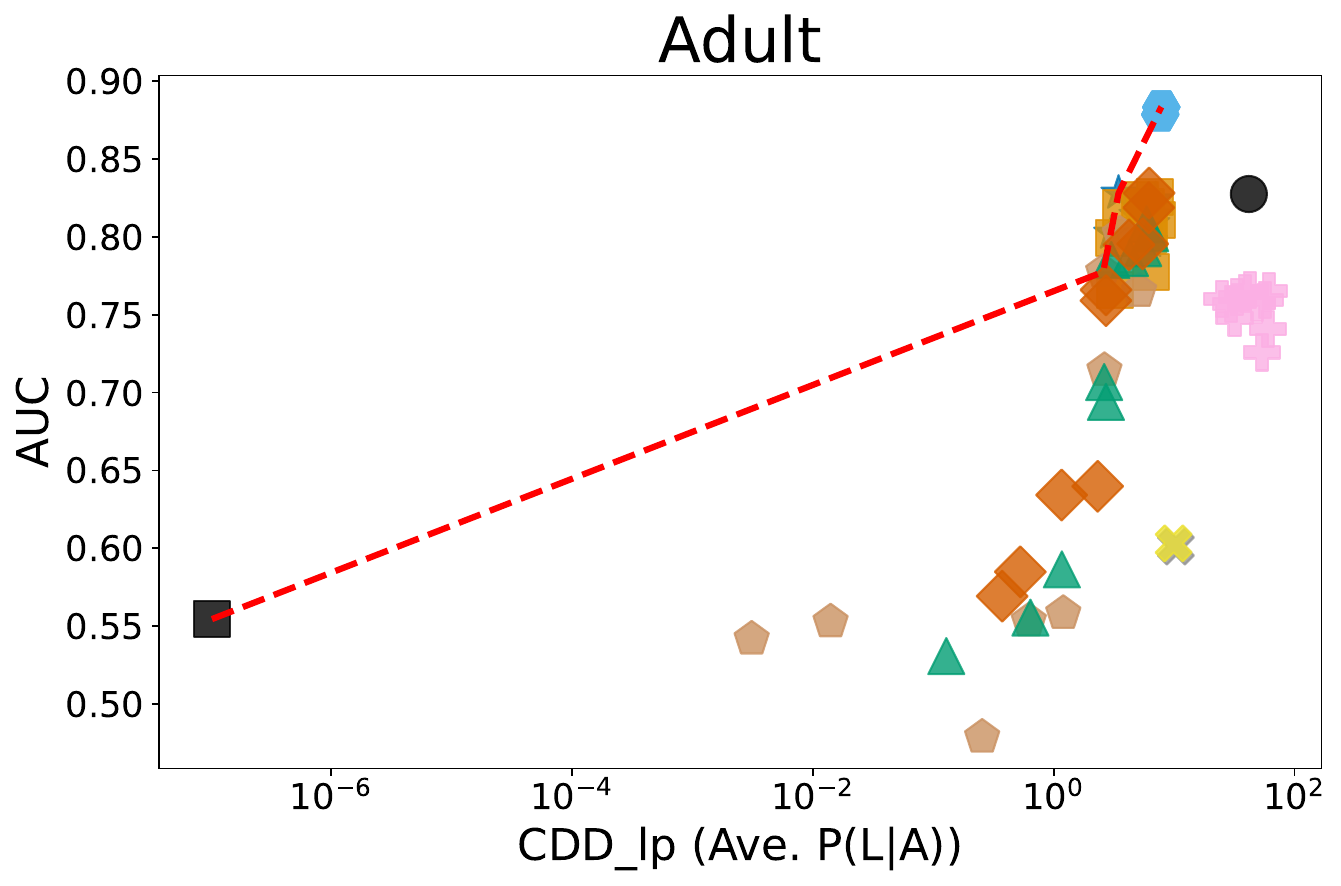}
    \end{subfigure}%
    \begin{subfigure}{0.215\textwidth}
        \centering
        \includegraphics[width=\linewidth]{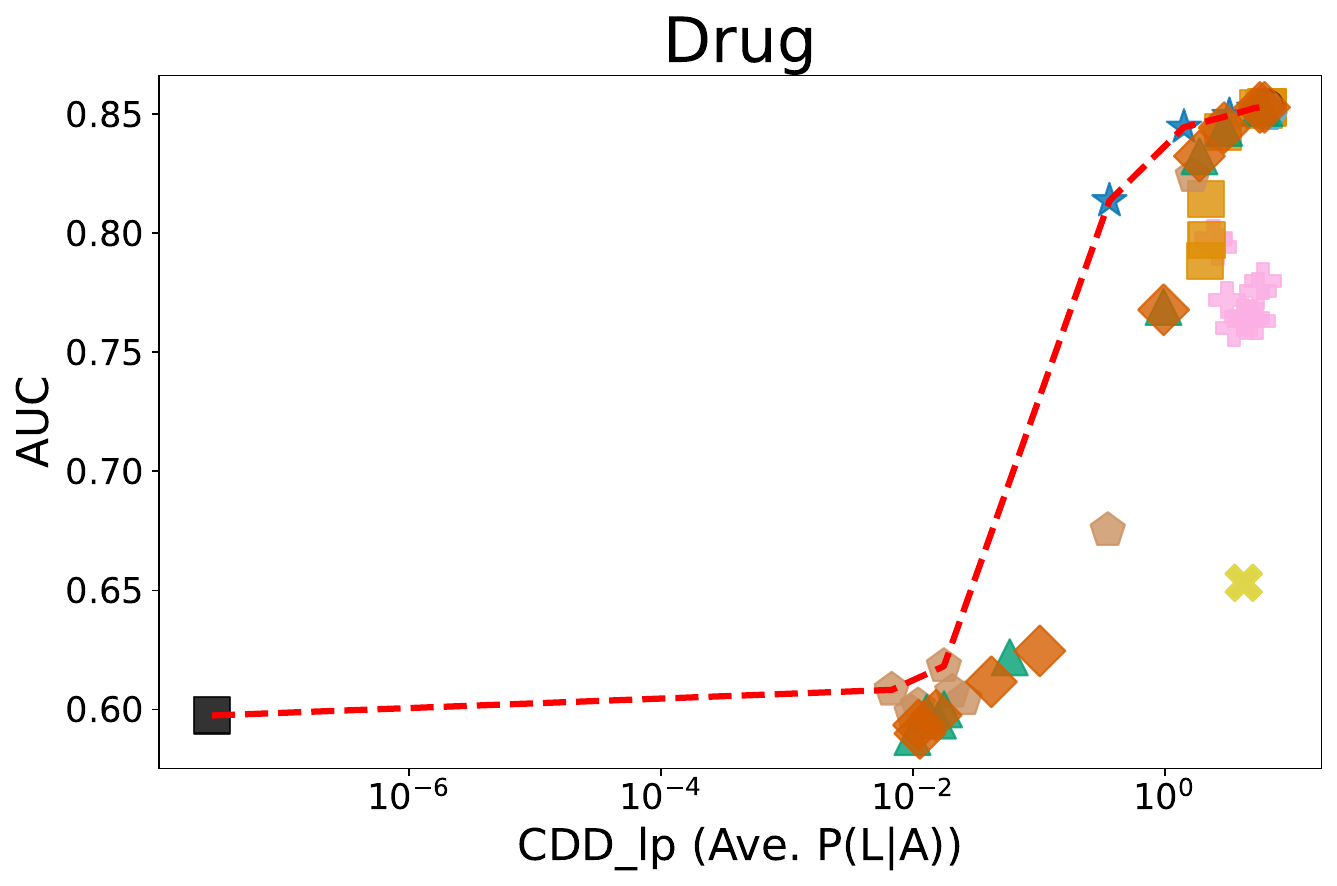}
    \end{subfigure}
    \begin{subfigure}{.225\textwidth}
        \centering
        \includegraphics[width=\linewidth]{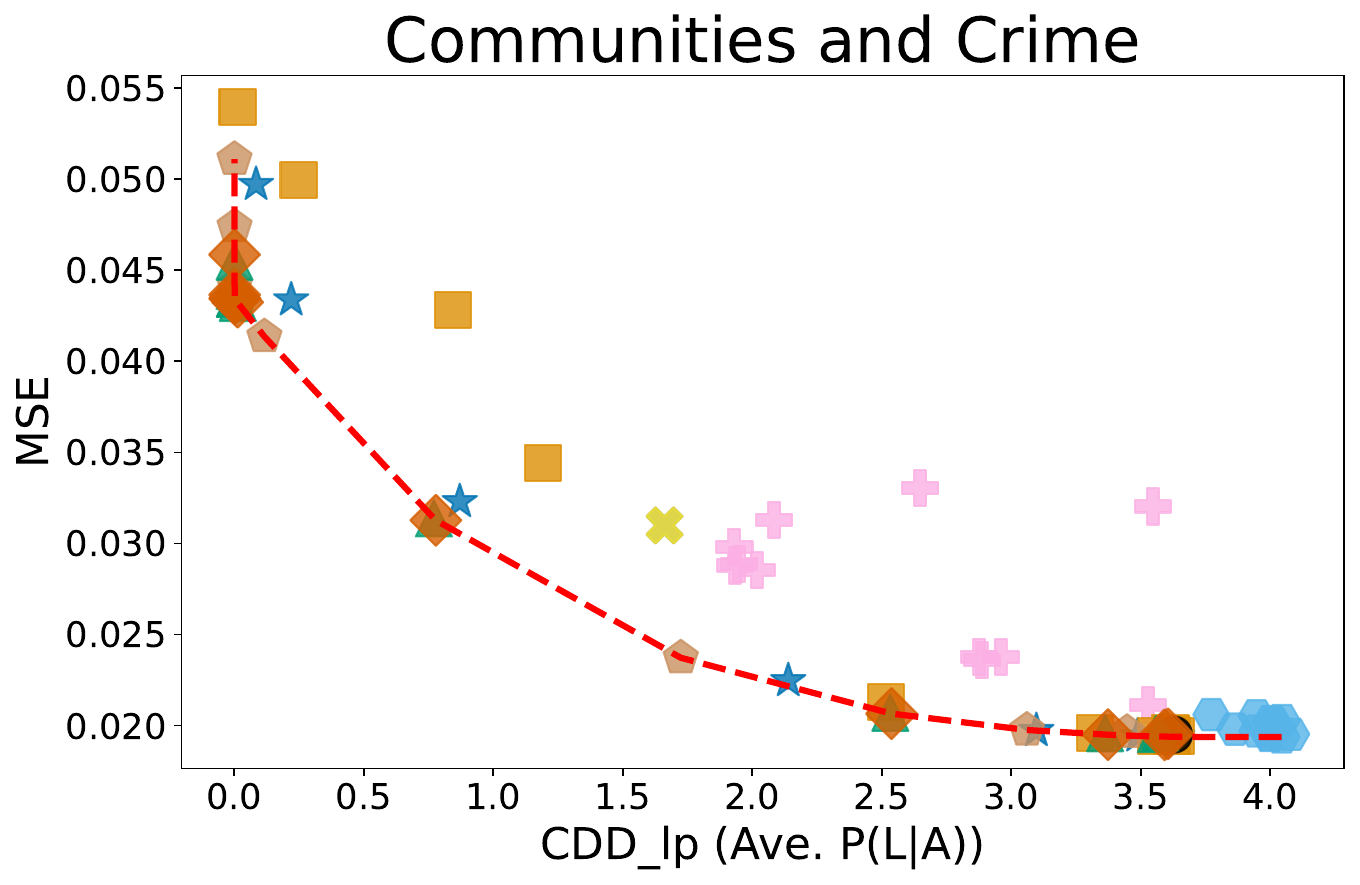}
    \end{subfigure}%
    \begin{subfigure}{0.31\textwidth}
        \centering
        \includegraphics[width=\linewidth]{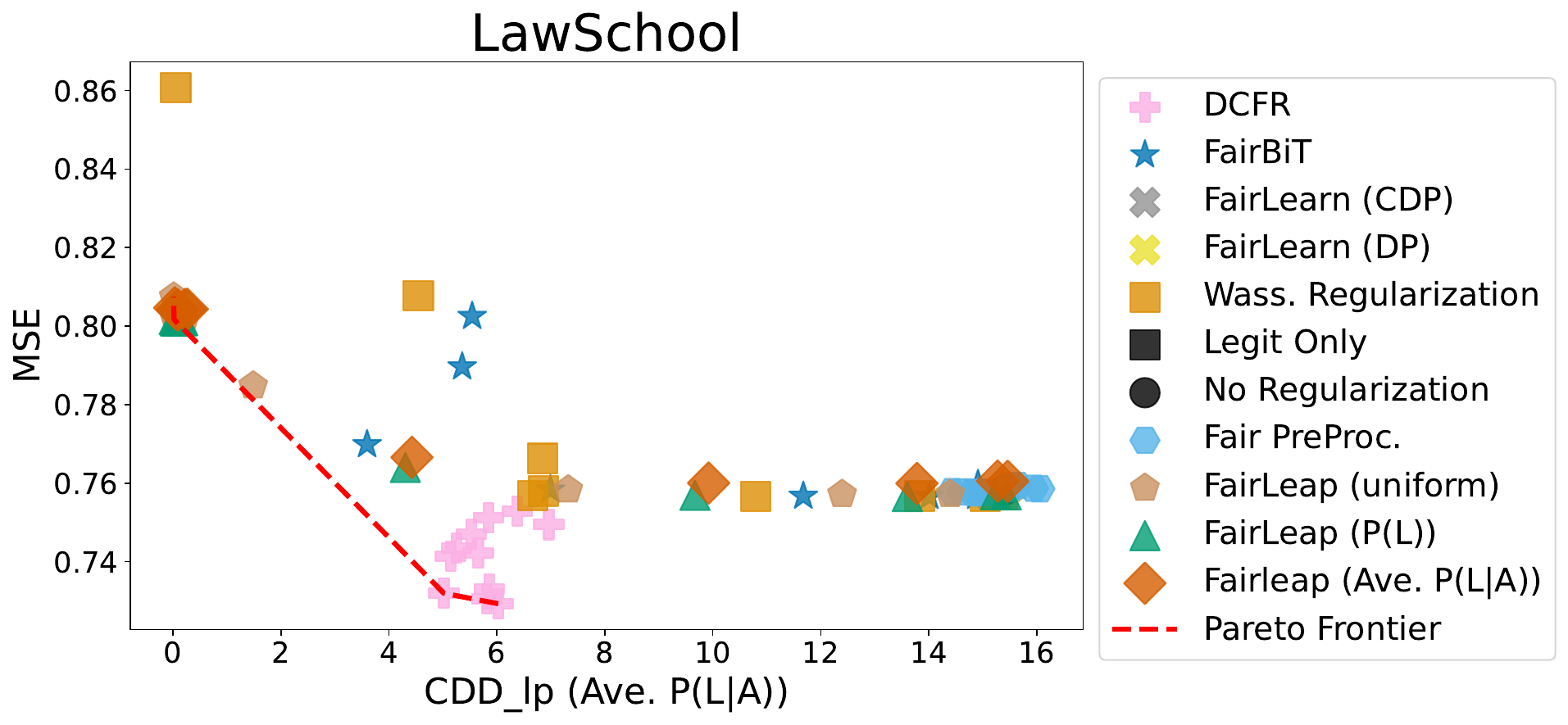}
    \end{subfigure}

    \begin{subfigure}{.215\textwidth}
        \centering
        \includegraphics[width=\linewidth]{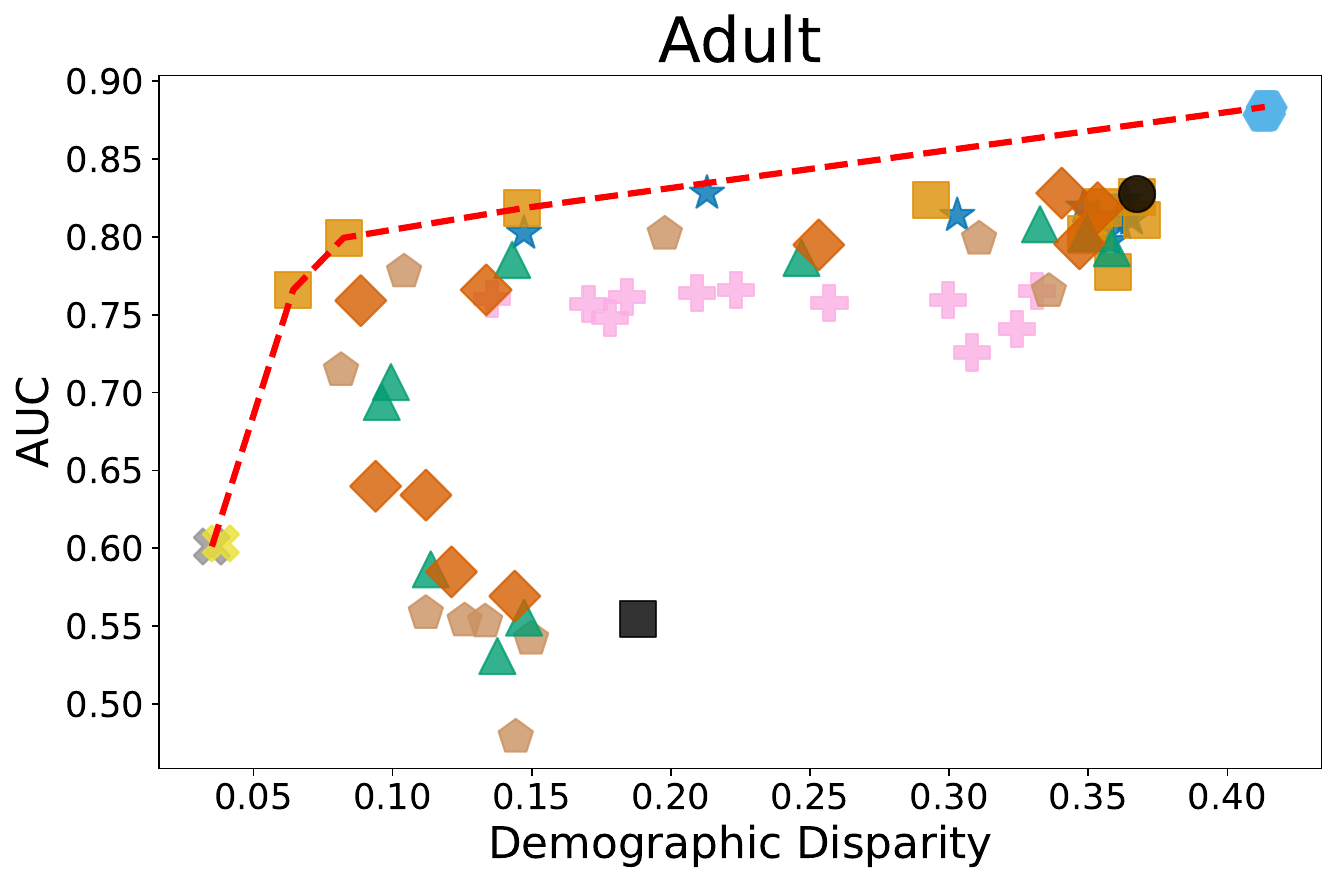}
    \end{subfigure}%
    \begin{subfigure}{0.215\textwidth}
        \centering
        \includegraphics[width=\linewidth]{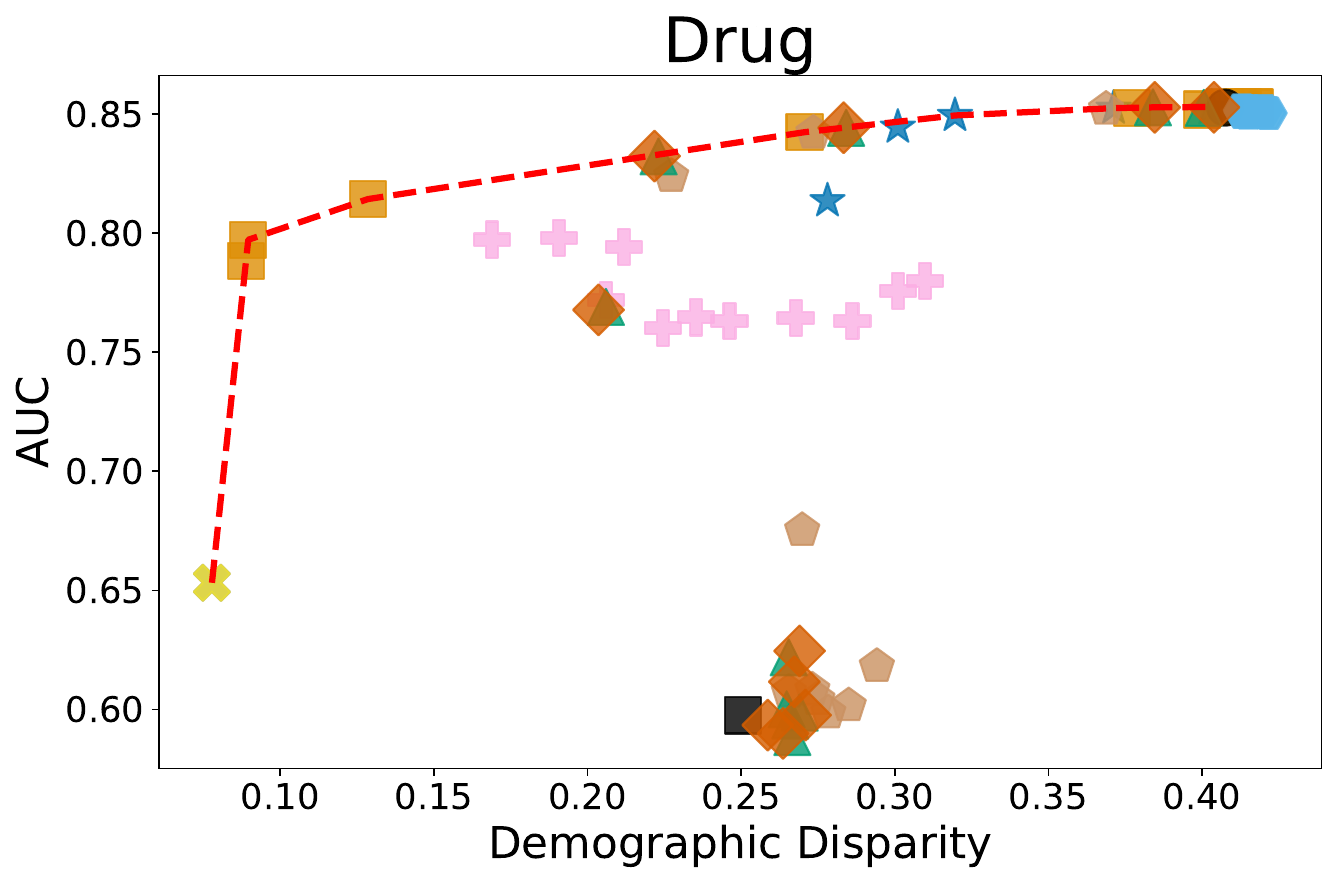}
    \end{subfigure}
    \begin{subfigure}{.225\textwidth}
        \centering
        \includegraphics[width=\linewidth]{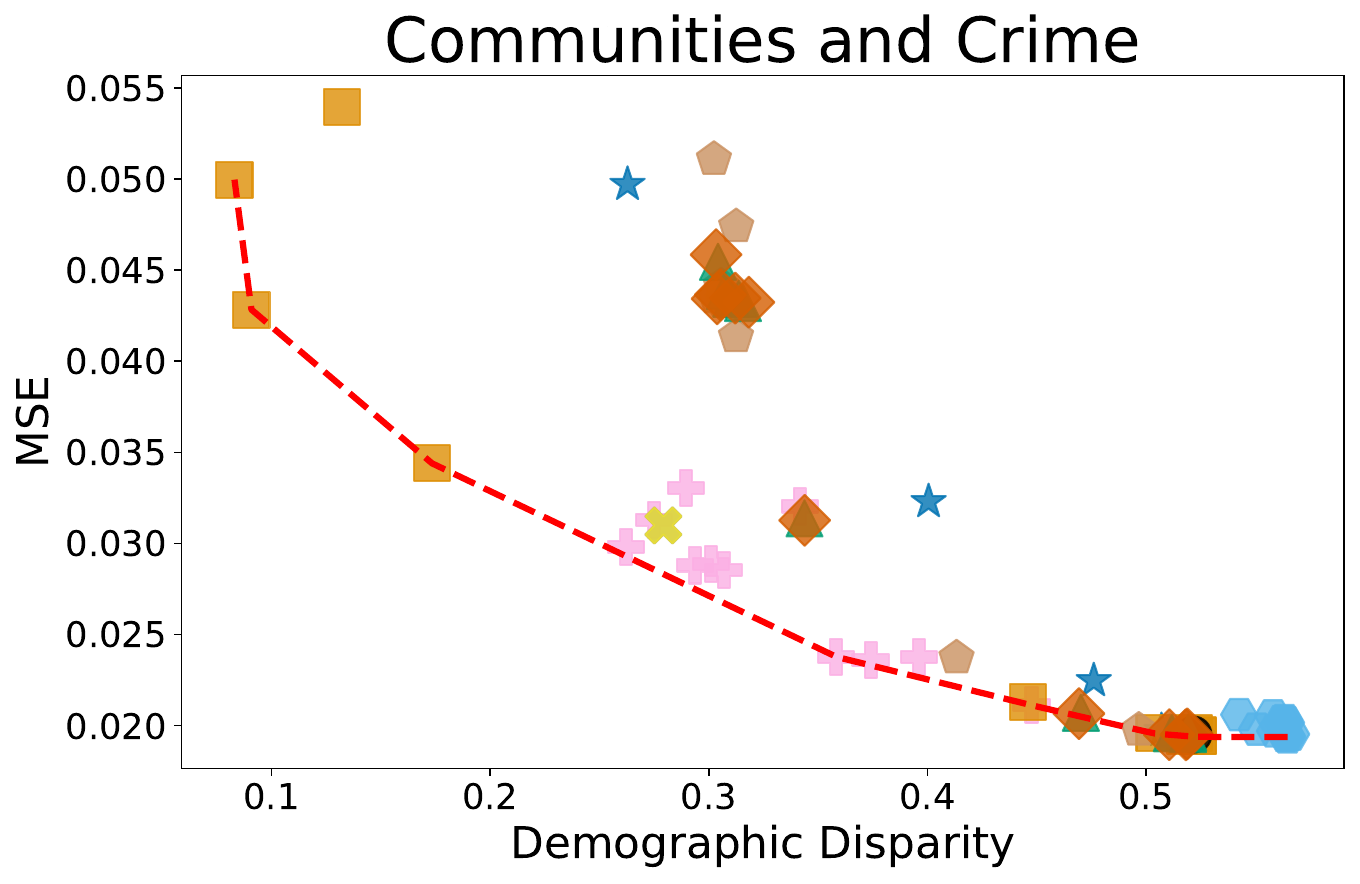}
    \end{subfigure}%
    \begin{subfigure}{0.31\textwidth}
        \centering
        \includegraphics[width=\linewidth]{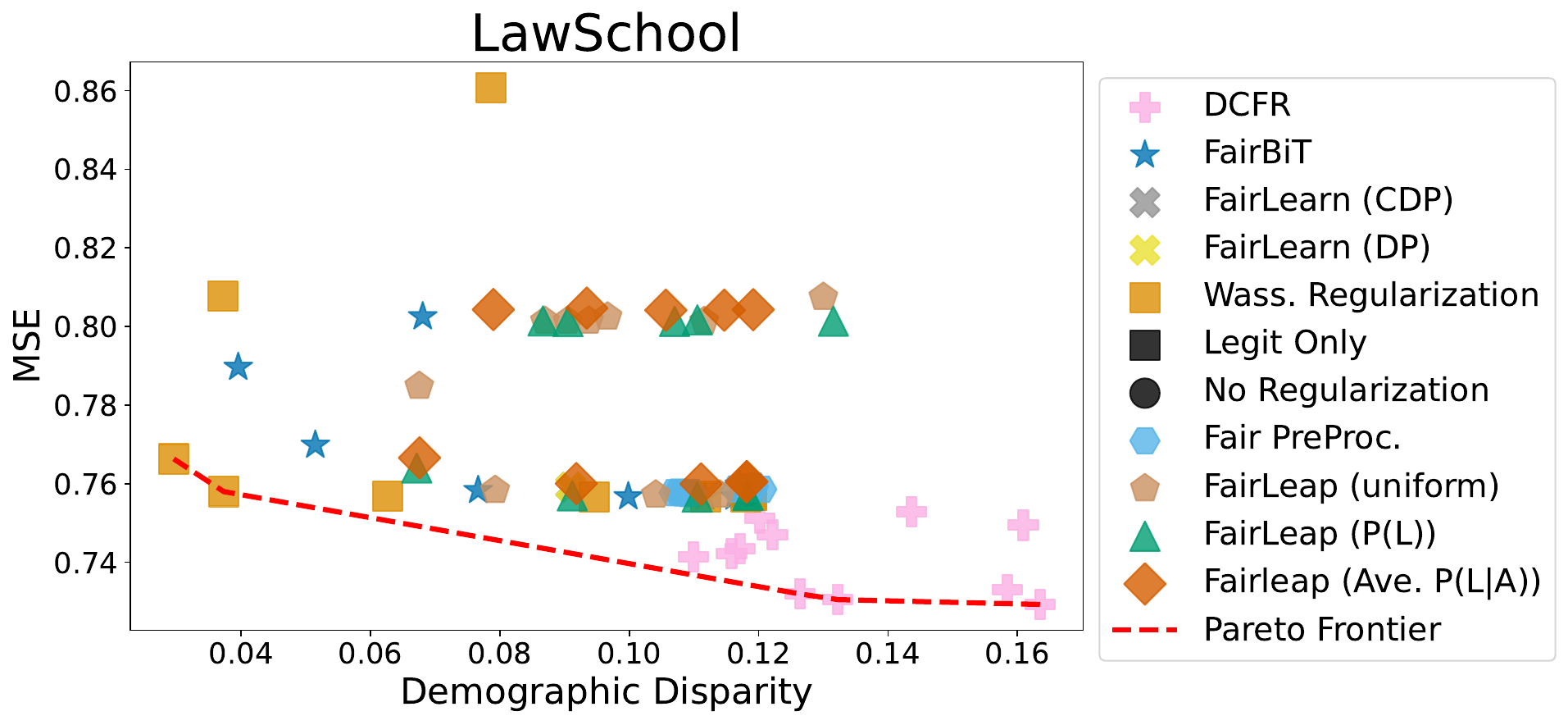}
    \end{subfigure}
    \caption{
    Fairness-performance trade-offs for \texttt{Adult}, \texttt{Drug}, \texttt{Communities and Crime}, and \texttt{LawSchool} datasets. In the top row, we show the results when fairness is measured by \cddlp{} with $\mb Q(L) = \P(L)$ and $p=1$. The figures in the middle row show the results when fairness is measured by \cddlp{} with $\mb Q(L) = \frac{\P(L|A=0)+\P(L|A=1)}{2}$ and $p=1$. The figures in the middle row show the results when fairness is measured by demographic disparity (DD), specifically using the 1-Wasserstein distance. Predictive power is measured by AUC for Classification and MSE for regression. Results are averaged over 10 runs, with different values for the same methods due to different hyper-parameter settings; see Appendix~\ref{sec:app:exp_details} for details. Overall, when fairness is measured by CDD metrics, \bcd\ and the variants of \fairlp{} are consistently among the highest performing, often providing better fairness-predictive power trade-offs than the other proposed methods. When fairness is measured by DD, unsurprisingly \textit{Wasserstein Reg.} performs the best. In this case, DCFR has good performance especially on regression datasets with some points on the frontier. \fairbit{} and the variants of \fairlp{} generally do not improve demographic disparity by much, although they still have many points on the Pareto frontier in classification.}
    \label{fig:empirical-comparison-results_appendix}
\end{figure*}



\begin{table*}[ht]
\centering
\resizebox{1.0\textwidth}{!}{%
\begin{tabular}{|c||c|c|c|c|c|}
\hline
Method & AUC  &  \begin{tabular}[c]{@{}c@{}}\cddwtext{}\\ (normalized) \end{tabular}  & \begin{tabular}[c]{@{}c@{}}\cddlptext{}  \\ (uniform) \end{tabular} & \begin{tabular}[c]{@{}c@{}}\cddlptext{}  \\ (Ave. $\P(L|A)$)  \end{tabular}& \begin{tabular}[c]{@{}c@{}}\cddlptext{}  \\ ($\P(L)$) \end{tabular} \\ \hline
DCFR, $\lambda_d:$0.10 & 0.74 $\pm$ 0.12 & 0.083 $\pm$ 0.029 & 561.824 $\pm$ 97.038 & 63.648 $\pm$ 11.645 & 59.971 $\pm$ 10.741 \\ \hline
DCFR, $\lambda_d:$0.25 & 0.76 $\pm$ 0.13 & 0.080 $\pm$ 0.029 & 533.695 $\pm$ 99.986 & 60.697 $\pm$ 12.268 & 57.169 $\pm$ 11.290 \\ \hline
DCFR, $\lambda_d:$0.50 & 0.76 $\pm$ 0.13 & 0.073 $\pm$ 0.030 & 452.759 $\pm$ 92.157 & 51.592 $\pm$ 11.407 & 48.545 $\pm$ 10.507 \\ \hline
DCFR, $\lambda_d:$0.75 & 0.77 $\pm$ 0.12 & 0.068 $\pm$ 0.029 & 398.082 $\pm$ 92.588 & 45.366 $\pm$ 11.173 & 42.859 $\pm$ 10.354 \\ \hline
DCFR, $\lambda_d:$1.00 & 0.76 $\pm$ 0.12 & 0.065 $\pm$ 0.027 & 363.625 $\pm$ 88.308 & 41.365 $\pm$ 10.721 & 39.031 $\pm$ 9.847 \\ \hline
DCFR, $\lambda_d:$1.50 & 0.75 $\pm$ 0.13 & 0.060 $\pm$ 0.027 & 297.195 $\pm$ 91.580 & 33.682 $\pm$ 10.967 & 31.769 $\pm$ 10.187 \\ \hline
DCFR, $\lambda_d:$2.00 & 0.76 $\pm$ 0.10 & 0.059 $\pm$ 0.030 & 278.981 $\pm$ 73.521 & 31.542 $\pm$ 8.868 & 29.892 $\pm$ 8.116 \\ \hline
DCFR, $\lambda_d:$5.00 & 0.76 $\pm$ 0.11 & 0.056 $\pm$ 0.035 & 239.655 $\pm$ 72.522 & 26.636 $\pm$ 8.521 & 25.056 $\pm$ 7.963 \\ \hline
DCFR, $\lambda_d:$10.00 & 0.76 $\pm$ 0.12 & 0.062 $\pm$ 0.035 & 318.831 $\pm$ 223.980 & 35.677 $\pm$ 25.750 & 33.521 $\pm$ 24.104 \\ \hline
DCFR, $\lambda_d:$15.00 & 0.77 $\pm$ 0.12 & 0.081 $\pm$ 0.022 & 574.569 $\pm$ 134.654 & 65.306 $\pm$ 16.085 & 61.141 $\pm$ 14.891 \\ \hline
DCFR, $\lambda_d:$20.00 & 0.73 $\pm$ 0.09 & 0.073 $\pm$ 0.031 & 511.350 $\pm$ 103.351 & 57.783 $\pm$ 11.787 & 54.074 $\pm$ 11.099 \\ \hline
FairBiT, $\lambda_b:$0.00 & 0.81 $\pm$ 0.01 & 0.058 $\pm$ 0.016 & 55.469 $\pm$ 7.486 & 6.168 $\pm$ 0.990 & 5.726 $\pm$ 0.870 \\ \hline
FairBiT, $\lambda_b:$0.00 & 0.81 $\pm$ 0.03 & 0.067 $\pm$ 0.022 & 61.142 $\pm$ 13.468 & 6.867 $\pm$ 1.717 & 6.369 $\pm$ 1.532 \\ \hline
FairBiT, $\lambda_b:$0.00 & 0.81 $\pm$ 0.05 & 0.067 $\pm$ 0.017 & 62.883 $\pm$ 9.146 & 7.109 $\pm$ 1.211 & 6.589 $\pm$ 1.096 \\ \hline
FairBiT, $\lambda_b:$0.00 & 0.82 $\pm$ 0.02 & 0.076 $\pm$ 0.021 & 67.350 $\pm$ 9.651 & 7.573 $\pm$ 1.386 & 7.037 $\pm$ 1.225 \\ \hline
FairBiT, $\lambda_b:$0.02 & 0.80 $\pm$ 0.05 & 0.063 $\pm$ 0.021 & 60.261 $\pm$ 10.298 & 6.749 $\pm$ 1.354 & 6.286 $\pm$ 1.198 \\ \hline
FairBiT, $\lambda_b:$0.06 & 0.81 $\pm$ 0.06 & 0.069 $\pm$ 0.018 & 63.161 $\pm$ 7.356 & 7.079 $\pm$ 1.109 & 6.576 $\pm$ 0.910 \\ \hline
FairBiT, $\lambda_b:$0.22 & 0.82 $\pm$ 0.04 & 0.068 $\pm$ 0.018 & 61.223 $\pm$ 9.513 & 6.841 $\pm$ 1.230 & 6.391 $\pm$ 1.125 \\ \hline
FairBiT, $\lambda_b:$0.77 & 0.81 $\pm$ 0.06 & 0.053 $\pm$ 0.019 & 46.851 $\pm$ 8.784 & 5.195 $\pm$ 1.199 & 4.878 $\pm$ 1.011 \\ \hline
FairBiT, $\lambda_b:$2.78 & 0.83 $\pm$ 0.02 & 0.041 $\pm$ 0.012 & 33.622 $\pm$ 4.609 & 3.655 $\pm$ 0.628 & 3.457 $\pm$ 0.542 \\ \hline
FairBiT, $\lambda_b:$10.00 & 0.80 $\pm$ 0.02 & 0.035 $\pm$ 0.015 & 29.978 $\pm$ 3.358 & 3.239 $\pm$ 0.398 & 3.011 $\pm$ 0.355 \\ \hline
FairLearn (CDP) & 0.60 $\pm$ 0.04 & 0.068 $\pm$ 0.038 & 95.417 $\pm$ 85.756 & 10.887 $\pm$ 9.986 & 10.340 $\pm$ 9.487 \\ \hline
FairLearn (DP) & 0.60 $\pm$ 0.04 & 0.070 $\pm$ 0.040 & 91.254 $\pm$ 89.047 & 10.468 $\pm$ 10.417 & 9.919 $\pm$ 9.871 \\ \hline
Wass. Regularization, $\lambda_w:$0.00 & 0.82 $\pm$ 0.01 & 0.056 $\pm$ 0.015 & 55.296 $\pm$ 10.003 & 6.249 $\pm$ 1.330 & 5.786 $\pm$ 1.141 \\ \hline
Wass. Regularization, $\lambda_w:$0.00 & 0.80 $\pm$ 0.05 & 0.062 $\pm$ 0.022 & 58.566 $\pm$ 12.635 & 6.611 $\pm$ 1.669 & 6.135 $\pm$ 1.453 \\ \hline
Wass. Regularization, $\lambda_w:$0.00 & 0.83 $\pm$ 0.02 & 0.071 $\pm$ 0.023 & 66.117 $\pm$ 12.194 & 7.477 $\pm$ 1.608 & 6.944 $\pm$ 1.432 \\ \hline
Wass. Regularization, $\lambda_w:$0.00 & 0.81 $\pm$ 0.05 & 0.077 $\pm$ 0.028 & 68.589 $\pm$ 17.184 & 7.735 $\pm$ 2.068 & 7.182 $\pm$ 1.882 \\ \hline
Wass. Regularization, $\lambda_w:$0.02 & 0.78 $\pm$ 0.08 & 0.065 $\pm$ 0.016 & 61.582 $\pm$ 11.624 & 6.929 $\pm$ 1.359 & 6.459 $\pm$ 1.254 \\ \hline
Wass. Regularization, $\lambda_w:$0.06 & 0.80 $\pm$ 0.06 & 0.061 $\pm$ 0.022 & 57.394 $\pm$ 9.826 & 6.481 $\pm$ 1.250 & 6.046 $\pm$ 1.073 \\ \hline
Wass. Regularization, $\lambda_w:$0.22 & 0.82 $\pm$ 0.04 & 0.057 $\pm$ 0.022 & 49.628 $\pm$ 11.510 & 5.516 $\pm$ 1.428 & 5.211 $\pm$ 1.281 \\ \hline
Wass. Regularization, $\lambda_w:$0.77 & 0.82 $\pm$ 0.01 & 0.045 $\pm$ 0.016 & 34.830 $\pm$ 4.871 & 3.834 $\pm$ 0.676 & 3.605 $\pm$ 0.609 \\ \hline
Wass. Regularization, $\lambda_w:$2.78 & 0.80 $\pm$ 0.02 & 0.043 $\pm$ 0.012 & 31.387 $\pm$ 3.768 & 3.391 $\pm$ 0.489 & 3.171 $\pm$ 0.434 \\ \hline
Wass. Regularization, $\lambda_w:$10.00 & 0.77 $\pm$ 0.04 & 0.045 $\pm$ 0.012 & 32.547 $\pm$ 3.389 & 3.493 $\pm$ 0.458 & 3.249 $\pm$ 0.395 \\ \hline
Legit Only & 0.55 $\pm$ 0.02 & -0.004 $\pm$ 0.014 & 0.000 $\pm$ 0.000 & 0.000 $\pm$ 0.000 & 0.000 $\pm$ 0.000 \\ \hline
No Regularization & 0.83 $\pm$ 0.03 & 0.080 $\pm$ 0.034 & 39.842 $\pm$ 35.694 & 44.793 $\pm$ 40.119 & 41.785 $\pm$ 37.418 \\ \hline
Fair PreProc., $\lambda_r:$0.00 & 0.88 $\pm$ 0.01 & 0.090 $\pm$ 0.014 & 7.445 $\pm$ 0.560 & 8.352 $\pm$ 0.734 & 7.850 $\pm$ 0.641 \\ \hline
Fair PreProc., $\lambda_r:$0.10 & 0.88 $\pm$ 0.01 & 0.089 $\pm$ 0.014 & 7.433 $\pm$ 0.570 & 8.339 $\pm$ 0.754 & 7.836 $\pm$ 0.658 \\ \hline
Fair PreProc., $\lambda_r:$0.20 & 0.88 $\pm$ 0.01 & 0.089 $\pm$ 0.014 & 7.431 $\pm$ 0.535 & 8.337 $\pm$ 0.740 & 7.834 $\pm$ 0.636 \\ \hline
Fair PreProc., $\lambda_r:$0.30 & 0.88 $\pm$ 0.01 & 0.089 $\pm$ 0.014 & 7.422 $\pm$ 0.478 & 8.327 $\pm$ 0.687 & 7.827 $\pm$ 0.587 \\ \hline
Fair PreProc., $\lambda_r:$0.40 & 0.88 $\pm$ 0.01 & 0.089 $\pm$ 0.014 & 7.439 $\pm$ 0.493 & 8.346 $\pm$ 0.711 & 7.842 $\pm$ 0.603 \\ \hline
Fair PreProc., $\lambda_r:$0.50 & 0.88 $\pm$ 0.01 & 0.088 $\pm$ 0.014 & 7.415 $\pm$ 0.506 & 8.319 $\pm$ 0.744 & 7.816 $\pm$ 0.620 \\ \hline
Fair PreProc., $\lambda_r:$0.60 & 0.88 $\pm$ 0.01 & 0.087 $\pm$ 0.012 & 7.301 $\pm$ 0.681 & 8.201 $\pm$ 0.928 & 7.703 $\pm$ 0.809 \\ \hline
Fair PreProc., $\lambda_r:$0.70 & 0.88 $\pm$ 0.01 & 0.087 $\pm$ 0.013 & 7.276 $\pm$ 0.696 & 8.174 $\pm$ 0.956 & 7.678 $\pm$ 0.834 \\ \hline
Fair PreProc., $\lambda_r:$0.80 & 0.88 $\pm$ 0.01 & 0.087 $\pm$ 0.013 & 7.282 $\pm$ 0.717 & 8.184 $\pm$ 0.984 & 7.687 $\pm$ 0.854 \\ \hline
Fair PreProc., $\lambda_r:$0.90 & 0.88 $\pm$ 0.01 & 0.087 $\pm$ 0.014 & 7.274 $\pm$ 0.834 & 8.173 $\pm$ 1.116 & 7.674 $\pm$ 0.974 \\ \hline
Fair PreProc., $\lambda_r:$1.00 & 0.88 $\pm$ 0.01 & 0.086 $\pm$ 0.014 & 7.257 $\pm$ 0.840 & 8.154 $\pm$ 1.122 & 7.655 $\pm$ 0.985 \\ \hline
\end{tabular}%
}
\caption{Mean and one standard deviation results for all methods-hyperparameter combinations for all methods except \fairlp, for the \texttt{Adult} dataset across 10 separate runs. }
\label{tab:adult_rest}
\end{table*}

\begin{table*}[ht]
\centering
\resizebox{1.0\textwidth}{!}{%
\begin{tabular}{|c||c|c|c|c|c|}
\hline
Method & AUC  &  \begin{tabular}[c]{@{}c@{}}\cddwtext{}\\ (normalized) \end{tabular}  & \begin{tabular}[c]{@{}c@{}}\cddlptext{}  \\ (uniform) \end{tabular} & \begin{tabular}[c]{@{}c@{}}\cddlptext{}  \\ (Ave. $\P(L|A)$)  \end{tabular}& \begin{tabular}[c]{@{}c@{}}\cddlptext{}  \\ ($\P(L)$) \end{tabular} \\ \hline
FairLeap (uniform), $\lambda_w:$0.00 & 0.77 $\pm$ 0.08 & 0.051 $\pm$ 0.018 & 49.145 $\pm$ 13.342 & 5.508 $\pm$ 1.748 & 5.127 $\pm$ 1.584 \\ \hline
FairLeap (uniform), $\lambda_w:$0.00 & 0.80 $\pm$ 0.05 & 0.056 $\pm$ 0.024 & 51.151 $\pm$ 10.382 & 5.711 $\pm$ 1.349 & 5.318 $\pm$ 1.229 \\ \hline
FairLeap (uniform), $\lambda_w:$0.00 & 0.80 $\pm$ 0.03 & 0.044 $\pm$ 0.017 & 35.180 $\pm$ 7.448 & 3.764 $\pm$ 0.963 & 3.546 $\pm$ 0.833 \\ \hline
FairLeap (uniform), $\lambda_w:$0.01 & 0.78 $\pm$ 0.06 & 0.034 $\pm$ 0.015 & 26.346 $\pm$ 5.901 & 2.739 $\pm$ 0.606 & 2.579 $\pm$ 0.607 \\ \hline
FairLeap (uniform), $\lambda_w:$0.05 & 0.71 $\pm$ 0.04 & 0.037 $\pm$ 0.012 & 26.331 $\pm$ 6.101 & 2.865 $\pm$ 0.756 & 2.640 $\pm$ 0.676 \\ \hline
FairLeap (uniform), $\lambda_w:$0.22 & 0.56 $\pm$ 0.12 & 0.014 $\pm$ 0.027 & 11.964 $\pm$ 10.538 & 1.328 $\pm$ 1.195 & 1.202 $\pm$ 1.080 \\ \hline
FairLeap (uniform), $\lambda_w:$1.00 & 0.55 $\pm$ 0.13 & 0.008 $\pm$ 0.027 & 6.054 $\pm$ 9.312 & 0.672 $\pm$ 1.054 & 0.621 $\pm$ 0.980 \\ \hline
FairLeap (uniform), $\lambda_w:$4.64 & 0.55 $\pm$ 0.20 & -0.004 $\pm$ 0.014 & 0.148 $\pm$ 0.112 & 0.015 $\pm$ 0.011 & 0.014 $\pm$ 0.010 \\ \hline
FairLeap (uniform), $\lambda_w:$21.54 & 0.54 $\pm$ 0.21 & -0.004 $\pm$ 0.014 & 0.033 $\pm$ 0.024 & 0.003 $\pm$ 0.002 & 0.003 $\pm$ 0.002 \\ \hline
FairLeap (uniform), $\lambda_w:$100.00 & 0.48 $\pm$ 0.18 & 0.004 $\pm$ 0.012 & 2.580 $\pm$ 5.533 & 0.254 $\pm$ 0.556 & 0.255 $\pm$ 0.560 \\ \hline
FairLeap (P(L)), $\lambda_w:$0.00 & 0.79 $\pm$ 0.04 & 0.055 $\pm$ 0.017 & 53.119 $\pm$ 9.445 & 5.962 $\pm$ 1.330 & 5.537 $\pm$ 1.181 \\ \hline
FairLeap (P(L)), $\lambda_w:$0.00 & 0.80 $\pm$ 0.05 & 0.066 $\pm$ 0.018 & 60.711 $\pm$ 11.942 & 6.823 $\pm$ 1.558 & 6.334 $\pm$ 1.416 \\ \hline
FairLeap (P(L)), $\lambda_w:$0.00 & 0.81 $\pm$ 0.05 & 0.062 $\pm$ 0.018 & 56.668 $\pm$ 9.152 & 6.349 $\pm$ 1.309 & 5.919 $\pm$ 1.150 \\ \hline
FairLeap (P(L)), $\lambda_w:$0.01 & 0.79 $\pm$ 0.07 & 0.049 $\pm$ 0.025 & 41.535 $\pm$ 11.905 & 4.595 $\pm$ 1.403 & 4.307 $\pm$ 1.302 \\ \hline
FairLeap (P(L)), $\lambda_w:$0.05 & 0.79 $\pm$ 0.06 & 0.041 $\pm$ 0.016 & 30.153 $\pm$ 4.635 & 3.219 $\pm$ 0.611 & 3.012 $\pm$ 0.524 \\ \hline
FairLeap (P(L)), $\lambda_w:$0.22 & 0.71 $\pm$ 0.09 & 0.033 $\pm$ 0.015 & 26.212 $\pm$ 7.478 & 2.851 $\pm$ 0.897 & 2.627 $\pm$ 0.809 \\ \hline
FairLeap (P(L)), $\lambda_w:$1.00 & 0.69 $\pm$ 0.05 & 0.040 $\pm$ 0.017 & 27.009 $\pm$ 7.529 & 2.950 $\pm$ 0.852 & 2.717 $\pm$ 0.770 \\ \hline
FairLeap (P(L)), $\lambda_w:$4.64 & 0.59 $\pm$ 0.12 & 0.019 $\pm$ 0.021 & 11.460 $\pm$ 8.641 & 1.246 $\pm$ 0.963 & 1.170 $\pm$ 0.910 \\ \hline
FairLeap (P(L)), $\lambda_w:$21.54 & 0.56 $\pm$ 0.16 & 0.007 $\pm$ 0.029 & 5.971 $\pm$ 12.050 & 0.705 $\pm$ 1.434 & 0.641 $\pm$ 1.297 \\ \hline
FairLeap (P(L)), $\lambda_w:$100.00 & 0.53 $\pm$ 0.15 & -0.001 $\pm$ 0.019 & 1.337 $\pm$ 4.006 & 0.141 $\pm$ 0.425 & 0.128 $\pm$ 0.387 \\ \hline
Fairleap (Ave. P(L$|$A)), $\lambda_w:$0.00 & 0.80 $\pm$ 0.05 & 0.055 $\pm$ 0.019 & 53.008 $\pm$ 13.259 & 5.913 $\pm$ 1.707 & 5.490 $\pm$ 1.540 \\ \hline
Fairleap (Ave. P(L$|$A)), $\lambda_w:$0.00 & 0.82 $\pm$ 0.01 & 0.064 $\pm$ 0.020 & 59.219 $\pm$ 9.992 & 6.634 $\pm$ 1.402 & 6.167 $\pm$ 1.258 \\ \hline
Fairleap (Ave. P(L$|$A)), $\lambda_w:$0.00 & 0.83 $\pm$ 0.02 & 0.065 $\pm$ 0.019 & 59.554 $\pm$ 8.023 & 6.659 $\pm$ 1.163 & 6.205 $\pm$ 1.017 \\ \hline
Fairleap (Ave. P(L$|$A)), $\lambda_w:$0.01 & 0.79 $\pm$ 0.05 & 0.047 $\pm$ 0.023 & 40.801 $\pm$ 9.111 & 4.484 $\pm$ 1.107 & 4.221 $\pm$ 1.032 \\ \hline
Fairleap (Ave. P(L$|$A)), $\lambda_w:$0.05 & 0.77 $\pm$ 0.07 & 0.036 $\pm$ 0.014 & 27.195 $\pm$ 6.888 & 2.883 $\pm$ 0.761 & 2.720 $\pm$ 0.719 \\ \hline
Fairleap (Ave. P(L$|$A)), $\lambda_w:$0.22 & 0.76 $\pm$ 0.04 & 0.038 $\pm$ 0.015 & 27.067 $\pm$ 5.693 & 2.929 $\pm$ 0.767 & 2.715 $\pm$ 0.677 \\ \hline
Fairleap (Ave. P(L$|$A)), $\lambda_w:$1.00 & 0.64 $\pm$ 0.11 & 0.036 $\pm$ 0.017 & 23.157 $\pm$ 5.814 & 2.522 $\pm$ 0.742 & 2.322 $\pm$ 0.640 \\ \hline
Fairleap (Ave. P(L$|$A)), $\lambda_w:$4.64 & 0.63 $\pm$ 0.10 & 0.018 $\pm$ 0.022 & 11.498 $\pm$ 10.472 & 1.261 $\pm$ 1.228 & 1.165 $\pm$ 1.105 \\ \hline
Fairleap (Ave. P(L$|$A)), $\lambda_w:$21.54 & 0.57 $\pm$ 0.16 & 0.003 $\pm$ 0.029 & 3.356 $\pm$ 9.768 & 0.402 $\pm$ 1.184 & 0.374 $\pm$ 1.103 \\ \hline
Fairleap (Ave. P(L$|$A)), $\lambda_w:$100 & 0.58 $\pm$ 0.19 & 0.004 $\pm$ 0.025 & 5.057 $\pm$ 11.993 & 0.569 $\pm$ 1.372 & 0.529 $\pm$ 1.283 \\ \hline
\end{tabular}%
}
\caption{Mean and one standard deviation results for all methods-hyperparameter combinations for the 3 variants of \fairlp, for the \texttt{Adult} dataset across 10 separate runs. }
\label{tab:adult_fairlp}
\end{table*}


\begin{table*}[ht]
\centering
\resizebox{1.0\textwidth}{!}{%
\begin{tabular}{|c||c|c|c|c|c|}
\hline
Method & AUC  &  \begin{tabular}[c]{@{}c@{}}\cddwtext{}\\ (normalized) \end{tabular}  & \begin{tabular}[c]{@{}c@{}}\cddlptext{}  \\ (uniform) \end{tabular} & \begin{tabular}[c]{@{}c@{}}\cddlptext{}  \\ (Ave. $\P(L|A)$)  \end{tabular}& \begin{tabular}[c]{@{}c@{}}\cddlptext{}  \\ ($\P(L)$) \end{tabular} \\ \hline
DCFR, $\lambda_d:$0.10 & 0.78 $\pm$ 0.02 & 0.128 $\pm$ 0.019 & 35.404 $\pm$ 5.013 & 6.070 $\pm$ 0.892 & 6.061 $\pm$ 0.892 \\ \hline
DCFR, $\lambda_d:$0.25 & 0.78 $\pm$ 0.01 & 0.119 $\pm$ 0.022 & 32.381 $\pm$ 6.650 & 5.504 $\pm$ 1.211 & 5.498 $\pm$ 1.216 \\ \hline
DCFR, $\lambda_d:$0.50 & 0.76 $\pm$ 0.02 & 0.117 $\pm$ 0.017 & 31.468 $\pm$ 5.937 & 5.411 $\pm$ 1.159 & 5.404 $\pm$ 1.146 \\ \hline
DCFR, $\lambda_d:$0.75 & 0.76 $\pm$ 0.02 & 0.105 $\pm$ 0.016 & 27.762 $\pm$ 4.501 & 4.827 $\pm$ 0.745 & 4.822 $\pm$ 0.751 \\ \hline
DCFR, $\lambda_d:$1.00 & 0.76 $\pm$ 0.02 & 0.099 $\pm$ 0.011 & 26.282 $\pm$ 2.846 & 4.403 $\pm$ 0.524 & 4.399 $\pm$ 0.526 \\ \hline
DCFR, $\lambda_d:$1.50 & 0.76 $\pm$ 0.02 & 0.096 $\pm$ 0.014 & 25.750 $\pm$ 5.251 & 4.189 $\pm$ 0.960 & 4.184 $\pm$ 0.964 \\ \hline
DCFR, $\lambda_d:$2.00 & 0.76 $\pm$ 0.02 & 0.089 $\pm$ 0.018 & 21.700 $\pm$ 6.444 & 3.530 $\pm$ 1.041 & 3.527 $\pm$ 1.043 \\ \hline
DCFR, $\lambda_d:$5.00 & 0.77 $\pm$ 0.02 & 0.073 $\pm$ 0.011 & 19.651 $\pm$ 3.494 & 3.117 $\pm$ 0.667 & 3.113 $\pm$ 0.664 \\ \hline
DCFR, $\lambda_d:$10.00 & 0.79 $\pm$ 0.02 & 0.063 $\pm$ 0.009 & 16.741 $\pm$ 4.430 & 2.664 $\pm$ 0.824 & 2.663 $\pm$ 0.825 \\ \hline
DCFR, $\lambda_d:$15.00 & 0.80 $\pm$ 0.02 & 0.055 $\pm$ 0.008 & 15.813 $\pm$ 2.478 & 2.444 $\pm$ 0.415 & 2.442 $\pm$ 0.416 \\ \hline
DCFR, $\lambda_d:$20.00 & 0.80 $\pm$ 0.02 & 0.051 $\pm$ 0.007 & 15.213 $\pm$ 4.191 & 2.436 $\pm$ 0.709 & 2.434 $\pm$ 0.708 \\ \hline
FairBiT, $\lambda_b:$0.00 & 0.85 $\pm$ 0.02 & 0.081 $\pm$ 0.012 & 36.468 $\pm$ 5.642 & 6.232 $\pm$ 1.079 & 6.238 $\pm$ 1.080 \\ \hline
FairBiT, $\lambda_b:$0.00 & 0.85 $\pm$ 0.02 & 0.086 $\pm$ 0.011 & 38.809 $\pm$ 3.924 & 6.630 $\pm$ 0.697 & 6.635 $\pm$ 0.699 \\ \hline
FairBiT, $\lambda_b:$0.00 & 0.85 $\pm$ 0.02 & 0.086 $\pm$ 0.011 & 38.823 $\pm$ 3.954 & 6.632 $\pm$ 0.701 & 6.638 $\pm$ 0.703 \\ \hline
FairBiT, $\lambda_b:$0.00 & 0.85 $\pm$ 0.02 & 0.085 $\pm$ 0.011 & 38.641 $\pm$ 3.955 & 6.600 $\pm$ 0.701 & 6.606 $\pm$ 0.704 \\ \hline
FairBiT, $\lambda_b:$0.02 & 0.85 $\pm$ 0.02 & 0.084 $\pm$ 0.011 & 38.058 $\pm$ 3.941 & 6.498 $\pm$ 0.698 & 6.503 $\pm$ 0.700 \\ \hline
FairBiT, $\lambda_b:$0.06 & 0.85 $\pm$ 0.02 & 0.079 $\pm$ 0.010 & 36.089 $\pm$ 3.867 & 6.152 $\pm$ 0.682 & 6.158 $\pm$ 0.685 \\ \hline
FairBiT, $\lambda_b:$0.22 & 0.85 $\pm$ 0.02 & 0.067 $\pm$ 0.008 & 30.326 $\pm$ 3.605 & 5.141 $\pm$ 0.638 & 5.146 $\pm$ 0.640 \\ \hline
FairBiT, $\lambda_b:$0.77 & 0.85 $\pm$ 0.02 & 0.044 $\pm$ 0.006 & 19.536 $\pm$ 3.026 & 3.262 $\pm$ 0.550 & 3.267 $\pm$ 0.551 \\ \hline
FairBiT, $\lambda_b:$2.78 & 0.84 $\pm$ 0.02 & 0.013 $\pm$ 0.002 & 8.458 $\pm$ 1.443 & 1.421 $\pm$ 0.286 & 1.423 $\pm$ 0.286 \\ \hline
FairBiT, $\lambda_b:$10.00 & 0.81 $\pm$ 0.02 & 0.002 $\pm$ 0.000 & 2.220 $\pm$ 0.420 & 0.363 $\pm$ 0.086 & 0.364 $\pm$ 0.086 \\ \hline
FairLearn (CDP) & 0.65 $\pm$ 0.02 & 0.144 $\pm$ 0.030 & 27.817 $\pm$ 9.400 & 4.239 $\pm$ 1.449 & 4.246 $\pm$ 1.458 \\ \hline
FairLearn (DP) & 0.65 $\pm$ 0.02 & 0.144 $\pm$ 0.030 & 27.817 $\pm$ 9.400 & 4.239 $\pm$ 1.449 & 4.246 $\pm$ 1.458 \\ \hline
Wass. Regularization, $\lambda_w:$0.00 & 0.85 $\pm$ 0.02 & 0.081 $\pm$ 0.012 & 36.200 $\pm$ 6.182 & 6.184 $\pm$ 1.170 & 6.191 $\pm$ 1.172 \\ \hline
Wass. Regularization, $\lambda_w:$0.00 & 0.85 $\pm$ 0.02 & 0.086 $\pm$ 0.011 & 38.842 $\pm$ 3.981 & 6.635 $\pm$ 0.702 & 6.641 $\pm$ 0.704 \\ \hline
Wass. Regularization, $\lambda_w:$0.00 & 0.85 $\pm$ 0.02 & 0.086 $\pm$ 0.011 & 38.772 $\pm$ 3.958 & 6.623 $\pm$ 0.701 & 6.628 $\pm$ 0.704 \\ \hline
Wass. Regularization, $\lambda_w:$0.00 & 0.85 $\pm$ 0.02 & 0.085 $\pm$ 0.011 & 38.394 $\pm$ 3.970 & 6.556 $\pm$ 0.705 & 6.562 $\pm$ 0.707 \\ \hline
Wass. Regularization, $\lambda_w:$0.02 & 0.85 $\pm$ 0.02 & 0.082 $\pm$ 0.011 & 37.124 $\pm$ 4.007 & 6.329 $\pm$ 0.707 & 6.335 $\pm$ 0.709 \\ \hline
Wass. Regularization, $\lambda_w:$0.06 & 0.85 $\pm$ 0.02 & 0.073 $\pm$ 0.010 & 32.556 $\pm$ 4.114 & 5.514 $\pm$ 0.712 & 5.520 $\pm$ 0.713 \\ \hline
Wass. Regularization, $\lambda_w:$0.22 & 0.84 $\pm$ 0.02 & 0.050 $\pm$ 0.007 & 17.982 $\pm$ 3.748 & 2.882 $\pm$ 0.673 & 2.887 $\pm$ 0.674 \\ \hline
Wass. Regularization, $\lambda_w:$0.77 & 0.81 $\pm$ 0.02 & 0.040 $\pm$ 0.005 & 14.006 $\pm$ 2.199 & 2.139 $\pm$ 0.447 & 2.142 $\pm$ 0.449 \\ \hline
Wass. Regularization, $\lambda_w:$2.78 & 0.80 $\pm$ 0.02 & 0.038 $\pm$ 0.005 & 14.166 $\pm$ 1.957 & 2.147 $\pm$ 0.337 & 2.149 $\pm$ 0.339 \\ \hline
Wass. Regularization, $\lambda_w:$10.00 & 0.79 $\pm$ 0.02 & 0.036 $\pm$ 0.006 & 13.862 $\pm$ 2.145 & 2.074 $\pm$ 0.349 & 2.076 $\pm$ 0.350 \\ \hline
Legit Only & 0.60 $\pm$ 0.02 & 0.001 $\pm$ 0.000 & 0.000 $\pm$ 0.000 & 0.000 $\pm$ 0.000 & 0.000 $\pm$ 0.000 \\ \hline
No Regularization & 0.85 $\pm$ 0.02 & 0.082 $\pm$ 0.009 & 4.128 $\pm$ 0.391 & 6.343 $\pm$ 0.664 & 6.349 $\pm$ 0.666 \\ \hline
Fair PreProc., $\lambda_r:$0.00 & 0.85 $\pm$ 0.02 & 0.084 $\pm$ 0.010 & 4.292 $\pm$ 0.375 & 6.597 $\pm$ 0.639 & 6.603 $\pm$ 0.641 \\ \hline
Fair PreProc., $\lambda_r:$0.10 & 0.85 $\pm$ 0.02 & 0.084 $\pm$ 0.010 & 4.291 $\pm$ 0.375 & 6.586 $\pm$ 0.646 & 6.592 $\pm$ 0.649 \\ \hline
Fair PreProc., $\lambda_r:$0.20 & 0.85 $\pm$ 0.02 & 0.084 $\pm$ 0.010 & 4.293 $\pm$ 0.395 & 6.599 $\pm$ 0.702 & 6.605 $\pm$ 0.704 \\ \hline
Fair PreProc., $\lambda_r:$0.30 & 0.85 $\pm$ 0.02 & 0.083 $\pm$ 0.009 & 4.223 $\pm$ 0.326 & 6.483 $\pm$ 0.626 & 6.489 $\pm$ 0.628 \\ \hline
Fair PreProc., $\lambda_r:$0.40 & 0.85 $\pm$ 0.02 & 0.083 $\pm$ 0.009 & 4.225 $\pm$ 0.337 & 6.488 $\pm$ 0.638 & 6.494 $\pm$ 0.640 \\ \hline
Fair PreProc., $\lambda_r:$0.50 & 0.85 $\pm$ 0.02 & 0.082 $\pm$ 0.009 & 4.206 $\pm$ 0.341 & 6.459 $\pm$ 0.658 & 6.465 $\pm$ 0.660 \\ \hline
Fair PreProc., $\lambda_r:$0.60 & 0.85 $\pm$ 0.02 & 0.084 $\pm$ 0.009 & 4.311 $\pm$ 0.344 & 6.623 $\pm$ 0.624 & 6.629 $\pm$ 0.624 \\ \hline
Fair PreProc., $\lambda_r:$0.70 & 0.85 $\pm$ 0.02 & 0.085 $\pm$ 0.009 & 4.351 $\pm$ 0.366 & 6.701 $\pm$ 0.715 & 6.707 $\pm$ 0.714 \\ \hline
Fair PreProc., $\lambda_r:$0.80 & 0.85 $\pm$ 0.02 & 0.084 $\pm$ 0.009 & 4.301 $\pm$ 0.327 & 6.624 $\pm$ 0.696 & 6.630 $\pm$ 0.695 \\ \hline
Fair PreProc., $\lambda_r:$0.90 & 0.85 $\pm$ 0.02 & 0.084 $\pm$ 0.009 & 4.278 $\pm$ 0.353 & 6.586 $\pm$ 0.750 & 6.592 $\pm$ 0.748 \\ \hline
Fair PreProc., $\lambda_r:$1.00 & 0.85 $\pm$ 0.02 & 0.084 $\pm$ 0.009 & 4.285 $\pm$ 0.342 & 6.597 $\pm$ 0.732 & 6.603 $\pm$ 0.731 \\ \hline
\end{tabular}%
}
\caption{Mean and one standard deviation results for all methods-hyperparameter combinations for all methods except \fairlp, for the \texttt{Drug} dataset across 10 separate runs. }
\label{tab:drug_rest}
\end{table*}

\begin{table*}[ht]
\centering
\resizebox{1.0\textwidth}{!}{%
\begin{tabular}{|c||c|c|c|c|c|}
\hline
Method & AUC  &  \begin{tabular}[c]{@{}c@{}}\cddwtext{}\\ (normalized) \end{tabular}  & \begin{tabular}[c]{@{}c@{}}\cddlptext{}  \\ (uniform) \end{tabular} & \begin{tabular}[c]{@{}c@{}}\cddlptext{}  \\ (Ave. $\P(L|A)$)  \end{tabular}& \begin{tabular}[c]{@{}c@{}}\cddlptext{}  \\ ($\P(L)$) \end{tabular} \\ \hline
FairLeap (uniform), $\lambda_w:$0.00 & 0.85 $\pm$ 0.02 & 0.071 $\pm$ 0.009 & 31.246 $\pm$ 3.632 & 5.289 $\pm$ 0.633 & 5.295 $\pm$ 0.635 \\ \hline
FairLeap (uniform), $\lambda_w:$0.00 & 0.84 $\pm$ 0.02 & 0.049 $\pm$ 0.006 & 16.752 $\pm$ 3.309 & 2.693 $\pm$ 0.615 & 2.698 $\pm$ 0.616 \\ \hline
FairLeap (uniform), $\lambda_w:$0.00 & 0.82 $\pm$ 0.02 & 0.032 $\pm$ 0.005 & 10.529 $\pm$ 2.412 & 1.667 $\pm$ 0.448 & 1.669 $\pm$ 0.449 \\ \hline
FairLeap (uniform), $\lambda_w:$0.01 & 0.68 $\pm$ 0.04 & 0.003 $\pm$ 0.003 & 1.956 $\pm$ 1.463 & 0.353 $\pm$ 0.278 & 0.353 $\pm$ 0.277 \\ \hline
FairLeap (uniform), $\lambda_w:$0.05 & 0.61 $\pm$ 0.03 & 0.001 $\pm$ 0.000 & 0.117 $\pm$ 0.050 & 0.021 $\pm$ 0.010 & 0.021 $\pm$ 0.010 \\ \hline
FairLeap (uniform), $\lambda_w:$0.22 & 0.60 $\pm$ 0.04 & 0.001 $\pm$ 0.000 & 0.136 $\pm$ 0.088 & 0.026 $\pm$ 0.017 & 0.026 $\pm$ 0.017 \\ \hline
FairLeap (uniform), $\lambda_w:$1.00 & 0.62 $\pm$ 0.03 & 0.001 $\pm$ 0.000 & 0.095 $\pm$ 0.045 & 0.018 $\pm$ 0.009 & 0.018 $\pm$ 0.009 \\ \hline
FairLeap (uniform), $\lambda_w:$4.64 & 0.60 $\pm$ 0.02 & 0.001 $\pm$ 0.000 & 0.051 $\pm$ 0.017 & 0.010 $\pm$ 0.004 & 0.010 $\pm$ 0.004 \\ \hline
FairLeap (uniform), $\lambda_w:$21.54 & 0.60 $\pm$ 0.05 & 0.001 $\pm$ 0.000 & 0.059 $\pm$ 0.023 & 0.011 $\pm$ 0.004 & 0.011 $\pm$ 0.004 \\ \hline
FairLeap (uniform), $\lambda_w:$100.00 & 0.61 $\pm$ 0.04 & 0.000 $\pm$ 0.000 & 0.037 $\pm$ 0.013 & 0.007 $\pm$ 0.002 & 0.007 $\pm$ 0.002 \\ \hline
FairLeap (P(L)), $\lambda_w:$0.00 & 0.85 $\pm$ 0.02 & 0.080 $\pm$ 0.009 & 36.136 $\pm$ 3.628 & 6.157 $\pm$ 0.651 & 6.163 $\pm$ 0.655 \\ \hline
FairLeap (P(L)), $\lambda_w:$0.00 & 0.85 $\pm$ 0.02 & 0.076 $\pm$ 0.010 & 33.516 $\pm$ 4.020 & 5.689 $\pm$ 0.704 & 5.695 $\pm$ 0.706 \\ \hline
FairLeap (P(L)), $\lambda_w:$0.00 & 0.84 $\pm$ 0.02 & 0.052 $\pm$ 0.007 & 18.332 $\pm$ 3.602 & 2.962 $\pm$ 0.651 & 2.967 $\pm$ 0.652 \\ \hline
FairLeap (P(L)), $\lambda_w:$0.01 & 0.83 $\pm$ 0.02 & 0.037 $\pm$ 0.005 & 12.080 $\pm$ 2.786 & 1.887 $\pm$ 0.506 & 1.890 $\pm$ 0.508 \\ \hline
FairLeap (P(L)), $\lambda_w:$0.05 & 0.77 $\pm$ 0.03 & 0.014 $\pm$ 0.005 & 6.028 $\pm$ 1.694 & 0.977 $\pm$ 0.265 & 0.979 $\pm$ 0.266 \\ \hline
FairLeap (P(L)), $\lambda_w:$0.22 & 0.62 $\pm$ 0.03 & 0.001 $\pm$ 0.000 & 0.340 $\pm$ 0.396 & 0.059 $\pm$ 0.072 & 0.059 $\pm$ 0.072 \\ \hline
FairLeap (P(L)), $\lambda_w:$1.00 & 0.60 $\pm$ 0.03 & 0.001 $\pm$ 0.000 & 0.097 $\pm$ 0.042 & 0.018 $\pm$ 0.008 & 0.018 $\pm$ 0.008 \\ \hline
FairLeap (P(L)), $\lambda_w:$4.64 & 0.60 $\pm$ 0.04 & 0.001 $\pm$ 0.000 & 0.086 $\pm$ 0.043 & 0.016 $\pm$ 0.008 & 0.016 $\pm$ 0.008 \\ \hline
FairLeap (P(L)), $\lambda_w:$21.54 & 0.60 $\pm$ 0.04 & 0.001 $\pm$ 0.000 & 0.072 $\pm$ 0.047 & 0.013 $\pm$ 0.009 & 0.013 $\pm$ 0.009 \\ \hline
FairLeap (P(L)), $\lambda_w:$100.00 & 0.59 $\pm$ 0.04 & 0.001 $\pm$ 0.000 & 0.053 $\pm$ 0.019 & 0.010 $\pm$ 0.004 & 0.010 $\pm$ 0.004 \\ \hline
Fairleap (Ave. P(L$|$A)), $\lambda_w:$0.00 & 0.85 $\pm$ 0.02 & 0.080 $\pm$ 0.009 & 36.337 $\pm$ 3.452 & 6.198 $\pm$ 0.652 & 6.204 $\pm$ 0.655 \\ \hline
Fairleap (Ave. P(L$|$A)), $\lambda_w:$0.00 & 0.85 $\pm$ 0.02 & 0.076 $\pm$ 0.010 & 33.548 $\pm$ 4.008 & 5.695 $\pm$ 0.699 & 5.701 $\pm$ 0.701 \\ \hline
Fairleap (Ave. P(L$|$A)), $\lambda_w:$0.00 & 0.84 $\pm$ 0.02 & 0.051 $\pm$ 0.007 & 18.292 $\pm$ 3.542 & 2.956 $\pm$ 0.639 & 2.961 $\pm$ 0.641 \\ \hline
Fairleap (Ave. P(L$|$A)), $\lambda_w:$0.01 & 0.83 $\pm$ 0.02 & 0.037 $\pm$ 0.005 & 12.087 $\pm$ 2.736 & 1.886 $\pm$ 0.501 & 1.889 $\pm$ 0.503 \\ \hline
Fairleap (Ave. P(L$|$A)), $\lambda_w:$0.05 & 0.77 $\pm$ 0.03 & 0.014 $\pm$ 0.005 & 6.013 $\pm$ 1.764 & 0.980 $\pm$ 0.279 & 0.982 $\pm$ 0.280 \\ \hline
Fairleap (Ave. P(L$|$A)), $\lambda_w:$0.22 & 0.62 $\pm$ 0.05 & 0.001 $\pm$ 0.001 & 0.605 $\pm$ 1.024 & 0.102 $\pm$ 0.161 & 0.102 $\pm$ 0.162 \\ \hline
Fairleap (Ave. P(L$|$A)), $\lambda_w:$1.00 & 0.61 $\pm$ 0.04 & 0.001 $\pm$ 0.000 & 0.242 $\pm$ 0.364 & 0.042 $\pm$ 0.061 & 0.042 $\pm$ 0.061 \\ \hline
Fairleap (Ave. P(L$|$A)), $\lambda_w:$4.64 & 0.60 $\pm$ 0.04 & 0.001 $\pm$ 0.000 & 0.081 $\pm$ 0.044 & 0.015 $\pm$ 0.009 & 0.015 $\pm$ 0.009 \\ \hline
Fairleap (Ave. P(L$|$A)), $\lambda_w:$21.54 & 0.59 $\pm$ 0.03 & 0.001 $\pm$ 0.000 & 0.063 $\pm$ 0.036 & 0.011 $\pm$ 0.007 & 0.011 $\pm$ 0.007 \\ \hline
Fairleap (Ave. P(L$|$A)), $\lambda_w:$100 & 0.59 $\pm$ 0.03 & 0.001 $\pm$ 0.000 & 0.060 $\pm$ 0.028 & 0.011 $\pm$ 0.006 & 0.011 $\pm$ 0.006 \\ \hline
\end{tabular}%
}
\caption{Mean and one standard deviation results for all methods-hyperparameter combinations for the 3 variants of \fairlp, for the \texttt{Drug} dataset across 10 separate runs. }
\label{tab:drug_fairlp}
\end{table*}


\begin{table*}[ht]
\centering
\resizebox{1.0\textwidth}{!}{%
\begin{tabular}{|c||c|c|c|c|c|}
\hline
Method & AUC  &  \begin{tabular}[c]{@{}c@{}}\cddwtext{}\\ (normalized) \end{tabular}  & \begin{tabular}[c]{@{}c@{}}\cddlptext{}  \\ (uniform) \end{tabular} & \begin{tabular}[c]{@{}c@{}}\cddlptext{}  \\ (Ave. $\P(L|A)$)  \end{tabular}& \begin{tabular}[c]{@{}c@{}}\cddlptext{}  \\ ($\P(L)$) \end{tabular} \\ \hline
DCFR, $\lambda_d:$0.10 & 0.75 $\pm$ 0.03 & 0.075 $\pm$ 0.010 & 64.365 $\pm$ 6.454 & 6.989 $\pm$ 0.807 & 6.962 $\pm$ 0.793 \\ \hline
DCFR, $\lambda_d:$0.25 & 0.75 $\pm$ 0.03 & 0.073 $\pm$ 0.010 & 58.616 $\pm$ 6.114 & 6.407 $\pm$ 0.734 & 6.382 $\pm$ 0.718 \\ \hline
DCFR, $\lambda_d:$0.50 & 0.75 $\pm$ 0.03 & 0.069 $\pm$ 0.009 & 53.749 $\pm$ 9.718 & 5.871 $\pm$ 1.030 & 5.850 $\pm$ 1.021 \\ \hline
DCFR, $\lambda_d:$0.75 & 0.75 $\pm$ 0.03 & 0.067 $\pm$ 0.008 & 50.864 $\pm$ 11.320 & 5.547 $\pm$ 1.268 & 5.522 $\pm$ 1.266 \\ \hline
DCFR, $\lambda_d:$1.00 & 0.74 $\pm$ 0.03 & 0.063 $\pm$ 0.007 & 48.152 $\pm$ 6.224 & 5.273 $\pm$ 0.666 & 5.250 $\pm$ 0.660 \\ \hline
DCFR, $\lambda_d:$1.50 & 0.74 $\pm$ 0.03 & 0.063 $\pm$ 0.007 & 46.919 $\pm$ 7.567 & 5.159 $\pm$ 0.802 & 5.141 $\pm$ 0.802 \\ \hline
DCFR, $\lambda_d:$2.00 & 0.74 $\pm$ 0.03 & 0.065 $\pm$ 0.012 & 52.374 $\pm$ 14.938 & 5.676 $\pm$ 1.728 & 5.655 $\pm$ 1.720 \\ \hline
DCFR, $\lambda_d:$5.00 & 0.73 $\pm$ 0.02 & 0.059 $\pm$ 0.009 & 47.078 $\pm$ 12.969 & 5.027 $\pm$ 1.437 & 5.018 $\pm$ 1.433 \\ \hline
DCFR, $\lambda_d:$10.00 & 0.73 $\pm$ 0.02 & 0.064 $\pm$ 0.014 & 54.736 $\pm$ 19.358 & 5.863 $\pm$ 2.115 & 5.855 $\pm$ 2.112 \\ \hline
DCFR, $\lambda_d:$15.00 & 0.73 $\pm$ 0.03 & 0.063 $\pm$ 0.007 & 55.557 $\pm$ 12.458 & 6.034 $\pm$ 1.416 & 6.025 $\pm$ 1.411 \\ \hline
DCFR, $\lambda_d:$20.00 & 0.73 $\pm$ 0.03 & 0.062 $\pm$ 0.016 & 54.114 $\pm$ 18.042 & 5.838 $\pm$ 2.050 & 5.818 $\pm$ 2.036 \\ \hline
FairBiT, $\lambda_b:$0.00 & 0.76 $\pm$ 0.02 & 0.046 $\pm$ 0.004 & 136.669 $\pm$ 24.976 & 14.949 $\pm$ 2.789 & 14.911 $\pm$ 2.781 \\ \hline
FairBiT, $\lambda_b:$0.00 & 0.76 $\pm$ 0.01 & 0.046 $\pm$ 0.003 & 138.824 $\pm$ 18.970 & 15.171 $\pm$ 2.138 & 15.133 $\pm$ 2.129 \\ \hline
FairBiT, $\lambda_b:$0.00 & 0.76 $\pm$ 0.02 & 0.046 $\pm$ 0.003 & 138.424 $\pm$ 18.793 & 15.126 $\pm$ 2.110 & 15.088 $\pm$ 2.102 \\ \hline
FairBiT, $\lambda_b:$0.00 & 0.76 $\pm$ 0.02 & 0.046 $\pm$ 0.003 & 136.560 $\pm$ 19.170 & 14.919 $\pm$ 2.125 & 14.881 $\pm$ 2.115 \\ \hline
FairBiT, $\lambda_b:$0.02 & 0.76 $\pm$ 0.02 & 0.045 $\pm$ 0.003 & 128.668 $\pm$ 18.087 & 14.048 $\pm$ 1.988 & 14.011 $\pm$ 1.977 \\ \hline
FairBiT, $\lambda_b:$0.06 & 0.76 $\pm$ 0.02 & 0.043 $\pm$ 0.002 & 107.480 $\pm$ 14.507 & 11.709 $\pm$ 1.587 & 11.676 $\pm$ 1.576 \\ \hline
FairBiT, $\lambda_b:$0.22 & 0.76 $\pm$ 0.02 & 0.036 $\pm$ 0.002 & 65.246 $\pm$ 9.924 & 7.015 $\pm$ 1.088 & 6.991 $\pm$ 1.074 \\ \hline
FairBiT, $\lambda_b:$0.77 & 0.77 $\pm$ 0.02 & 0.018 $\pm$ 0.002 & 33.994 $\pm$ 6.489 & 3.602 $\pm$ 0.706 & 3.596 $\pm$ 0.705 \\ \hline
FairBiT, $\lambda_b:$2.78 & 0.79 $\pm$ 0.01 & 0.004 $\pm$ 0.000 & 49.416 $\pm$ 11.565 & 5.359 $\pm$ 1.245 & 5.357 $\pm$ 1.244 \\ \hline
FairBiT, $\lambda_b:$10.00 & 0.80 $\pm$ 0.01 & 0.001 $\pm$ 0.000 & 50.967 $\pm$ 8.455 & 5.547 $\pm$ 0.934 & 5.543 $\pm$ 0.932 \\ \hline
FairLearn (CDP) & 0.76 $\pm$ 0.03 & 0.049 $\pm$ 0.010 & 143.814 $\pm$ 75.740 & 15.553 $\pm$ 8.365 & 15.520 $\pm$ 8.354 \\ \hline
FairLearn (DP) & 0.76 $\pm$ 0.03 & 0.049 $\pm$ 0.010 & 143.814 $\pm$ 75.740 & 15.553 $\pm$ 8.365 & 15.520 $\pm$ 8.354 \\ \hline
Wass. Regularization, $\lambda_w:$0.00 & 0.76 $\pm$ 0.02 & 0.047 $\pm$ 0.002 & 140.926 $\pm$ 13.829 & 15.402 $\pm$ 1.521 & 15.364 $\pm$ 1.510 \\ \hline
Wass. Regularization, $\lambda_w:$0.00 & 0.76 $\pm$ 0.02 & 0.046 $\pm$ 0.002 & 139.740 $\pm$ 12.253 & 15.272 $\pm$ 1.337 & 15.234 $\pm$ 1.330 \\ \hline
Wass. Regularization, $\lambda_w:$0.00 & 0.76 $\pm$ 0.02 & 0.046 $\pm$ 0.002 & 137.978 $\pm$ 12.025 & 15.082 $\pm$ 1.406 & 15.044 $\pm$ 1.407 \\ \hline
Wass. Regularization, $\lambda_w:$0.00 & 0.76 $\pm$ 0.02 & 0.045 $\pm$ 0.002 & 126.986 $\pm$ 12.668 & 13.864 $\pm$ 1.494 & 13.828 $\pm$ 1.496 \\ \hline
Wass. Regularization, $\lambda_w:$0.02 & 0.76 $\pm$ 0.02 & 0.043 $\pm$ 0.001 & 99.530 $\pm$ 8.605 & 10.823 $\pm$ 1.009 & 10.791 $\pm$ 1.005 \\ \hline
Wass. Regularization, $\lambda_w:$0.06 & 0.76 $\pm$ 0.02 & 0.039 $\pm$ 0.001 & 62.866 $\pm$ 7.935 & 6.689 $\pm$ 0.890 & 6.668 $\pm$ 0.879 \\ \hline
Wass. Regularization, $\lambda_w:$0.22 & 0.76 $\pm$ 0.01 & 0.035 $\pm$ 0.001 & 64.476 $\pm$ 12.326 & 6.867 $\pm$ 1.375 & 6.858 $\pm$ 1.372 \\ \hline
Wass. Regularization, $\lambda_w:$0.77 & 0.77 $\pm$ 0.01 & 0.023 $\pm$ 0.002 & 63.646 $\pm$ 12.699 & 6.852 $\pm$ 1.337 & 6.850 $\pm$ 1.339 \\ \hline
Wass. Regularization, $\lambda_w:$2.78 & 0.81 $\pm$ 0.01 & 0.001 $\pm$ 0.000 & 41.850 $\pm$ 5.725 & 4.549 $\pm$ 0.630 & 4.546 $\pm$ 0.628 \\ \hline
Wass. Regularization, $\lambda_w:$10.00 & 0.86 $\pm$ 0.02 & -0.000 $\pm$ 0.000 & 0.494 $\pm$ 0.479 & 0.054 $\pm$ 0.053 & 0.054 $\pm$ 0.053 \\ \hline
Legit Only & 0.81 $\pm$ 0.01 & 0.003 $\pm$ 0.001 & 0.000 $\pm$ 0.000 & 0.000 $\pm$ 0.000 & 0.000 $\pm$ 0.000 \\ \hline
No Regularization & 0.76 $\pm$ 0.02 & 0.047 $\pm$ 0.002 & 14.133 $\pm$ 1.406 & 15.444 $\pm$ 1.500 & 15.406 $\pm$ 1.491 \\ \hline
Fair PreProc., $\lambda_r:$0.00 & 0.76 $\pm$ 0.02 & 0.048 $\pm$ 0.003 & 14.364 $\pm$ 3.107 & 15.683 $\pm$ 3.610 & 15.644 $\pm$ 3.613 \\ \hline
Fair PreProc., $\lambda_r:$0.10 & 0.76 $\pm$ 0.02 & 0.048 $\pm$ 0.003 & 14.243 $\pm$ 3.231 & 15.547 $\pm$ 3.732 & 15.509 $\pm$ 3.736 \\ \hline
Fair PreProc., $\lambda_r:$0.20 & 0.76 $\pm$ 0.02 & 0.048 $\pm$ 0.003 & 14.264 $\pm$ 3.322 & 15.571 $\pm$ 3.842 & 15.532 $\pm$ 3.843 \\ \hline
Fair PreProc., $\lambda_r:$0.30 & 0.76 $\pm$ 0.02 & 0.048 $\pm$ 0.004 & 14.734 $\pm$ 3.860 & 16.100 $\pm$ 4.436 & 16.059 $\pm$ 4.435 \\ \hline
Fair PreProc., $\lambda_r:$0.40 & 0.76 $\pm$ 0.02 & 0.048 $\pm$ 0.004 & 14.175 $\pm$ 4.103 & 15.482 $\pm$ 4.707 & 15.441 $\pm$ 4.703 \\ \hline
Fair PreProc., $\lambda_r:$0.50 & 0.76 $\pm$ 0.02 & 0.048 $\pm$ 0.005 & 14.632 $\pm$ 4.545 & 15.984 $\pm$ 5.201 & 15.943 $\pm$ 5.195 \\ \hline
Fair PreProc., $\lambda_r:$0.60 & 0.76 $\pm$ 0.02 & 0.047 $\pm$ 0.005 & 13.237 $\pm$ 5.258 & 14.477 $\pm$ 5.981 & 14.435 $\pm$ 5.971 \\ \hline
Fair PreProc., $\lambda_r:$0.70 & 0.76 $\pm$ 0.02 & 0.047 $\pm$ 0.006 & 13.397 $\pm$ 5.474 & 14.657 $\pm$ 6.223 & 14.614 $\pm$ 6.213 \\ \hline
Fair PreProc., $\lambda_r:$0.80 & 0.76 $\pm$ 0.02 & 0.047 $\pm$ 0.007 & 13.573 $\pm$ 6.436 & 14.864 $\pm$ 7.262 & 14.821 $\pm$ 7.246 \\ \hline
Fair PreProc., $\lambda_r:$0.90 & 0.76 $\pm$ 0.02 & 0.048 $\pm$ 0.008 & 13.771 $\pm$ 6.689 & 15.076 $\pm$ 7.556 & 15.031 $\pm$ 7.536 \\ \hline
Fair PreProc., $\lambda_r:$1.00 & 0.76 $\pm$ 0.02 & 0.048 $\pm$ 0.007 & 13.636 $\pm$ 6.555 & 14.933 $\pm$ 7.408 & 14.888 $\pm$ 7.392 \\ \hline

\end{tabular}%
}
\caption{Mean and one standard deviation results for all methods-hyperparameter combinations for all methods except \fairlp, for the \texttt{LawSchool} dataset across 10 separate runs. }
\label{tab:lawschool_rest}
\end{table*}

\begin{table*}[ht]
\centering
\resizebox{1.0\textwidth}{!}{%
\begin{tabular}{|c||c|c|c|c|c|}
\hline
Method & AUC  &  \begin{tabular}[c]{@{}c@{}}\cddwtext{}\\ (normalized) \end{tabular}  & \begin{tabular}[c]{@{}c@{}}\cddlptext{}  \\ (uniform) \end{tabular} & \begin{tabular}[c]{@{}c@{}}\cddlptext{}  \\ (Ave. $\P(L|A)$)  \end{tabular}& \begin{tabular}[c]{@{}c@{}}\cddlptext{}  \\ ($\P(L)$) \end{tabular} \\ \hline
FairLeap (uniform), $\lambda_w:$0.00 & 0.76 $\pm$ 0.02 & 0.046 $\pm$ 0.002 & 132.110 $\pm$ 14.357 & 14.432 $\pm$ 1.624 & 14.395 $\pm$ 1.619 \\ \hline
FairLeap (uniform), $\lambda_w:$0.00 & 0.76 $\pm$ 0.02 & 0.043 $\pm$ 0.002 & 113.947 $\pm$ 14.296 & 12.429 $\pm$ 1.589 & 12.395 $\pm$ 1.581 \\ \hline
FairLeap (uniform), $\lambda_w:$0.00 & 0.76 $\pm$ 0.02 & 0.036 $\pm$ 0.002 & 68.021 $\pm$ 9.092 & 7.347 $\pm$ 1.048 & 7.323 $\pm$ 1.038 \\ \hline
FairLeap (uniform), $\lambda_w:$0.01 & 0.78 $\pm$ 0.01 & 0.007 $\pm$ 0.001 & 14.041 $\pm$ 3.161 & 1.491 $\pm$ 0.346 & 1.486 $\pm$ 0.345 \\ \hline
FairLeap (uniform), $\lambda_w:$0.05 & 0.80 $\pm$ 0.01 & 0.003 $\pm$ 0.000 & 1.980 $\pm$ 1.219 & 0.217 $\pm$ 0.134 & 0.216 $\pm$ 0.134 \\ \hline
FairLeap (uniform), $\lambda_w:$0.22 & 0.80 $\pm$ 0.01 & 0.003 $\pm$ 0.000 & 1.332 $\pm$ 0.542 & 0.146 $\pm$ 0.059 & 0.145 $\pm$ 0.059 \\ \hline
FairLeap (uniform), $\lambda_w:$1.00 & 0.80 $\pm$ 0.01 & 0.003 $\pm$ 0.000 & 0.664 $\pm$ 0.415 & 0.072 $\pm$ 0.045 & 0.072 $\pm$ 0.045 \\ \hline
FairLeap (uniform), $\lambda_w:$4.64 & 0.80 $\pm$ 0.01 & 0.003 $\pm$ 0.000 & 0.404 $\pm$ 0.336 & 0.044 $\pm$ 0.037 & 0.044 $\pm$ 0.037 \\ \hline
FairLeap (uniform), $\lambda_w:$21.54 & 0.80 $\pm$ 0.01 & 0.002 $\pm$ 0.001 & 0.171 $\pm$ 0.143 & 0.019 $\pm$ 0.016 & 0.019 $\pm$ 0.016 \\ \hline
FairLeap (uniform), $\lambda_w:$100.00 & 0.81 $\pm$ 0.01 & 0.002 $\pm$ 0.001 & 0.137 $\pm$ 0.128 & 0.015 $\pm$ 0.014 & 0.015 $\pm$ 0.014 \\ \hline
FairLeap (P(L)), $\lambda_w:$0.00 & 0.76 $\pm$ 0.02 & 0.047 $\pm$ 0.002 & 141.579 $\pm$ 17.141 & 15.474 $\pm$ 1.950 & 15.435 $\pm$ 1.944 \\ \hline
FairLeap (P(L)), $\lambda_w:$0.00 & 0.76 $\pm$ 0.02 & 0.046 $\pm$ 0.002 & 139.786 $\pm$ 15.280 & 15.275 $\pm$ 1.747 & 15.237 $\pm$ 1.741 \\ \hline
FairLeap (P(L)), $\lambda_w:$0.00 & 0.76 $\pm$ 0.02 & 0.045 $\pm$ 0.002 & 124.893 $\pm$ 14.577 & 13.638 $\pm$ 1.697 & 13.603 $\pm$ 1.696 \\ \hline
FairLeap (P(L)), $\lambda_w:$0.01 & 0.76 $\pm$ 0.02 & 0.040 $\pm$ 0.002 & 89.241 $\pm$ 11.527 & 9.698 $\pm$ 1.320 & 9.669 $\pm$ 1.312 \\ \hline
FairLeap (P(L)), $\lambda_w:$0.05 & 0.76 $\pm$ 0.01 & 0.026 $\pm$ 0.001 & 40.443 $\pm$ 6.568 & 4.320 $\pm$ 0.751 & 4.305 $\pm$ 0.743 \\ \hline
FairLeap (P(L)), $\lambda_w:$0.22 & 0.80 $\pm$ 0.01 & 0.003 $\pm$ 0.000 & 1.684 $\pm$ 0.915 & 0.183 $\pm$ 0.099 & 0.182 $\pm$ 0.099 \\ \hline
FairLeap (P(L)), $\lambda_w:$1.00 & 0.80 $\pm$ 0.01 & 0.003 $\pm$ 0.000 & 1.545 $\pm$ 0.469 & 0.168 $\pm$ 0.051 & 0.168 $\pm$ 0.051 \\ \hline
FairLeap (P(L)), $\lambda_w:$4.64 & 0.80 $\pm$ 0.01 & 0.003 $\pm$ 0.000 & 0.829 $\pm$ 0.552 & 0.090 $\pm$ 0.060 & 0.090 $\pm$ 0.060 \\ \hline
FairLeap (P(L)), $\lambda_w:$21.54 & 0.80 $\pm$ 0.01 & 0.003 $\pm$ 0.000 & 0.437 $\pm$ 0.351 & 0.047 $\pm$ 0.038 & 0.047 $\pm$ 0.038 \\ \hline
FairLeap (P(L)), $\lambda_w:$100.00 & 0.80 $\pm$ 0.01 & 0.002 $\pm$ 0.000 & 0.210 $\pm$ 0.124 & 0.023 $\pm$ 0.013 & 0.023 $\pm$ 0.013 \\ \hline
Fairleap (Ave. P(L$|$A)), $\lambda_w:$0.00 & 0.76 $\pm$ 0.01 & 0.047 $\pm$ 0.002 & 141.081 $\pm$ 10.070 & 15.495 $\pm$ 1.004 & 15.462 $\pm$ 1.003 \\ \hline
Fairleap (Ave. P(L$|$A)), $\lambda_w:$0.00 & 0.76 $\pm$ 0.02 & 0.046 $\pm$ 0.002 & 139.381 $\pm$ 8.513 & 15.306 $\pm$ 0.910 & 15.274 $\pm$ 0.915 \\ \hline
Fairleap (Ave. P(L$|$A)), $\lambda_w:$0.00 & 0.76 $\pm$ 0.01 & 0.045 $\pm$ 0.002 & 125.836 $\pm$ 9.474 & 13.811 $\pm$ 1.067 & 13.781 $\pm$ 1.074 \\ \hline
Fairleap (Ave. P(L$|$A)), $\lambda_w:$0.01 & 0.76 $\pm$ 0.01 & 0.041 $\pm$ 0.002 & 91.037 $\pm$ 9.269 & 9.947 $\pm$ 1.031 & 9.923 $\pm$ 1.021 \\ \hline
Fairleap (Ave. P(L$|$A)), $\lambda_w:$0.05 & 0.77 $\pm$ 0.01 & 0.026 $\pm$ 0.001 & 41.407 $\pm$ 7.574 & 4.440 $\pm$ 0.853 & 4.428 $\pm$ 0.844 \\ \hline
Fairleap (Ave. P(L$|$A)), $\lambda_w:$0.22 & 0.80 $\pm$ 0.01 & 0.003 $\pm$ 0.001 & 2.453 $\pm$ 1.993 & 0.267 $\pm$ 0.220 & 0.267 $\pm$ 0.220 \\ \hline
Fairleap (Ave. P(L$|$A)), $\lambda_w:$1.00 & 0.80 $\pm$ 0.01 & 0.003 $\pm$ 0.001 & 2.062 $\pm$ 0.795 & 0.225 $\pm$ 0.088 & 0.225 $\pm$ 0.088 \\ \hline
Fairleap (Ave. P(L$|$A)), $\lambda_w:$4.64 & 0.80 $\pm$ 0.01 & 0.003 $\pm$ 0.000 & 0.829 $\pm$ 0.869 & 0.091 $\pm$ 0.096 & 0.091 $\pm$ 0.096 \\ \hline
Fairleap (Ave. P(L$|$A)), $\lambda_w:$21.54 & 0.80 $\pm$ 0.01 & 0.003 $\pm$ 0.000 & 0.737 $\pm$ 0.591 & 0.081 $\pm$ 0.065 & 0.081 $\pm$ 0.065 \\ \hline
Fairleap (Ave. P(L$|$A)), $\lambda_w:$100 & 0.80 $\pm$ 0.01 & 0.002 $\pm$ 0.000 & 0.303 $\pm$ 0.217 & 0.033 $\pm$ 0.024 & 0.033 $\pm$ 0.024 \\ \hline
\end{tabular}%
}
\caption{Mean and one standard deviation results for all methods-hyperparameter combinations for the 3 variants of \fairlp, for the \texttt{LawSchool} dataset across 10 separate runs. }
\label{tab:lawschool_fairlp}
\end{table*}


\begin{table*}[ht]
\centering
\resizebox{1.0\textwidth}{!}{%
\begin{tabular}{|c||c|c|c|c|c|}
\hline
Method & AUC  &  \begin{tabular}[c]{@{}c@{}}\cddwtext{}\\ (normalized) \end{tabular}  & \begin{tabular}[c]{@{}c@{}}\cddlptext{}  \\ (uniform) \end{tabular} & \begin{tabular}[c]{@{}c@{}}\cddlptext{}  \\ (Ave. $\P(L|A)$)  \end{tabular}& \begin{tabular}[c]{@{}c@{}}\cddlptext{}  \\ ($\P(L)$) \end{tabular} \\ \hline
DCFR, $\lambda_d:$0.10 & 0.03 $\pm$ 0.00 & 2.975 $\pm$ 1.093 & 33.472 $\pm$ 5.698 & 3.566 $\pm$ 0.715 & 3.548 $\pm$ 0.722 \\ \hline
DCFR, $\lambda_d:$0.25 & 0.03 $\pm$ 0.00 & 2.963 $\pm$ 1.092 & 25.007 $\pm$ 4.511 & 2.662 $\pm$ 0.583 & 2.649 $\pm$ 0.588 \\ \hline
DCFR, $\lambda_d:$0.50 & 0.03 $\pm$ 0.00 & 2.950 $\pm$ 1.092 & 19.614 $\pm$ 5.013 & 2.095 $\pm$ 0.626 & 2.085 $\pm$ 0.627 \\ \hline
DCFR, $\lambda_d:$0.75 & 0.03 $\pm$ 0.00 & 2.947 $\pm$ 1.093 & 18.263 $\pm$ 4.990 & 1.939 $\pm$ 0.624 & 1.931 $\pm$ 0.623 \\ \hline
DCFR, $\lambda_d:$1.00 & 0.03 $\pm$ 0.00 & 2.945 $\pm$ 1.093 & 18.464 $\pm$ 3.793 & 1.957 $\pm$ 0.487 & 1.950 $\pm$ 0.490 \\ \hline
DCFR, $\lambda_d:$1.50 & 0.03 $\pm$ 0.00 & 2.946 $\pm$ 1.097 & 19.063 $\pm$ 6.640 & 2.023 $\pm$ 0.730 & 2.017 $\pm$ 0.744 \\ \hline
DCFR, $\lambda_d:$2.00 & 0.03 $\pm$ 0.00 & 2.943 $\pm$ 1.095 & 18.377 $\pm$ 4.436 & 1.941 $\pm$ 0.509 & 1.934 $\pm$ 0.512 \\ \hline
DCFR, $\lambda_d:$5.00 & 0.02 $\pm$ 0.00 & 2.957 $\pm$ 1.101 & 27.067 $\pm$ 15.656 & 2.894 $\pm$ 1.747 & 2.887 $\pm$ 1.748 \\ \hline
DCFR, $\lambda_d:$10.00 & 0.02 $\pm$ 0.00 & 2.958 $\pm$ 1.101 & 28.300 $\pm$ 17.502 & 2.964 $\pm$ 1.833 & 2.959 $\pm$ 1.842 \\ \hline
DCFR, $\lambda_d:$15.00 & 0.02 $\pm$ 0.00 & 2.966 $\pm$ 1.096 & 33.330 $\pm$ 14.067 & 3.542 $\pm$ 1.528 & 3.529 $\pm$ 1.528 \\ \hline
DCFR, $\lambda_d:$20.00 & 0.02 $\pm$ 0.00 & 2.956 $\pm$ 1.084 & 27.457 $\pm$ 14.423 & 2.891 $\pm$ 1.592 & 2.876 $\pm$ 1.583 \\ \hline
FairBiT, $\lambda_b:$0.00 & 0.02 $\pm$ 0.00 & 2.887 $\pm$ 0.971 & 35.226 $\pm$ 3.210 & 3.615 $\pm$ 0.350 & 3.614 $\pm$ 0.353 \\ \hline
FairBiT, $\lambda_b:$0.00 & 0.02 $\pm$ 0.00 & 2.888 $\pm$ 0.971 & 35.491 $\pm$ 3.707 & 3.642 $\pm$ 0.399 & 3.641 $\pm$ 0.403 \\ \hline
FairBiT, $\lambda_b:$0.00 & 0.02 $\pm$ 0.00 & 2.888 $\pm$ 0.971 & 35.475 $\pm$ 3.764 & 3.640 $\pm$ 0.407 & 3.639 $\pm$ 0.410 \\ \hline
FairBiT, $\lambda_b:$0.00 & 0.02 $\pm$ 0.00 & 2.887 $\pm$ 0.971 & 35.171 $\pm$ 3.787 & 3.608 $\pm$ 0.410 & 3.607 $\pm$ 0.414 \\ \hline
FairBiT, $\lambda_b:$0.02 & 0.02 $\pm$ 0.00 & 2.884 $\pm$ 0.971 & 34.018 $\pm$ 3.718 & 3.489 $\pm$ 0.403 & 3.488 $\pm$ 0.406 \\ \hline
FairBiT, $\lambda_b:$0.06 & 0.02 $\pm$ 0.00 & 2.873 $\pm$ 0.971 & 30.206 $\pm$ 3.313 & 3.098 $\pm$ 0.360 & 3.097 $\pm$ 0.363 \\ \hline
FairBiT, $\lambda_b:$0.22 & 0.02 $\pm$ 0.00 & 2.851 $\pm$ 0.970 & 20.869 $\pm$ 2.518 & 2.140 $\pm$ 0.273 & 2.139 $\pm$ 0.275 \\ \hline
FairBiT, $\lambda_b:$0.77 & 0.03 $\pm$ 0.00 & 2.826 $\pm$ 0.970 & 8.504 $\pm$ 1.165 & 0.871 $\pm$ 0.125 & 0.871 $\pm$ 0.126 \\ \hline
FairBiT, $\lambda_b:$2.78 & 0.04 $\pm$ 0.00 & 2.817 $\pm$ 0.970 & 2.175 $\pm$ 0.278 & 0.220 $\pm$ 0.030 & 0.220 $\pm$ 0.030 \\ \hline
FairBiT, $\lambda_b:$10.00 & 0.05 $\pm$ 0.00 & 2.816 $\pm$ 0.970 & 0.838 $\pm$ 0.143 & 0.084 $\pm$ 0.015 & 0.084 $\pm$ 0.015 \\ \hline
FairLearn (CDP) & 0.03 $\pm$ 0.01 & 2.851 $\pm$ 0.969 & 16.307 $\pm$ 4.753 & 1.660 $\pm$ 0.489 & 1.661 $\pm$ 0.490 \\ \hline
FairLearn (DP) & 0.03 $\pm$ 0.01 & 2.851 $\pm$ 0.969 & 16.307 $\pm$ 4.753 & 1.660 $\pm$ 0.489 & 1.661 $\pm$ 0.490 \\ \hline
Wass. Regularization, $\lambda_w:$0.00 & 0.02 $\pm$ 0.00 & 2.887 $\pm$ 0.971 & 35.248 $\pm$ 3.370 & 3.616 $\pm$ 0.375 & 3.615 $\pm$ 0.378 \\ \hline
Wass. Regularization, $\lambda_w:$0.00 & 0.02 $\pm$ 0.00 & 2.888 $\pm$ 0.971 & 35.439 $\pm$ 3.633 & 3.636 $\pm$ 0.397 & 3.635 $\pm$ 0.401 \\ \hline
Wass. Regularization, $\lambda_w:$0.00 & 0.02 $\pm$ 0.00 & 2.887 $\pm$ 0.971 & 35.322 $\pm$ 3.720 & 3.623 $\pm$ 0.406 & 3.622 $\pm$ 0.409 \\ \hline
Wass. Regularization, $\lambda_w:$0.00 & 0.02 $\pm$ 0.00 & 2.886 $\pm$ 0.971 & 34.715 $\pm$ 3.736 & 3.561 $\pm$ 0.407 & 3.560 $\pm$ 0.410 \\ \hline
Wass. Regularization, $\lambda_w:$0.02 & 0.02 $\pm$ 0.00 & 2.880 $\pm$ 0.971 & 32.396 $\pm$ 3.645 & 3.323 $\pm$ 0.397 & 3.322 $\pm$ 0.400 \\ \hline
Wass. Regularization, $\lambda_w:$0.06 & 0.02 $\pm$ 0.00 & 2.863 $\pm$ 0.970 & 24.545 $\pm$ 3.189 & 2.518 $\pm$ 0.348 & 2.517 $\pm$ 0.351 \\ \hline
Wass. Regularization, $\lambda_w:$0.22 & 0.03 $\pm$ 0.00 & 2.837 $\pm$ 0.968 & 11.789 $\pm$ 2.384 & 1.189 $\pm$ 0.245 & 1.192 $\pm$ 0.246 \\ \hline
Wass. Regularization, $\lambda_w:$0.77 & 0.04 $\pm$ 0.00 & 2.826 $\pm$ 0.969 & 8.325 $\pm$ 1.393 & 0.841 $\pm$ 0.146 & 0.845 $\pm$ 0.147 \\ \hline
Wass. Regularization, $\lambda_w:$2.78 & 0.05 $\pm$ 0.00 & 2.816 $\pm$ 0.970 & 2.453 $\pm$ 0.492 & 0.247 $\pm$ 0.051 & 0.248 $\pm$ 0.051 \\ \hline
Wass. Regularization, $\lambda_w:$10.00 & 0.05 $\pm$ 0.00 & 2.815 $\pm$ 0.969 & 0.121 $\pm$ 0.066 & 0.012 $\pm$ 0.007 & 0.012 $\pm$ 0.007 \\ \hline
Legit Only & 0.04 $\pm$ 0.00 & 2.820 $\pm$ 0.970 & 0.000 $\pm$ 0.000 & 0.000 $\pm$ 0.000 & 0.000 $\pm$ 0.000 \\ \hline
No Regularization & 0.02 $\pm$ 0.00 & 2.887 $\pm$ 0.971 & 3.537 $\pm$ 0.317 & 3.629 $\pm$ 0.353 & 3.628 $\pm$ 0.356 \\ \hline
Fair PreProc., $\lambda_r:$0.00 & 0.02 $\pm$ 0.00 & 2.893 $\pm$ 0.971 & 3.992 $\pm$ 0.420 & 4.081 $\pm$ 0.444 & 4.081 $\pm$ 0.447 \\ \hline
Fair PreProc., $\lambda_r:$0.10 & 0.02 $\pm$ 0.00 & 2.891 $\pm$ 0.971 & 3.955 $\pm$ 0.321 & 4.044 $\pm$ 0.345 & 4.045 $\pm$ 0.348 \\ \hline
Fair PreProc., $\lambda_r:$0.20 & 0.02 $\pm$ 0.00 & 2.890 $\pm$ 0.970 & 3.909 $\pm$ 0.388 & 3.997 $\pm$ 0.408 & 3.998 $\pm$ 0.411 \\ \hline
Fair PreProc., $\lambda_r:$0.30 & 0.02 $\pm$ 0.00 & 2.889 $\pm$ 0.970 & 3.914 $\pm$ 0.326 & 3.998 $\pm$ 0.348 & 3.999 $\pm$ 0.351 \\ \hline
Fair PreProc., $\lambda_r:$0.40 & 0.02 $\pm$ 0.00 & 2.886 $\pm$ 0.970 & 3.780 $\pm$ 0.376 & 3.863 $\pm$ 0.386 & 3.864 $\pm$ 0.389 \\ \hline
Fair PreProc., $\lambda_r:$0.50 & 0.02 $\pm$ 0.00 & 2.886 $\pm$ 0.969 & 3.868 $\pm$ 0.362 & 3.948 $\pm$ 0.381 & 3.950 $\pm$ 0.385 \\ \hline
Fair PreProc., $\lambda_r:$0.60 & 0.02 $\pm$ 0.00 & 2.888 $\pm$ 0.970 & 3.931 $\pm$ 0.485 & 4.011 $\pm$ 0.493 & 4.013 $\pm$ 0.497 \\ \hline
Fair PreProc., $\lambda_r:$0.70 & 0.02 $\pm$ 0.00 & 2.888 $\pm$ 0.970 & 3.924 $\pm$ 0.487 & 4.001 $\pm$ 0.499 & 4.004 $\pm$ 0.503 \\ \hline
Fair PreProc., $\lambda_r:$0.80 & 0.02 $\pm$ 0.00 & 2.888 $\pm$ 0.967 & 3.962 $\pm$ 0.605 & 4.042 $\pm$ 0.610 & 4.045 $\pm$ 0.615 \\ \hline
Fair PreProc., $\lambda_r:$0.90 & 0.02 $\pm$ 0.00 & 2.886 $\pm$ 0.968 & 3.868 $\pm$ 0.606 & 3.944 $\pm$ 0.613 & 3.947 $\pm$ 0.618 \\ \hline
Fair PreProc., $\lambda_r:$1.00 & 0.02 $\pm$ 0.00 & 2.884 $\pm$ 0.968 & 3.693 $\pm$ 0.576 & 3.770 $\pm$ 0.577 & 3.773 $\pm$ 0.582 \\ \hline
\end{tabular}%
}
\caption{Mean and one standard deviation results for all methods-hyperparameter combinations for all methods except \fairlp, for the \texttt{Communities and Crime} dataset across 10 separate runs. }
\label{tab:crime_rest}
\end{table*}

\begin{table*}[ht]
\centering
\resizebox{1.0\textwidth}{!}{%
\begin{tabular}{|c||c|c|c|c|c|}
\hline
Method & AUC  &  \begin{tabular}[c]{@{}c@{}}\cddwtext{}\\ (normalized) \end{tabular}  & \begin{tabular}[c]{@{}c@{}}\cddlptext{}  \\ (uniform) \end{tabular} & \begin{tabular}[c]{@{}c@{}}\cddlptext{}  \\ (Ave. $\P(L|A)$)  \end{tabular}& \begin{tabular}[c]{@{}c@{}}\cddlptext{}  \\ ($\P(L)$) \end{tabular} \\ \hline
FairLeap (uniform), $\lambda_w:$0.00 & 0.02 $\pm$ 0.00 & 2.884 $\pm$ 0.971 & 33.601 $\pm$ 2.883 & 3.449 $\pm$ 0.314 & 3.447 $\pm$ 0.318 \\ \hline
FairLeap (uniform), $\lambda_w:$0.00 & 0.02 $\pm$ 0.00 & 2.877 $\pm$ 0.971 & 29.820 $\pm$ 3.179 & 3.062 $\pm$ 0.349 & 3.060 $\pm$ 0.352 \\ \hline
FairLeap (uniform), $\lambda_w:$0.00 & 0.02 $\pm$ 0.00 & 2.851 $\pm$ 0.970 & 16.821 $\pm$ 2.517 & 1.725 $\pm$ 0.269 & 1.724 $\pm$ 0.271 \\ \hline
FairLeap (uniform), $\lambda_w:$0.01 & 0.04 $\pm$ 0.00 & 2.820 $\pm$ 0.970 & 1.133 $\pm$ 0.599 & 0.116 $\pm$ 0.063 & 0.116 $\pm$ 0.063 \\ \hline
FairLeap (uniform), $\lambda_w:$0.05 & 0.04 $\pm$ 0.00 & 2.820 $\pm$ 0.971 & 0.058 $\pm$ 0.021 & 0.006 $\pm$ 0.002 & 0.006 $\pm$ 0.002 \\ \hline
FairLeap (uniform), $\lambda_w:$0.22 & 0.04 $\pm$ 0.00 & 2.820 $\pm$ 0.971 & 0.027 $\pm$ 0.007 & 0.003 $\pm$ 0.001 & 0.003 $\pm$ 0.001 \\ \hline
FairLeap (uniform), $\lambda_w:$1.00 & 0.04 $\pm$ 0.00 & 2.819 $\pm$ 0.970 & 0.016 $\pm$ 0.005 & 0.002 $\pm$ 0.001 & 0.002 $\pm$ 0.001 \\ \hline
FairLeap (uniform), $\lambda_w:$4.64 & 0.04 $\pm$ 0.00 & 2.818 $\pm$ 0.970 & 0.013 $\pm$ 0.005 & 0.001 $\pm$ 0.000 & 0.001 $\pm$ 0.000 \\ \hline
FairLeap (uniform), $\lambda_w:$21.54 & 0.05 $\pm$ 0.00 & 2.816 $\pm$ 0.970 & 0.010 $\pm$ 0.005 & 0.001 $\pm$ 0.001 & 0.001 $\pm$ 0.001 \\ \hline
FairLeap (uniform), $\lambda_w:$100.00 & 0.05 $\pm$ 0.00 & 2.815 $\pm$ 0.970 & 0.006 $\pm$ 0.002 & 0.001 $\pm$ 0.000 & 0.001 $\pm$ 0.000 \\ \hline
FairLeap (P(L)), $\lambda_w:$0.00 & 0.02 $\pm$ 0.00 & 2.886 $\pm$ 0.972 & 34.671 $\pm$ 3.001 & 3.560 $\pm$ 0.331 & 3.559 $\pm$ 0.334 \\ \hline
FairLeap (P(L)), $\lambda_w:$0.00 & 0.02 $\pm$ 0.00 & 2.887 $\pm$ 0.971 & 34.868 $\pm$ 3.827 & 3.577 $\pm$ 0.414 & 3.576 $\pm$ 0.417 \\ \hline
FairLeap (P(L)), $\lambda_w:$0.00 & 0.02 $\pm$ 0.00 & 2.882 $\pm$ 0.971 & 32.780 $\pm$ 3.571 & 3.364 $\pm$ 0.390 & 3.363 $\pm$ 0.393 \\ \hline
FairLeap (P(L)), $\lambda_w:$0.01 & 0.02 $\pm$ 0.00 & 2.866 $\pm$ 0.970 & 24.682 $\pm$ 3.075 & 2.534 $\pm$ 0.334 & 2.533 $\pm$ 0.336 \\ \hline
FairLeap (P(L)), $\lambda_w:$0.05 & 0.03 $\pm$ 0.00 & 2.831 $\pm$ 0.969 & 7.530 $\pm$ 1.620 & 0.770 $\pm$ 0.169 & 0.771 $\pm$ 0.170 \\ \hline
FairLeap (P(L)), $\lambda_w:$0.22 & 0.04 $\pm$ 0.00 & 2.820 $\pm$ 0.971 & 0.133 $\pm$ 0.081 & 0.014 $\pm$ 0.008 & 0.014 $\pm$ 0.008 \\ \hline
FairLeap (P(L)), $\lambda_w:$1.00 & 0.04 $\pm$ 0.00 & 2.820 $\pm$ 0.971 & 0.033 $\pm$ 0.008 & 0.003 $\pm$ 0.001 & 0.003 $\pm$ 0.001 \\ \hline
FairLeap (P(L)), $\lambda_w:$4.64 & 0.04 $\pm$ 0.00 & 2.820 $\pm$ 0.971 & 0.017 $\pm$ 0.004 & 0.002 $\pm$ 0.000 & 0.002 $\pm$ 0.000 \\ \hline
FairLeap (P(L)), $\lambda_w:$21.54 & 0.04 $\pm$ 0.00 & 2.819 $\pm$ 0.970 & 0.013 $\pm$ 0.005 & 0.001 $\pm$ 0.001 & 0.001 $\pm$ 0.001 \\ \hline
FairLeap (P(L)), $\lambda_w:$100.00 & 0.05 $\pm$ 0.00 & 2.817 $\pm$ 0.970 & 0.009 $\pm$ 0.003 & 0.001 $\pm$ 0.000 & 0.001 $\pm$ 0.000 \\ \hline
Fairleap (Ave. P(L$|$A)), $\lambda_w:$0.00 & 0.02 $\pm$ 0.00 & 2.887 $\pm$ 0.971 & 35.155 $\pm$ 3.198 & 3.607 $\pm$ 0.356 & 3.605 $\pm$ 0.358 \\ \hline
Fairleap (Ave. P(L$|$A)), $\lambda_w:$0.00 & 0.02 $\pm$ 0.00 & 2.887 $\pm$ 0.971 & 35.033 $\pm$ 3.772 & 3.593 $\pm$ 0.408 & 3.592 $\pm$ 0.411 \\ \hline
Fairleap (Ave. P(L$|$A)), $\lambda_w:$0.00 & 0.02 $\pm$ 0.00 & 2.882 $\pm$ 0.971 & 32.893 $\pm$ 3.562 & 3.375 $\pm$ 0.387 & 3.374 $\pm$ 0.390 \\ \hline
Fairleap (Ave. P(L$|$A)), $\lambda_w:$0.01 & 0.02 $\pm$ 0.00 & 2.866 $\pm$ 0.970 & 24.730 $\pm$ 3.019 & 2.539 $\pm$ 0.328 & 2.538 $\pm$ 0.331 \\ \hline
Fairleap (Ave. P(L$|$A)), $\lambda_w:$0.05 & 0.03 $\pm$ 0.00 & 2.831 $\pm$ 0.970 & 7.615 $\pm$ 1.553 & 0.778 $\pm$ 0.161 & 0.778 $\pm$ 0.162 \\ \hline
Fairleap (Ave. P(L$|$A)), $\lambda_w:$0.22 & 0.04 $\pm$ 0.00 & 2.820 $\pm$ 0.971 & 0.126 $\pm$ 0.129 & 0.013 $\pm$ 0.013 & 0.013 $\pm$ 0.013 \\ \hline
Fairleap (Ave. P(L$|$A)), $\lambda_w:$1.00 & 0.04 $\pm$ 0.00 & 2.820 $\pm$ 0.971 & 0.027 $\pm$ 0.006 & 0.003 $\pm$ 0.001 & 0.003 $\pm$ 0.001 \\ \hline
Fairleap (Ave. P(L$|$A)), $\lambda_w:$4.64 & 0.04 $\pm$ 0.00 & 2.819 $\pm$ 0.970 & 0.018 $\pm$ 0.005 & 0.002 $\pm$ 0.001 & 0.002 $\pm$ 0.001 \\ \hline
Fairleap (Ave. P(L$|$A)), $\lambda_w:$21.54 & 0.04 $\pm$ 0.00 & 2.819 $\pm$ 0.970 & 0.013 $\pm$ 0.004 & 0.001 $\pm$ 0.000 & 0.001 $\pm$ 0.000 \\ \hline
Fairleap (Ave. P(L$|$A)), $\lambda_w:$100 & 0.05 $\pm$ 0.00 & 2.817 $\pm$ 0.970 & 0.010 $\pm$ 0.002 & 0.001 $\pm$ 0.000 & 0.001 $\pm$ 0.000 \\ \hline
\end{tabular}%
}
\caption{Mean and one standard deviation results for all methods-hyperparameter combinations for the 3 variants of \fairlp, for the \texttt{Commuinites and Crime} dataset across 10 separate runs. }
\label{tab:crime_fairlp}
\end{table*}

\end{document}